\documentclass{article}
\usepackage[utf8]{inputenc}
\usepackage[T1]{fontenc}

\usepackage{blindtext}
\usepackage{titlesec}

\usepackage{microtype}
\usepackage{graphicx}
\usepackage{subfigure}
\usepackage{booktabs} 

\usepackage{amsmath,amsfonts,bm}

\def\eqref#1{equation~\ref{#1}}

\def\1{\bm{1}}

\def\vtheta{{\bm{\theta}}}
\def\va{{\bm{a}}}
\def\vb{{\bm{b}}}

\def\vf{{\bm{f}}}
\def\vg{{\bm{g}}}

\def\vs{{\bm{s}}}

\def\vu{{\bm{u}}}
\def\vv{{\bm{v}}}
\def\vw{{\bm{w}}}
\def\vx{{\bm{x}}}
\def\vy{{\bm{y}}}
\def\vz{{\bm{z}}}

\def\vY{{\bm{Y}}}
\def\vF{{\bm{F}}}

\def\mA{{\bm{A}}}

\def\mD{{\bm{D}}}

\def\mF{{\bm{F}}}

\def\mI{{\bm{I}}}

\def\mL{{\bm{L}}}

\def\mP{{\bm{P}}}

\def\mU{{\bm{U}}}

\def\mW{{\bm{W}}}

\def\mY{{\bm{Y}}}

\def\mxi{{\bm{\xi}}}

\def\mLambda{{\bm{\Lambda}}}

\def\my{{\bm{y}}}
\def\mLambda{{\bm{\Lambda}}}

\DeclareMathAlphabet{\mathsfit}{\encodingdefault}{\sfdefault}{m}{sl}
\SetMathAlphabet{\mathsfit}{bold}{\encodingdefault}{\sfdefault}{bx}{n}

\def\gD{{\mathcal{D}}}

\def\gL{{\mathcal{L}}}

\def\gN{{\mathcal{N}}}

\def\gP{{\mathcal{P}}}

\def\gS{{\mathcal{S}}}

\def\gW{{\mathcal{W}}}
\def\gX{{\mathcal{X}}}
\def\gY{{\mathcal{Y}}}
\def\gZ{{\mathcal{Z}}}

\def\sE{{\mathbb{E}}}

\def\sP{{\mathbb{P}}}

\def\sR{{\mathbb{R}}}

\DeclareMathOperator*{\argmax}{arg\,max}
\DeclareMathOperator*{\argmin}{arg\,min}

\newcommand{\norm}[1]{\left \lVert #1 \right \rVert}
\newcommand{\normin}[1]{\lVert #1 \rVert}

\newcommand{\eg}{{{\em e.g}.,~}}

\usepackage{hyperref}

\usepackage[accepted]{icml2023}

\usepackage{amsmath}
\usepackage{amssymb}
\usepackage{mathtools}
\usepackage{amsthm}
\usepackage{aliascnt}
\usepackage{thm-restate}
\usepackage{bbm}

\definecolor{darkblue}{rgb}{0.0, 0.0, 0.7}

\usepackage[capitalize,noabbrev,nameinlink]{cleveref}

\usepackage{minitoc}
\theoremstyle{plain}
\newtheorem{theorem}{Theorem}[section]
\newtheorem{proposition}[theorem]{Proposition}
\newtheorem{lemma}[theorem]{Lemma}

\theoremstyle{definition}
\newtheorem{definition}[theorem]{Definition}
\newtheorem{assumption}[theorem]{Assumption}
\theoremstyle{remark}

\newcommand{\ie}{{\em i.e.,~}}

\newaliascnt{problem}{equation}
\aliascntresetthe{problem}
\creflabelformat{problem}{#2\textup{(#1)}#3}
\makeatletter

\def\endproblem{\eqno \hbox{\@eqnnum}$$\@ignoretrue}
\makeatother

\newaliascnt{model}{equation}
\aliascntresetthe{model}
\creflabelformat{model}{#2\textup{(#1)}#3}
\makeatletter

\def\endmodel{\eqno \hbox{\@eqnnum}$$\@ignoretrue}
\makeatother

\crefname{problem}{Problem}{Problems}
\crefname{algorithm}{Algorithm}{Algorithms}
\crefname{figure}{Fig.}{Fig.}
\crefname{proposition}{Proposition}{Propositions}
\crefname{appendix}{Appendix}{Appendices}
\crefname{assumption}{Assumption}{Assumptions}
\crefname{theorem}{Theorem}{Theorems}
\crefname{section}{Section}{Sections}

\crefname{pb_multiline}{Problem}{Problems}
\creflabelformat{pb_multiline}{#2{\upshape(#1)}#3}

\usepackage{enumitem}  
\newlist{lemmaenum}{enumerate}{1} 
\setlist[lemmaenum]{label=\emph{\roman*)}, ref=\thetheorem~\emph{\roman*)}}
\crefalias{lemmaenumi}{lemma}

\usepackage[textsize=tiny]{todonotes}

\icmltitlerunning{Synergies between Disentanglement and Sparsity}

\begin{document}
\doparttoc 
\faketableofcontents 

\twocolumn[

\icmltitle{Synergies between Disentanglement and Sparsity: \\
Generalization and Identifiability in Multi-Task Learning}






\icmlsetsymbol{equal}{*}

\begin{icmlauthorlist}
\icmlauthor{S\'ebastien Lachapelle}{equal,yyy}
\icmlauthor{Tristan Deleu}{equal,yyy}
\icmlauthor{Divyat Mahajan}{yyy}
\icmlauthor{Ioannis Mitliagkas}{yyy,cifar}
\icmlauthor{Yoshua Bengio}{yyy,cifar}
\icmlauthor{Simon Lacoste-Julien}{yyy,cifar}
\icmlauthor{Quentin Bertrand}{yyy}
\end{icmlauthorlist}

\icmlaffiliation{yyy}{Mila \& DIRO, Université de Montréal}
\icmlaffiliation{cifar}{Canada CIFAR AI Chair}

\icmlcorrespondingauthor{S\'ebastien Lachapelle}{lachaseb@mila.quebec}
\icmlcorrespondingauthor{Tristan Deleu}{deleutri@mila.quebec}

\icmltitlerunning{Synergies Between Disentanglement and Sparsity}
\icmlkeywords{Identifiability, Disentanglement, Multi-task learning, Multitask learning}

\vskip 0.3in
]



\printAffiliationsAndNotice{\icmlEqualContribution} 

\begin{abstract}
Although disentangled representations are often said to be beneficial for downstream tasks, current empirical and theoretical understanding is limited. In this work, we provide evidence that disentangled representations coupled with sparse task-specific predictors improve generalization. In the context of multi-task learning, we prove a new identifiability result that provides conditions under which maximally sparse predictors yield disentangled representations. Motivated by this theoretical result, we propose a practical approach to learn disentangled representations based on a sparsity-promoting bi-level optimization problem. Finally, we explore a meta-learning version of this algorithm based on group Lasso multiclass SVM predictors, for which we derive a tractable dual formulation. It obtains competitive results on standard few-shot classification benchmarks, while each task is using only a fraction of the learned representations.
\end{abstract}

\section{Introduction}
The recent literature on self-supervised learning has provided evidence that learning a representation on large corpuses of data can yield strong performances on a wide variety of downstream tasks~\citep{devlin2018bert,chen2020simclr}, especially in few-shot learning scenarios where the training data for these tasks is limited \citep{brown2020gpt3,dosovitskiy2020vit,radford2021clip}. Beyond transferring across multiple tasks, these learned representations also lead to improved robustness against distribution shifts \citep{wortsman2022wiseft} as well as stunning text-conditioned image generation~\citep{ramesh2022dalle2}. However, preliminary assessments of the latter have highlighted shortcomings related to compositionality~\citep{marcus2022preliminary}, suggesting new algorithmic innovations are needed.

Another line of work has argued for the integration of ideas from causality to make progress towards more robust and transferable machine learning systems~\citep{Pearl2018TheSP,scholkopf2019causality,InducGoyal2021}. \textit{Causal representation learning} has emerged recently as a field aiming to define and learn representations suited for causal reasoning~\citep{scholkopf2021causal}. This set of ideas is strongly related to learning \emph{disentangled representations}~\citep{bengio2013representation}. 
Informally, a representation is considered disentangled when its components are in one-to-one correspondence with natural and interpretable factors of variations, such as object positions, colors or shapes.
Although a plethora of works have investigated theoretically under which conditions disentanglement is possible through the lens of identifiability~\citep{TCL2016,PCL17,HyvarinenST19, iVAEkhemakhem20a,pmlr-v119-locatello20a,slowVAE,vonkugelgen2021selfsupervised,gresele2021independent,lachapelle2022disentanglement, CITRIS, ahuja2022towards}, fewer works have tackled \textit{how a disentangled representation could be beneficial for downstream tasks}. Those who did mainly provide empirical rather than theoretical evidence for or against its usefulness~\citep{pmlr-v97-locatello19a,steenkiste2019DisForAbstractReasoning,miladinovic2019disentangledODE,dittadi2021sim2real_dis,montero2021disGen}. We believe our work can bring some theoretical insights as to when and why disentanglement can help.

In this work, we explore synergies between disentanglement and sparse task-specific predictors in the context of multi-task learning.
At the heart of our contributions is the assumption that only a small subset of all factors of variations are useful for each downstream task, and this subset might change from one task to another. We will refer to such tasks as \textit{sparse tasks}, and their corresponding sets of useful factors as their \textit{supports}. This assumption was initially suggested by \citet[Section 3.5]{bengio2013representation}:
``the feature set being trained may be destined to be used in multiple tasks that may have distinct [and unknown] subsets of relevant features. Considerations such as these lead us to the conclusion that
the most robust approach to feature learning is to disentangle as many factors as possible, discarding as little information about the data as is practical''.
This strategy is in line with the current self-supervised learning trend~\citep{radford2021clip}, except for its focus on disentanglement.

\subsection{Contributions}
\setlist{nolistsep}
\begin{enumerate}[noitemsep]
    \item We formalize this ``sparse task assumption'' and argue theoretically and empirically how, when it holds, a disentangled representation coupled with a sparsity-regularized task-specific predictor can generalize better than their entangled counterparts~(\Cref{sec:dis+sparse=gen}).
    \item We introduce a novel identifiability result (\Cref{thm:disentanglement_via_optim}) which shows how one can leverage multiple sparse supervised tasks to learn a shared disentangled representation by regularizing the task-specific predictors to be maximally sparse (\Cref{sec:identifiability}). We note that the usage of supervision is in line with many recent results which leverages more or less weak forms of supervision to guarantee identifiability. Contrary to many existing identifiability results, ours allows for statistically dependent latent factors and a non-invertible map between observations and latents.
    \item Motivated by this result, we propose a tractable bi-level optimization (\Cref{pb:OG_problem_inner_relax}) to learn the shared representation while regularizing the task-specific predictors to be sparse (\Cref{sec:sparse_bilevel}). We validate our theory by showing our approach can indeed disentangle latent factors on tasks constructed from the 3D Shapes dataset~\citep{3dshapes18}.
    \item Finally, we draw a connection between this bi-level optimization problem and formulations from the meta-learning literature.
    Inspired by our identifiability result, we enhance an existing method \citep{Lee_Maji_Ravichandran_Soatto2019meta}, where the task-specific predictors are now group-sparse SVMs. We show that this new meta-learning algorithm achieves competitive performance on the \textit{mini}ImageNet benchmark  \citep{vinyals2016matchingnet}, while only using a fraction of the representation.
\end{enumerate} 

We emphasize that, although related, the theoretical contributions of Sections~\ref{sec:dis+sparse=gen}~\&~\ref{sub:disentanglement_sparse_mtl} are distinct and stand of their own. Indeed, Section~\ref{sec:dis+sparse=gen} shows how disentangled representations combined with sparsity regularization can improve generalization, while Section~\ref{sub:disentanglement_sparse_mtl} shows how regularizing task-specific predictors to be sparse can induce disentanglement in a multi-task learning setting.

\subsection{Background}
We start by introducing formally the notion of entangled and disentangled representations. 

First, we assume the existence of some ground-truth encoder function $\vf_\vtheta: \sR^d \rightarrow \sR^m$ that maps observations $\vx \in \gX \subseteq \sR^d$, \eg images, to its corresponding interpretable and usually lower dimensional representation $\vf_\vtheta(\vx) \in \sR^m$, $m \leq d$.
The exact form of this ground-truth encoder depends
on the task at hand, but also on what the machine learning practitioner considers as interpretable.
The learned encoder function
is denoted by $\vf_{\hat\vtheta}: \sR^d \rightarrow \sR^m$, and should not be conflated with the ground-truth representation $\vf_{\vtheta}$.
For example, $\vf_{\hat\vtheta}$ can be parametrized by a neural network.
Throughout, we are going to use the following definition of disentanglement.

\begin{definition}[Disentangled Representation, \citealt{iVAEkhemakhem20a,lachapelle2022disentanglement}]\label{def:disentanglement}
    A learned encoder function $\vf_{\hat\vtheta}: \sR^d \rightarrow \sR^m$ is said to be \emph{disentangled w.r.t. the ground-truth representation} $\vf_\vtheta$ when there exists an invertible diagonal matrix $\mD$ and a permutation matrix $\mP$ such that, for all $\vx \in \gX$, $\vf_{\hat\vtheta}(\vx) = \mD\mP \vf_{\vtheta}(\vx)$. 
\end{definition}
Intuitively, a representation is disentangled when there is a one-to-one correspondence between its components and those of the ground-truth representation, up to rescaling. When an encoder $\vf_{\hat\vtheta}$ is not disentangled, we say it is \emph{entangled}. Note that there exist less stringent notions of disentanglement which allow for component-wise nonlinear invertible transformations of the factors~\citep{PCL17,HyvarinenST19}.

\textbf{Notation.}
Capital bold letters denote matrices and lowercase bold letters denote vectors.
The set of integers from $1$ to $n$ is denoted by $[n]$.
We write $\norm{\cdot}$ for the Euclidean norm on vectors and the Frobenius norm on matrices.
For a matrix $\mA \in \sR^{k \times m}$,
$\norm{\mA}_{2, 1}
=
\sum_{j=1}^m \norm{\mA_{:j}}$,
and
$\norm{\mA}_{2, 0} = \sum_{j = 1}^m \mathbbm{1}_{\norm{\mA_{:j}} \neq 0} $, where $\mathbbm{1}$ is the indicator function.
The ground-truth parameter of the encoder function is $\vtheta$, while that of the learned representation is $\hat \vtheta$. We follow this convention for all the parameters throughout.
{\Cref{tab:TableOfNotation} in \Cref{app:mle_invariance_proof} summarizes all the notation.}

\section{Disentanglement and Sparse Task-Specific Predictors Improve Generalization}\label{sec:dis+sparse=gen} 
\label{sub:disentang_sparse_gen}
In this section,
we show that for any \emph{linearly equivalent} representation (entangled or disentangled), the maximum likelihood estimator defined in \Cref{pb:mle} yields the same model (\Cref{prop:mle_invariance}).
However, we also show that disentangled representations have better generalization properties when the task-specific predictor is regularized to be sparse. (\Cref{prop:sparsity_bayes_optimal,fig:sparsity-disentanglement-gains}). Our analysis is centred around the following assumption.


\begin{assumption}[Linear equivalence]\label{ass:lin_eq}
The learned encoder $\vf_{\hat\vtheta}$ is \emph{linearly equivalent} to the ground-truth encoder $\vf_\vtheta$, \ie there exists an invertible matrix $\mL$ such that, for all $\vx \in \gX$, ${\vf_{\hat\vtheta}(\vx) = \mL \vf_{\vtheta}(\vx)}$.
\end{assumption}

Note that similar notions of linear equivalence were used e.g. by \citet{HyvarinenST19,iVAEkhemakhem20a, roeder2020linear}

Despite being assumed linearly equivalent, the learned representation $\vf_{\hat\vtheta}$ might not be disentangled (\Cref{def:disentanglement}); in that case, we say the representation is \textit{linearly entangled}.
When we refer to a disentangled representation, we write $\mL := \mD\mP$.
\citet{roeder2020linear} have shown that many common methods learn representations identifiable up to linear equivalence, such as deep neural networks for classification, contrastive learning~\citep{oord2018representation,radford2021clip} and autoregressive language models~\citep{mikolov2010recurrent,GPT3}.

\subsection{MLE invariance to linear feature transformations}

Consider the following maximum likelihood estimator (MLE):\footnote{We assume the solution is unique.}
\begin{problem}\label{pb:mle}
    \hat\mW^{(\hat\vtheta)}_n := \argmax_{\tilde\mW} \sum_{(\vx, y) \in \gD} \log p(y; \bm\eta = \tilde\mW\vf_{\hat\vtheta}(\vx))\, , \label{eq:mle_no_reg}
\end{problem}
where $y$ denotes the label,
$\gD :=  \{(\vx^{(i)}, y^{(i)})\}_{i=1}^n$ is the dataset,
$p(y; \bm\eta)$ is a distribution over labels\footnote{$p(y; \bm\eta)$ could be a Gaussian density (regression) or a categorical distribution (classification).} parameterized by $\bm\eta \in \sR^k$,
and $\hat\mW \in \sR^{k \times m}$ is the \textit{task-specific predictor}.
The following result shows that the model estimated via maximum likelihood defined in~\Cref{pb:mle} is invariant to invertible linear transformations of the features.
Note that it is an almost direct consequence of the invariance of MLE to reparametrization~\citep[Thm.~7.2.10]{CaseBerg2001}.
See~\Cref{app:mle_invariance_proof} for a proof.
\begin{restatable}{proposition}{mleInvariance}\label{prop:mle_invariance}
    Let $\hat\mW^{(\hat\vtheta)}_n$ and $\hat\mW^{(\vtheta)}_n$ be the solutions to~\Cref{pb:mle} with the representations $\vf_{\hat\vtheta}$ and $\vf_\vtheta$, respectively (which we assume are unique).
    If $\vf_{\hat\vtheta}$ and $\vf_{\vtheta}$ are linearly equivalent (\cref{ass:lin_eq}), then we have, $ \forall \vx \in \gX$, $\hat\mW^{(\hat\vtheta)}_n\vf_{\hat\vtheta}(\vx) = \hat\mW^{(\vtheta)}_n\vf_{\vtheta}(\vx)$.
\end{restatable}
\Cref{prop:mle_invariance}
shows that the model $p(y;\hat\mW^{(\hat\vtheta)}_n\vf_{\hat\vtheta}(\vx))$ learned by~\Cref{pb:mle} is independent of $\mL$, \ie \textit{the learned model is the same
for disentangled and linearly entangled representations}. We thus expect both disentangled and linearly entangled representations to perform identically on downstream tasks.

\subsection{An advantage of disentangled representations}
We are now going to see how adding sparsity regularization to~\cref{pb:mle} favors the disentangled representation when the ground-truth data generating process is truly sparse.
\begin{assumption}[Data generation process]\label{ass:sparse_task}
    The input-label pairs are i.i.d. samples from the distribution $p(\vx, y):= p(y; \mW \vf_\vtheta(\vx))p(\vx)$, where $\mW \in \sR^{k \times m}$ is the ground-truth coefficient matrix such that $\normin{\mW}_{2,0} = \ell$.
\end{assumption}
To formalize the hypothesis that \emph{only a subset of the features $\vf_\vtheta(\vx)$ are actually useful to predict the target $y$}, we assume that the ground-truth coefficient matrix $\mW$ is column sparse, \ie $\normin{\hat\mW}_{2,0} = \ell < m$.
Under this assumption, it is natural to constrain the MLE as such:
\begin{problem}
    \hat\mW^{(\hat\vtheta, \ell)}_n := \argmax_{\normin{\tilde\mW}_{2,0} \leq \ell} \sum_{(\vx, y) \in \gD} \log p(y; \tilde\mW\vf_{\hat\vtheta}(\vx)) 
    \label{pb:mle_l0}
    \enspace .
\end{problem}
To analyze the impact of this additional constraint on the generalization error, we consider both the estimation error (a.k.a. variance) and the approximation error (a.k.a. bias) separately~\citep[Chapter 4]{MohriRostamizadehTalwalkar18}.

\textbf{Estimation error.} The sparsity constraint of~\cref{pb:mle_l0} decreases the size of the hypothesis class considered to minimize the negative log-likelihood and should thus yield a decrease in estimation error for both entangled and disentangled representations (\ie reduce overfitting). Sparsity regularization is a well-understood approach to control the complexity of a predictor, see for example~\citet{Bickel_Ritov_Tsybakov2009,Lounici_Pontil_Tsybakov2010,MohriRostamizadehTalwalkar18}.

\textbf{Approximation error.} Disentangled and entangled representations differ in how the sparsity constraint of~\cref{pb:mle_l0} impacts their approximation errors. The following proposition will help us see how this regularization favors disentangled representations over entangled ones.

\begin{restatable}{proposition}{sparsityBayesOptimal}
\label{prop:sparsity_bayes_optimal}
    Let $\hat\mW^{(\hat\vtheta)}_\infty$ be the (assumed unique) solution of the population-based MLE,
    \begin{math}
        \argmax_{\tilde\mW} \sE_{p(\vx,y)}\log p(y;
    \tilde\mW\vf_{\hat\vtheta}(\vx)) \label{pb:population_mle}
     \end{math}.
    If \cref{ass:lin_eq} (linear equivalence) \& \cref{ass:sparse_task} (data generating process) hold, $\hat\mW_\infty^{(\hat\vtheta)} = \mW\mL^{-1}$.
\end{restatable}

From \Cref{prop:sparsity_bayes_optimal}, one can see that if the representation $\vf_{\hat\vtheta}$ is disentangled  ($\mL = \mD\mP$), then
$$\normin{\hat\mW^{(\hat\vtheta)}_\infty}_{2,0} = \normin{\mW (\mD \mP)^{-1}}_{2,0} = \normin{\mW}_{2,0} = \ell\,.$$
Thus, the sparsity constraint in~\Cref{pb:mle_l0} does not exclude the population MLE estimator from its hypothesis class which means no approximation error is entailed (no bias). Contrarily, when $\vf_{\hat\vtheta}$ is linearly entangled, the population MLE might have more nonzero columns than the ground-truth (since $\mL^{-1}$ might destroy the sparsity of $\mW$), and thus would be excluded from the hypothesis space of~\Cref{pb:mle_l0}, which means an approximation error is introduced.

\textbf{Conclusion.} The above points suggest that \textit{if the ground-truth task is sufficiently sparse, the disentangled representation should benefit from sparsity regularization (assuming the number of samples is low) because it reduces the estimation error (variance) without increasing the approximation error (bias).} In contrast, an entangled representation might not benefit from sparsity regularization if the increase in approximation error is more important than the reduction in estimation error.

\textbf{Empirical validation (\Cref{fig:sparsity-disentanglement-gains}).} We now present a simple simulated experiment that illustrates the above claim {that \textit{disentangled representations coupled with sparsity regularization can yield better generalization.}}
\Cref{fig:sparsity-disentanglement-gains} compares the generalization performances of $L_1$ and $L_2$-penalized linear regressions \citep{Tibshirani1996,Hoerl1970}, computed on the top of both disentangled and linearly entangled representations, which are frozen during training.
$L_1$-penalized linear regression coupled with the disentangled representation yields better generalization than other alternatives when $\ell / m = 5\%$ and when the number of samples is very small.
One can also see that disentanglement, sparsity regularization, and sufficient sparsity in the ground-truth data generating process are necessary for significant improvements, in line with our discussion.
Lastly, all methods yield similar performance when the number of samples grows. More details and discussions can be found in \Cref{app:dis_lass_gen}.


\section{Sparse Multi-Task Learning for Disentanglement}
\label{sub:disentanglement_sparse_mtl}
In \Cref{sec:dis+sparse=gen}, we argued that disentangled representations can improve generalization when combined with sparse task-specific predictors, but we did not mention how to obtain a disentangled representation in the first place. In this section, we first provide a new identification result (\Cref{thm:disentanglement_via_optim}, \Cref{sec:identifiability}), which states that in the multi-task learning setting, regularizing the task-specific predictors to be sparse can yield disentangled representations. Then, in \Cref{sub:tractable_bilevel}, we provide a practical way to learn disentangled representations motivated by our identifiability result.

\subsection{Task \& data generating process}
\label{sub:task_data_gen}
Throughout this section, we assume the learner is given a set of $T$ datasets $\{\gD_1, \dots, \gD_T\}$ where each dataset $\gD_t := \{(\vx^{(t, i)}, y^{(t, i)})\}_{i=1}^n$ consists of $n$ couples of input $\vx \in \sR^{d}$ and label $y \in \gY$. The set of labels $\gY$  might contain either class indices or real values, depending on whether we are concerned with classification or regression tasks.

Our theory relies on the assumption that, for each task $t$, the dataset $\gD_t$ is made of i.i.d. samples from the distribution
\begin{align}
    p(\vx, y \mid \mW^{(t)}):= p(y; \mW^{(t)} \vf_\vtheta(\vx))p(\vx \mid \mW^{(t)})\,,
\end{align}
where $\mW^{(t)} \in \sR^{k \times m}$ is the task-specific ground-truth coefficient matrix. We emphasize that the representation $\vf_\vtheta$ is shared across all the tasks while the coefficient matrices $\mW^{(t)}$ are task-specific. Also note that the distribution over $\vx$ is allowed to change from one task to another. However, we assume that its support, $\gX$, is fixed across tasks.

\begin{figure}[t]
    \centering
    \includegraphics[width=1\columnwidth]{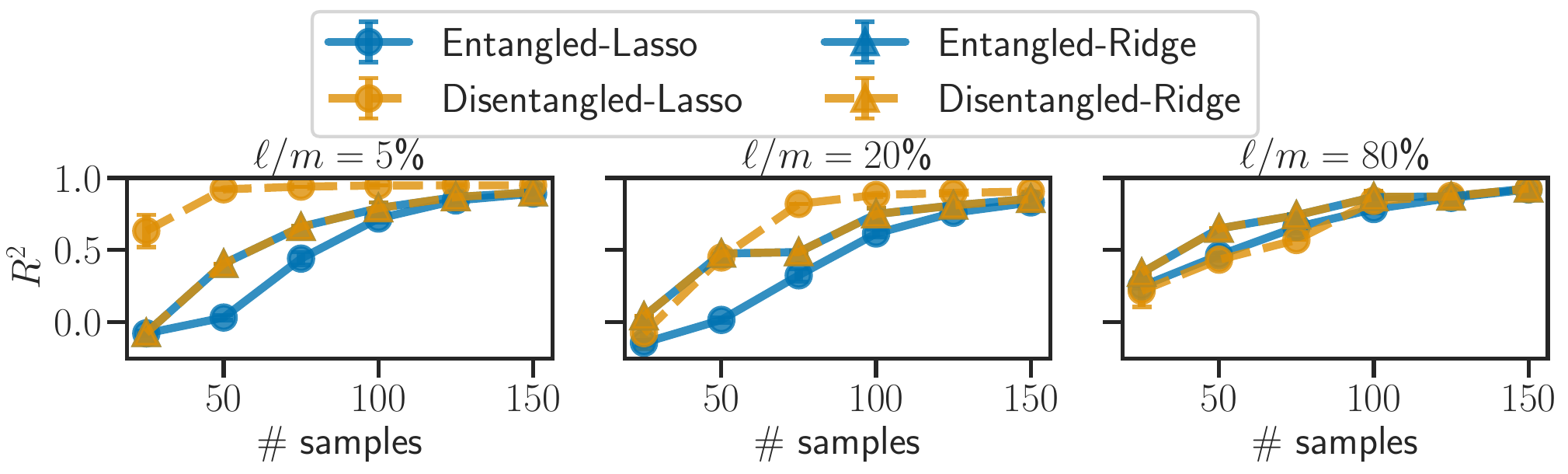}
    \vspace{-0.75cm}
    \caption{Test performance for the entangled and disentangled representation using Lasso and Ridge regression.
    All the results are averaged over $10$ seeds, with standard error shown in error bars.}
    \label{fig:sparsity-disentanglement-gains}
    \vspace{-0.1cm}
\end{figure}

We further assume that the task-specific matrices $\mW^{(t)}$ are i.i.d. samples from some probability measure $\sP_\mW$ with support $\gW$. We will see in~\cref{sec:assumptions} that the most critical assumptions of our theory concern $\sP_\mW$.

\subsection{Main identifiability result} \label{sec:identifiability}

We are now ready to show the main theoretical result of this work, which provides a bi-level optimization problem for which the optimal representations are guaranteed to be disentangled. It assumes infinitely many tasks are observed, with task-specific ground-truth matrices $\mW$ sampled from $\sP_\mW$. We denote by $\hat\mW^{(\mW)}$ the task-specific estimator of $\mW$. We delay the presentation of its technical assumptions to~\cref{sec:assumptions}. See~\Cref{app:main_thm} for a proof. 

\begin{restatable}[Sparse multi-task learning for disentanglement]{theorem}{disViaOptInner}\label{thm:disentanglement_via_optim}
Let $\hat\vtheta$ be a minimizer of
\begin{align}
        \min_{\hat\vtheta}\ &\sE_{\sP_\mW}\sE_{p(\vx, y \mid \mW)}-\log p(y; \hat\mW^{(\mW)}\vf_{\hat\vtheta}(\vx))
        \label[pb_multiline]{pb:OG_problem_inner} \\
        \mathrm{s. t.}\ &
        \ \
        \hat\mW^{(\mW)} \in \!\!\!\!\!\!\!\!\! \argmin_{\substack{\tilde\mW\  \textnormal{s.t.} \\ ||\tilde\mW||_{2,0} \leq ||\mW||_{2,0}}}
        \!\!\!\!\!\!\!\!\! \sE_{p(\vx, y \mid \mW)}-\log p(y; \tilde\mW\vf_{\hat\vtheta}(\vx))
        \enspace , \nonumber
\end{align}
where the constraint holds for all $\mW \in \gW$ and where $\sP_\mW$ and $p(\vx, y \mid \mW)$ are described in \cref{sub:task_data_gen}.
Under~\Cref{ass:eta_ident,ass:suff_var_rep,ass:suff_var_task,ass:intra_supp_task_var,ass:suff_support} and if $\vf_{\tilde\vtheta}$ is continuous for all $\tilde\vtheta$, $\vf_{\hat\vtheta}$ is disentangled w.r.t. $\vf_\vtheta$~(\Cref{def:disentanglement}).
\end{restatable}
Intuitively, this optimization problem effectively selects a representation $\vf_{\hat\vtheta}$ that (i) allows a perfect fit of the data distribution, and (ii) allows the task-specific estimators $\hat\mW^{(\mW)}$ to be as sparse as the ground-truth $\mW$. The theorem guarantees that such a representation must be disentangled.

Under the same assumptions and with the same disentanglement guarantees, \Cref{thm:disentanglement_via_optim_outer} in \Cref{app:ident_theory} presents a variation of \Cref{pb:OG_problem_inner} which enforces the weaker constraint $\sE_{\sP_\mW}\normin{\hat\mW^{(\mW)}}_{2,0} \leq \sE_{\sP_\mW}\normin{\mW}_{2,0}$, instead of $\normin{\hat\mW^{(\mW)}}_{2,0} \leq \normin{\mW}_{2,0}$ for each task $\mW$ individually.

\textbf{Characteristic features of our theory.} (i) Contrary to most identifiability results for disentanglement (\cref{sec:related_work}), we do not assume the observations $\vx$ are generated by transforming a latent random vector $\vz$ through a bijective decoder $\vg$. Instead, we assume the existence of a \textbf{not necessarily invertible ground-truth feature extractor} $\vf_\vtheta(\vx)$ from which the labels can be predicted using only a subset of its components in every task. (ii) Most previous works make assumptions about the distribution of latent factors, e.g., (conditional) independence, exponential family or other parametric assumptions. In contrast, we make no such assumption except a rather weak assumption on the support of the ground-truth features~(\Cref{ass:suff_var_rep}). Crucially, this allows for \textbf{statistically dependent latent factors}, which we explore empirically in \Cref{sec:dis_exp_sub}.

\subsection{Assumptions of~\cref{thm:disentanglement_via_optim}}\label{sec:assumptions}
We now present the technical assumptions of \cref{thm:disentanglement_via_optim}.

Perhaps unsurprisingly, the parameters $\bm\eta$ have to be identifiable from $p(y; \bm\eta)$ in order for $\vf_\vtheta$ to be identifiable.

\begin{figure}[tb]
        \centering
        \includegraphics[width=0.4\linewidth]{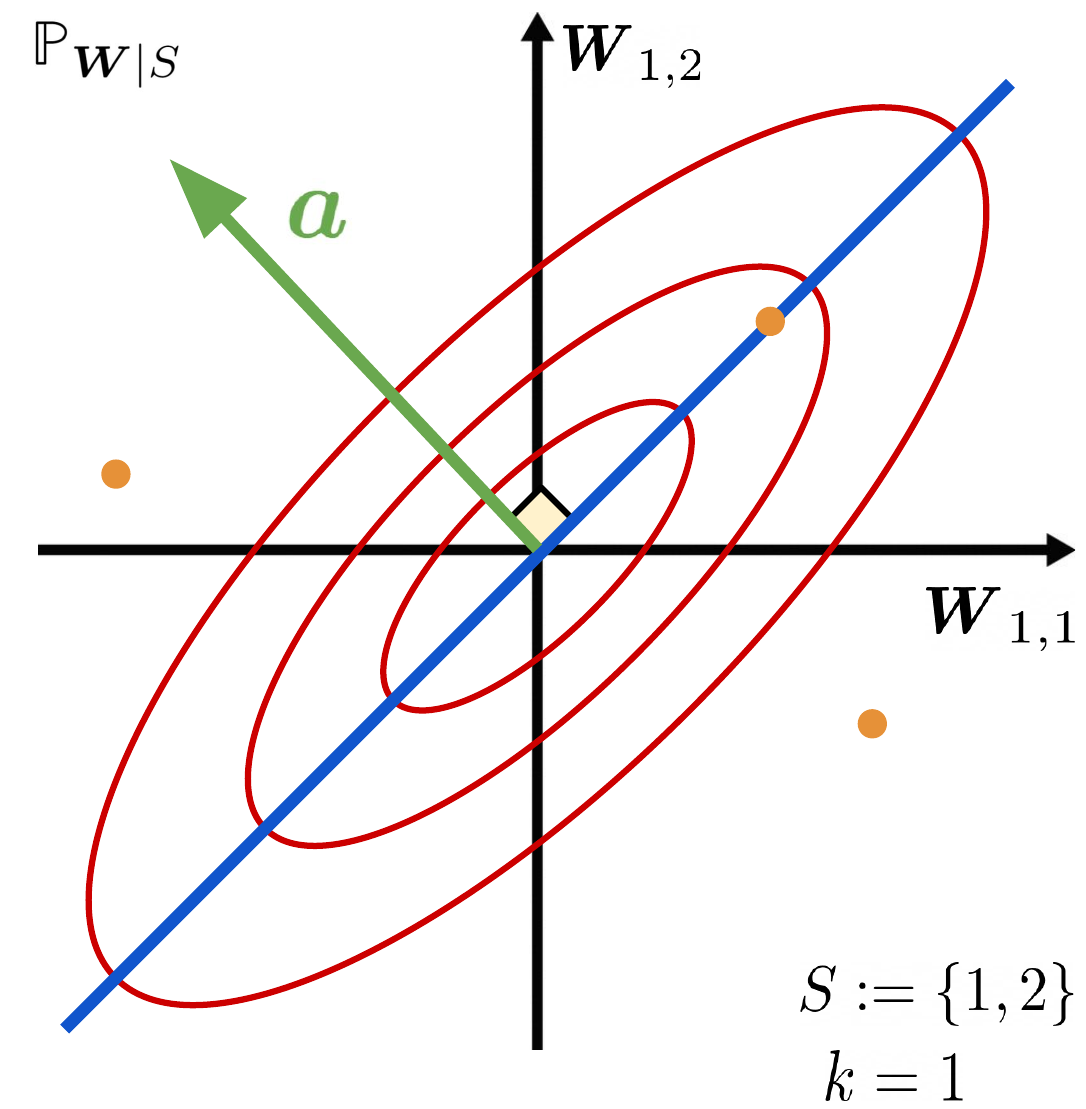}
        \caption{
                Illustration of~\Cref{ass:intra_supp_task_var} showing three examples of distribution $\sP_{\mW \mid S}$. The red distribution satisfies the assumption, but the blue and orange distributions do not. The red lines are level sets of a Gaussian distribution with full rank covariance.
                The blue line represents the support of a Gaussian distribution with a low-rank covariance. The orange dots represent a distribution with finite support.
                The green vector $\va$ shows that the condition is violated for both the blue and the orange distribution, since, in both cases, $\mW_{1,S}$ and $\va$ are orthogonal ($\mW_{1,S}\va = 0$) with probability greater than zero.
        }
        \label{fig:assumption_6}
\end{figure}

\begin{assumption}[Identifiability of $\bm\eta$ from $p(y; \bm\eta)$]\label{ass:eta_ident}
$ {\mathrm{KL}(p(y ; \bm{\eta}) \mid\mid p(y ; \tilde{\bm{\eta}}))} = 0\ \implies \bm\eta = \tilde{\bm{\eta}}$, where $\mathrm{KL}$ denotes the Kullback-Leibler divergence.
\end{assumption}
This property holds, \eg when $p(y ; \bm\eta)$ is a Gaussian in the usual $\mu, \sigma^2$ parameterization. Generally, it also holds for minimal parameterizations of exponential families~\citep{WainwrightJordan08}.

The following assumption requires the ground-truth representation $\vf_\theta(\vx)$ to vary enough such that its image cannot be trapped inside a proper subspace.

\begin{assumption}[Sufficient representation variability]\label{ass:suff_var_rep}
There exists $\vx^{(1)}, \dots, \vx^{(m)} \in \gX$ such that the matrix $\mF := [\vf_\vtheta(\vx^{(1)}), \dots, \vf_\vtheta(\vx^{(m)})]$ is invertible.
\end{assumption}

The following assumption requires that the support of {the distribution} $\sP_\mW$ is sufficiently rich.
\begin{assumption}[Sufficient task variability]\label{ass:suff_var_task}
There exists $\mW^{(1)}, \dots, \mW^{(m)} \in \gW$
and indices $i_1, \dots, i_{m} \in [k]$ such that the rows $\mW_{i_1, :}^{(1)}, \dots, \mW_{i_{m}, :}^{(m)}$ are linearly independent.
\end{assumption}
Under~\Cref{ass:eta_ident,ass:suff_var_rep,ass:suff_var_task}, the representation $\vf_\vtheta$ is identifiable up to linear equivalence (see~\Cref{thm:linear_ident} in~\Cref{app:ident_theory}). Similar results were shown by~\citet{roeder2020linear,ahuja2022towards}. The next assumptions will guarantee disentanglement.

In order to formalize the intuitive idea that most tasks do not require all features, we will denote by $S^{(t)}$ the support of the matrix $\mW^{(t)}$, \ie
$$S^{(t)} := \{j \in [m] \mid \mW^{(t)}_{: j} \not= \bm 0 \}\,.$$ In other words, $S^{(t)}$ is the set of features which are useful to predict $y$ in the $t$-th task; note that it is unknown to the learner.
For our analysis,
we decompose $\sP_\mW$ as
\begin{equation}
    \sP_\mW = \sum_{S \in \gP([m])} p(S)\sP_{\mW \mid S} \,,
\end{equation}
where $\gP([m])$ is the collection of all subsets of $[m]$, $p(S)$ is the probability that the support of $\mW$ is $S$ and $\sP_{\mW \mid S}$ is the conditional distribution of $\mW$ given that its support is $S$. Let $\gS$ be the support of the distribution $p(S)$, \ie $\gS := \{S \in \gP([m]) \mid p(S) > 0\}$. The set $\gS$ will have an important role in \Cref{ass:suff_support}.

The following assumption requires that $\sP_{\mW \mid S}$ does not concentrate mass on certain proper subspaces.
\begin{figure}[tb]
        \includegraphics[width=1\linewidth]{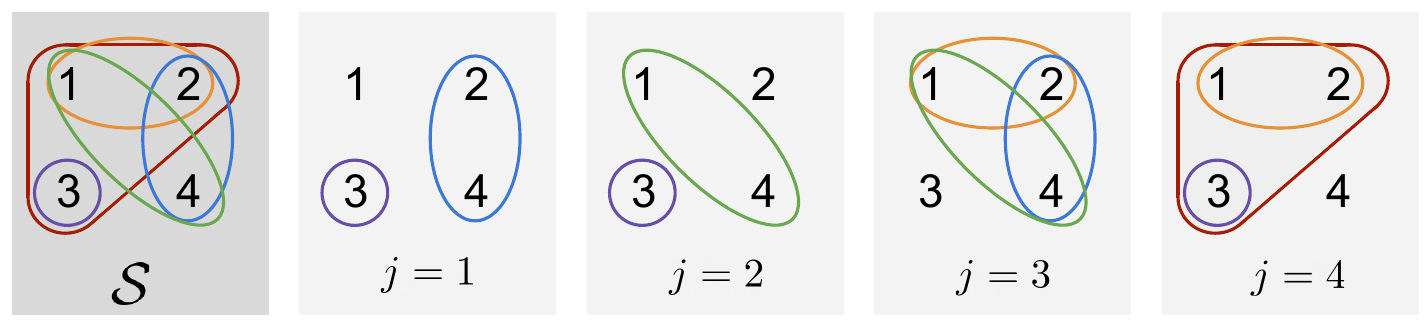}
        \vspace{-0.5cm}
        \caption{
        The leftmost figure represents $\gS$, the set of task supports observed under the ground-truth distribution $p(S)$.
        The other figures form a verification that~\Cref{ass:suff_support} holds for $\gS$.
        }
        \label{fig:assumption_7}
\end{figure}
\begin{assumption}[Intra-support sufficient task variability]\label{ass:intra_supp_task_var}
For all $S \in \gS$ and all $\va \in \sR^{|S|} \backslash \{0\}$,
$$\sP_{\mW \mid S}\{\mW \in \sR^{k \times m} \mid \mW_{:S}\va = \bm0\} = 0\,.$$
\end{assumption}
We illustrate the above assumption in the simpler case where $k=1$.
For instance, \Cref{ass:intra_supp_task_var}
holds when the distribution of $\mW_{1,S} \mid S$ has a density w.r.t. the Lebesgue measure on $\sR^{|S|}$, which is true for example when $\mW_{1, S} \mid S \sim \gN(\bm0, \bm\Sigma)$ and the covariance matrix $\bm\Sigma$ is full rank {(red distribution in \Cref{fig:assumption_6})}.
However, if $\bm\Sigma$ is not full rank, the probability distribution of $\mW_{1,S} \mid S$ concentrates its mass on a proper linear subspace $V \subsetneq \sR^{|S|}$, which violates~\Cref{ass:intra_supp_task_var} {(blue distribution in~\Cref{fig:assumption_6})}.
{
    Another important counter-example is when $\sP_{\mW \mid S}$ concentrates some of its mass on a point $\mW^{(0)}$, \ie $\sP_{\mW \mid S}\{\mW^{(0)}\} > 0$ (orange distribution in~\Cref{fig:assumption_6}). We provide a concrete numerical example of what can go wrong when the support of the $\sP_{\mW \mid S}$ is finite in \Cref{app:counter_example_suff_var}.}
    {
        Interestingly, there are distributions over $\mW_{1,S} \mid S$ that do not have a density w.r.t. the Lebesgue measure, but still satisfy \Cref{ass:intra_supp_task_var}.
        This is the case, e.g., when $\mW_{1,S} \mid S$ puts uniform mass over a $(|S|-1)$-dimensional sphere embedded in $\sR^{|S|}$ and centered at zero. See~\Cref{app:discuss_assumption} for a justification.
    }

The following assumption requires that the support $\gS$ of $p(S)$ is ``rich enough''.
\begin{assumption}[Sufficient variability of the task supports]\label{ass:suff_support}
For all $j \in [m]$,
$$\bigcup_{S \in \gS \mid j \not\in S} S = [m] \setminus \{j\}\,.$$
\end{assumption}
%
Intuitively, \Cref{ass:suff_support} requires that, for every feature $j$, one can find a set of tasks such that their supports cover all features except $j$ itself. \Cref{fig:assumption_7}  shows an example of $\gS$ satisfying~\Cref{ass:suff_support}. Appendix~\ref{app:proba_suff_support} provides a probabilistic argument showing that Assumption~\ref{ass:suff_support} holds ``in most cases'' when the number of supports is very large.
That being said, we conjecture that removing this assumption would yield a form of \textit{partial disentanglement} resembling the one developed by \citet{lachapelle2022partial} in which some groups of latent factors would remain entangled.

%
\subsection{Tractable bilevel optimization problems for sparse multitask learning}\label{sec:sparse_bilevel}
\label{sub:tractable_bilevel}
The proposed approach to jointly estimate the representation and the task-specific predictors relies on a bilevel optimization problem (\cref{pb:OG_problem_inner}) that is intractable because of the non-convex constraints.
To obtain a tractable bi-level optimization problem, the $L_{2,0}$ constraints are replaced by their convex relaxations in the penalized form,
which are also known to promote group sparsity~\citep{Argyriou2008}:
\begin{align}
        \min_{\hat\vtheta}\ &\ -\frac{1}{Tn}\sum_{t=1}^T\sum_{(\vx, y) \in \gD_t} \log p(y; \hat\mW^{(t)}\vf_{\hat\vtheta}(\vx)) \label[pb_multiline]{pb:OG_problem_inner_relax}
        \\
        \mathrm{s. t.}\ &
        \,
        \hat\mW^{(t)}
        \in
        \argmin_{\tilde\mW } \frac{1}{n}\sum_{(\vx, y) \in \gD_t} \!\!\!\!\!-\log p(y; \tilde\mW\vf_{\hat\vtheta}(\vx)) \nonumber \\[-8pt]
        & \qquad\qquad\qquad\qquad\qquad\qquad\qquad\quad+ \lambda_{t} ||\tilde\mW||_{2,1}
        \enspace , \nonumber
\end{align}
where the constraint holds for all $t \in [T]$.
Following \citet{bengio2000,Pedregosa2016}, one can compute the (hyper)gradient of the outer function using implicit differentiation, even if the inner optimization problem is non-smooth \citep{Bertrand2020,Bolte2021,Malezieux2022,Bolte2022}.
Once the hypergradient is computed, one can optimize~\Cref{pb:OG_problem_inner_relax} with usual first-order methods \citep{Wright_Nocedal1999}.

Note that the quantity $\hat\mW^{(t)} \vf_{\hat\vtheta}(\vx)$
is invariant to simultaneous rescaling of $\hat\mW^{(t)}$ by a scalar and of $\vf_{\hat\vtheta}(\vx)$ by its inverse.
Thus, without constraints on $\vf_{\hat\vtheta}(\vx)$, $\normin{\hat\mW^{(t)}}_{2,1}$ can be made arbitrarily small.
This issue is similar to the one faced in sparse dictionary learning \citep{kreutz2003dictionary,Mairal_Ponce_Sapiro_Zisserman_Bach2008,Mairal_Bach_Ponce2009,Mairal2011}, where unit-norm constraints are usually imposed on dictionary columns.
In our case, since $\vf_{\hat\vtheta}$ is parametrized by a neural network, we suggest applying batch or layer normalization~\citep{batchnorm2015,layernorm}
to control the norm of $\vf_{\hat\vtheta}(\vx)$.
Since the number of relevant features might be task-dependent,
~\Cref{pb:OG_problem_inner_relax} has one regularization hyperparameter $\lambda_t$ per task.
However, in practice, we select $\lambda_{t} := \lambda$ for all $t \in [T]$ to limit the number of hyperparameters. We also use an adaptive scheme to have $\lambda$ in a reasonable range throughout training, which we explain in \Cref{app:arch_solver_hyperparams}.

\Cref{app:relaxing_outer} introduces a similar relaxation of \Cref{thm:disentanglement_via_optim_outer} (mentioned in \Cref{sec:identifiability}) in which the sparsity penalty appears in the outer problem instead of the inner problem. \Cref{app:exp_outer_reg} presents empirical results showing this alternative approach yields very similar results.


\textbf{Link with meta-learning.}
The bi-level formulation \Cref{pb:OG_problem_inner_relax} is closely related to
\emph{metric-based meta-learning} methods \citep{snell2017protonet,bertinetto2018r2d2}, where a shared representation $\vf_{\hat\vtheta}$ is learned across all tasks via simple task-specific predictors, such as linear classifiers.
In the general meta-learning setting \citep{Finn_Abbel_Levine2017}, one is given a large number of training datasets $(\gD_{t}^{\mathrm{train}})_{1\leq t \leq T}$, which usually only contain a small number of samples $n$. As opposed to the multi-task setting (\ie unlike in \Cref{sub:task_data_gen}), one is also given separate \emph{test datasets} $(\gD_{t}^{\mathrm{test}})_{1 \leq t \leq T} $ of $n'$ samples for each task $t$, to evaluate how well the learned model generalizes to new test samples. In meta-learning, the goal is to \emph{learn a learning procedure} that will generalize well on new unseen tasks.

{
Formally, metric-based meta-learning can be formulated as}
\begin{align}
        \min_{\hat\vtheta}\ \ &
        \frac{1}{Tn'}\sum_{t =1}^T
        \sum_{(\vx, y) \in \gD_t^{\mathrm{test}}}
            \gL_{\mathrm{out}}
            \big(\hat \mW_{\hat\vtheta}^{(t)}; f_{\hat \vtheta}(\vx), y
            \big)\label[pb_multiline]{eq:meta-learning-bi-level}
        \\
        \mathrm{s.t.}\ \ &
        \hat \mW_{\hat\vtheta}^{(t)}
        \in
        \argmin_{\tilde\mW}
        \frac{1}{n}\sum_{(\vx, y) \in \gD_t^{\mathrm{train}}}
        \gL_{\mathrm{in}}
        \big(
            \tilde\mW; f_{\hat \vtheta}(\vx), y
        \big) \enspace.
        \nonumber
\end{align}
The main difference between \Cref{pb:OG_problem_inner_relax} and \Cref{eq:meta-learning-bi-level} is that, in the latter, the inner and outer loss functions $\gL_{\mathrm{in}}$ and $\gL_{\mathrm{out}}$ are not evaluated on the same dataset. \cref{sec:few_shot_learning} shows experiments with a meta-learning variant of \cref{pb:OG_problem_inner_relax} based on group Lasso multiclass
SVM predictors.

\section{Related Work}\label{sec:related_work}
\begin{figure*}[t]
    \centering
    \includegraphics[width=0.7\linewidth]{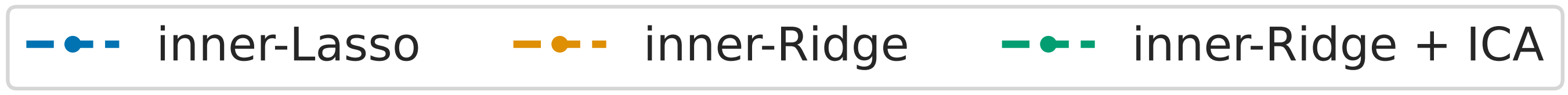}
    \includegraphics[width=0.85\linewidth]{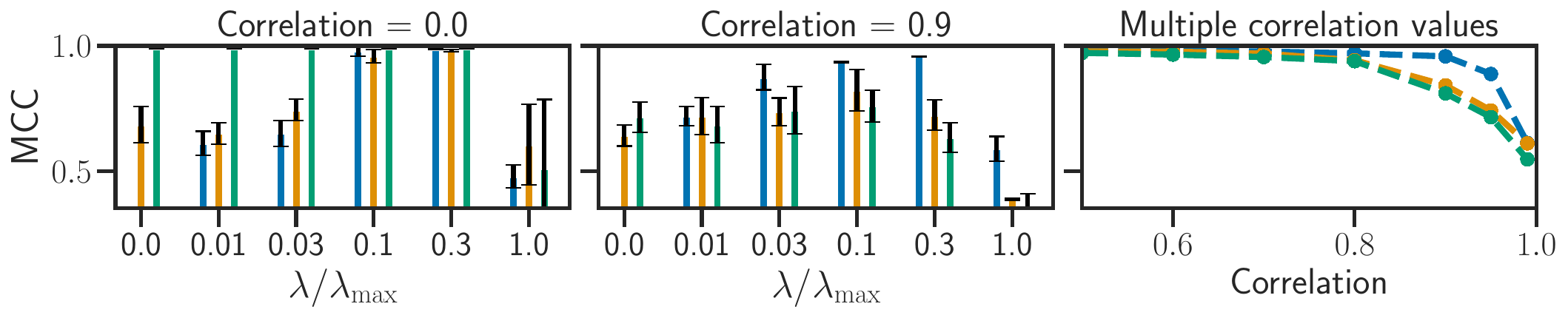}
    \includegraphics[width=0.85\linewidth]{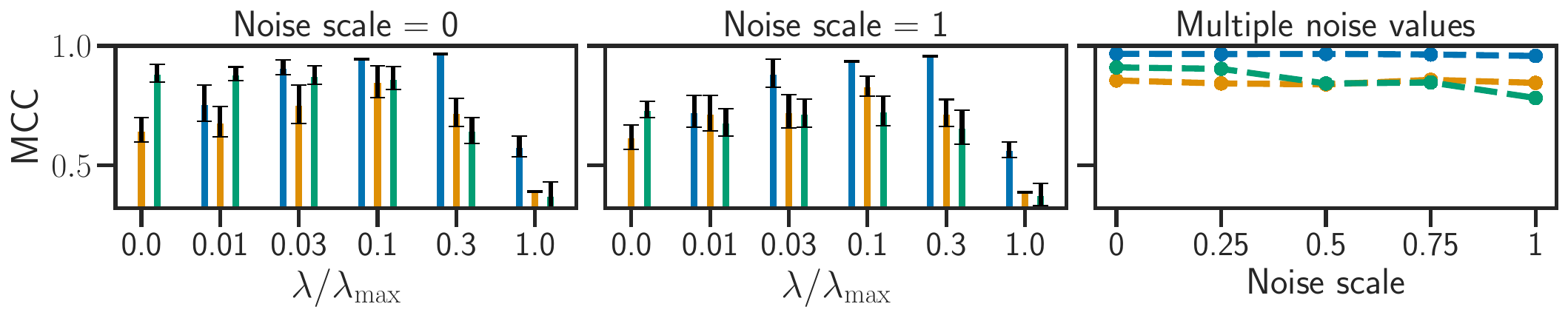}
    \vspace{-0.5cm}
    \caption{Disentanglement performance (MCC) for all three methods considered as a function of the regularization parameter (left and middle).
    Varying level of correlation between latents (top) and noise on the latents (bottom).
    The right columns show performances of the best hyperparameter for different values of correlation and noise. We explain what is $\lambda_{\max}$ in \Cref{app:arch_solver_hyperparams}.}
    \vspace*{-1.em}
    \label{fig:3dshape_mcc}
\end{figure*}
\looseness=-1
\textbf{Disentanglement.} Since the work of \citet{bengio2013representation}, many methods have been proposed to learn disentangled representations based on various heuristics~\citep{Higgins2017betaVAELB,tcvae, pmlr-v80-kim18b,kumar2018variational,Bouchacourt_Tomioka_Nowozin_2018}. Following the work of~\citet{pmlr-v97-locatello19a}, which highlighted the lack of identifiability in modern deep generative models, many works have proposed more or less weak forms of supervision motivated by identifiability analyses~\citep{pmlr-v119-locatello20a,slowVAE,vonkugelgen2021selfsupervised,ahuja2022properties,ahuja2022towards,zheng2022on}. A similar line of work have adopted the causal representation learning perspective~\citep{lachapelle2022disentanglement,lachapelle2022partial,CITRIS,lippe2022icitris,ahuja2022sparse,yao2022learning,brehmer2022weakly}.

\looseness=-1
The problem of identifiability was well known among the \textit{independent component analysis}~(ICA) community~\citep{ICAbook,HYVARINEN1999429} which came up with solutions for general nonlinear mixing functions by leveraging auxiliary information~\citep{TCL2016,PCL17,HyvarinenST19,iVAEkhemakhem20a,ice-beem20}. Another approach is to consider restricted hypothesis classes of mixing functions~\citep{TalebJutten1999,gresele2021independent,zheng2022on,moran2022identifiable}. 
\citet{Locatello2020Disentangling} proposed a semi-supervised learning approach to disentangle in cases where a few samples are labelled with the values of the factors of variations themselves. This is different from our approach as the labels that we consider can be sampled from some $p(y; \mW\vf_{\hat\vtheta}(\vx))$, which is more general. \citet{ahuja2022towards} consider a setting similar to ours, but they rely on the independence and non-gaussianity of the latent factors for disentanglement using linear ICA. 
See the end of~\cref{sec:identifiability} for further discussions on how our theory distinguishes itself from most methods cited above. 

\looseness=-1
\textbf{Multi-task, transfer \& invariant learning.}
While the statistical advantages of multi-task representation learning are well understood~\citep{Lounici_Pontil_Tsybakov2010,Lounici2011oracle,benefitOfMTRL2016}, 
the theoretical benefits of disentanglement for transfer learning are not clearly established (apart from \citealt{Zhangetal22b}).
Some works have investigated this question empirically and obtained both positive~\citep{steenkiste2019DisForAbstractReasoning,miladinovic2019disentangledODE,dittadi2021sim2real_dis} and negative results~\citep{pmlr-v97-locatello19a,montero2021disGen}. Invariant risk minimization~\citep{arjovsky2020invariant,IRMgames,krueger2021outofdistribution,lu2021nonlinear} aims at learning a representation that elicits a single predictor that is optimal for all tasks.
This differs from our approach which learns one predictor per task.

\textbf{Dictionary learning and sparse coding.}
We contrast our approach, which jointly learns a \textit{dense representation} and sparse task-specific predictors (\Cref{pb:OG_problem_inner_relax}), with the line of work which consists in learning \textit{sparse representations}~\citep{Chen1998atomic,sparseComponentAnalysisSurvey}.
For instance, sparse dictionary learning \citep{Mairal_Bach_Ponce2009,Mairal2011,sparseCodingForMTL2013} 
is an unsupervised technique that aims at learning a dictionary of \textit{atoms} used to reconstruct inputs via sparse linear combinations of its elements. 
The representation of a single input consists of the coefficients of the linear combination of atoms that minimizes a sparsity-regularized reconstruction loss.
In the case of supervised dictionary learning \citep{Mairal_Ponce_Sapiro_Zisserman_Bach2008}, an additional (potentially expressive) classifier is learned on top of that representation.
This large literature has led to a wide variety of estimators: for instance, \citet[Eq. 4]{Mairal_Ponce_Sapiro_Zisserman_Bach2008}, which minimizes the sum of the classification error and the approximation error of the code, or \citet{Mairal2011}.
introducing bi-level formulations.
%
While sharing similar optimization challenges, our method is conceptually different and computes the representation of a single input $\vx$ by evaluating the learned function $\vf_{\hat\vtheta}$.

\section{Experiments}
\label{sec:dis_exp}
We present experiments on disentanglement and few-shot learning. Our implementation relies on \texttt{jax} and \texttt{jaxopt} \citep{jax2018github,jaxopt} and is available 
here: \url{https://github.com/tristandeleu/synergies-disentanglement-sparsity}.
%
\subsection{Disentanglement in 3D Shapes} \label{sec:dis_exp_sub}
We now illustrate~\Cref{thm:disentanglement_via_optim} by applying~\Cref{pb:OG_problem_inner_relax} to
tasks generated using the 3D Shapes dataset~\citep{3dshapes18}.

\emph{Data generation.} For all tasks $t$, the labelled dataset $\gD_t = \{(\vx^{(t,i)}), y^{(t,i)})\}_{i=1}^n$ is generated by first sampling the ground-truth latent variables $\vz^{(t,i)}$ 
i.i.d. according to some distribution $p(\vz)$, while the corresponding input is obtained doing $\vx^{(t,i)} := \vf_\vtheta^{-1}(\vz^{(t,i)})$ ($\vf_\vtheta$ is invertible in 3D Shapes). Then, a sparse weight vector $\vw^{(t)}$ is sampled randomly to compute the labels of each example as $y^{(t,i)} := \vw^{(t)} \cdot \vz^{(t,i)} + \epsilon^{(t,i)}$, where $\epsilon^{(t,i)}$ is independent Gaussian noise. \Cref{fig:3dshape_mcc} explores various choices of $p(\vz)$ by varying the level of correlation between the latent variables and by varying the level of noise on the ground-truth latents.
See~\Cref{app:dis_exp} for more details about the data generating process and \cref{fig:distributions_over_latents} to visualize various $p(\vz)$.

\emph{Algorithms.} In this setting where $p(y; \bm\eta)$ is a Gaussian with fixed variance, the inner problem of~\Cref{pb:OG_problem_inner_relax} amounts to Lasso regression, we thus refer to this approach as inner-Lasso. We also evaluate a simple variation of~\Cref{pb:OG_problem_inner_relax} in which the $L_1$ norm is replaced by an $L_2$ norm and refer to it as inner-Ridge.
In addition, we evaluate the representation obtained by performing linear ICA~\citep{comon1992} on the representation learned by inner-Ridge: the case $\lambda=0$ corresponds to the approach of~\citet{ahuja2022towards}.

\emph{Discussion.} \Cref{fig:3dshape_mcc} reports disentanglement performances of the three methods, as measured by the \textit{mean correlation coefficient}, or MCC~\citep{TCL2016, iVAEkhemakhem20a}~(\Cref{app:dis_exp}). In all settings, inner-Lasso obtains high MCC for some values of $\lambda$, being on par or surpassing the baselines. As the theory suggests, it is robust to high levels of correlations between the latents, as opposed to inner-Ridge with ICA which is very much affected by strong correlations (since ICA assumes independence). We can also see how additional noise on the latent variables hurts inner-Ridge with ICA while leaving inner-Lasso unaffected. \Cref{fig:3dshape_influ_corr_r} in \Cref{app:dis_exp} shows that all methods find a representation which is linearly equivalent to the ground-truth representation, except for very large values of $\lambda$. 
{\Cref{app:experiment_violation} studies empirically to what extent inner-Lasso is robust to violations of \Cref{ass:suff_support}, \Cref{app:visual_eval} presents a visual evaluation of disentanglement and \Cref{app:dci_metrics} reports the DCI metric~\citep{eastwood2018framework} on the same experiments.} We did not explore hyperparameter selection in this work, which is a difficult problem for disentanglement because a goodness-of-fit score evaluated on a held-out dataset will not be informative because of the lack of identifiability. Nevertheless, one can use heuristics such as the \textit{unsupervised disentanglement ranking} score proposed by \citet{Duan2020UDR}. 
%
%
%
\subsection{Sparse task-specific predictors in few-shot learning}\label{sec:few_shot_learning}
Despite the lack of ground-truth latent factors in standard few-shot learning benchmarks, we also evaluate sparse meta-learning objectives on the \textit{mini}ImageNet dataset \citep{vinyals2016matchingnet}. The purpose of this experiment is to show that the sparse formulation of standard metric-based meta-learning techniques reaches similar performance while using a fraction of the features (\Cref{fig:meta_learning}, right).

Inspired by \citet{Lee_Maji_Ravichandran_Soatto2019meta}, where the task-specific classifiers are multiclass support-vector machines (SVMs, \citealt{Crammer_Singer2001}), we propose to use group Lasso penalized multiclass SVMs, to introduce sparsity in the classifiers.
Using the notation of \Cref{eq:meta-learning-bi-level}, we choose
    \begin{align}
        \gL_{\mathrm{in}}(\mW;  f_{\hat \vtheta}(\vx_i), \vy_i)
        &=
        \max_ {l \in [k]}
            \left (
            ( \mW_{\my_i:} - \mW_{l:} ) \cdot f_{\hat \vtheta}(\vx_i)
            - \mY_{il} \right )
            \nonumber
            \\
            &+ \lambda_1 \normin{\mW}_{2, 1}
            + \tfrac{\lambda_2}{2} \normin{\mW}^2 \enspace,
        \label[pb_multiline]{pb:primal_multiclass_group_svm_main}
        \\
        \gL_{\mathrm{out}}(\mW;  f_{\hat \vtheta}(\vx_i), \vy_i)
        &=  \mathrm{CE} (\mW f_{\hat \vtheta}(\vx_i), \mY_{i:}) \enspace,
    \end{align}

with $\vY \in \mathbb{R}^{n \times k}$ the one-hot encoding of $\vy \in \sR^n$ and $\mathrm{CE}$ the cross-entropy. The difference with \citet{Lee_Maji_Ravichandran_Soatto2019meta} is the sparsity-promoting term $\normin{\mW}_{2, 1}$, which makes the bi-level optimization problem harder to solve.
That is why we propose solving the dual \citep[Chap. 5]{Boyd_Vandenberghe2004} of this inner optimization problem, which writes
\begin{align}
    &
     \min_{\mLambda \in \sR^{n \times k}}
    \frac{1}{\lambda_2}
    \sum_{j=1}^m
        \normin{\mathrm{BST} \left ( (\mY - \mLambda)^{\top} \mF_{:j}, \lambda_1 \right )}^2
    + \langle \mY, \mLambda \rangle
    \nonumber
    \\
    &
    \mathrm{s. t.} \;
    \forall i, l, \in [n] \times [k],
    \;
    \sum \limits_{l'=1}^k \mLambda_{il'}=1
    \; \text{ and } \;
        \mLambda_{il} \geq 0 \enspace,
        \label[pb_multiline]{pb:dual_multiclass_group_svm_main}
\end{align}
with $\mathrm{BST} : (\va, \tau) \mapsto \left( 1 - {\tau}/{\norm{\va}} \right)_+ \va$ is the block~soft-thresholding operator, $\vF \in \sR^{n \times m}$ the \mbox{concatenation} of $\{\vf_{\hat \vtheta}(x)\}_{(x, y) \in \gD^{\mathrm{train}}}$.
\looseness=-1
~In~\mbox{addition},~the \mbox{primal-dual} link writes, $\forall j \in [m],\ \mW_{:j} = \mathrm{BST} \left ( (\mY - \mLambda)^{\top} \mF_{:j}, \lambda_1 \right ) / \lambda_2$. The derivation of the dual can be found in \Cref{proof:dual_mtl_lasso},
Solving this kind of problem in the dual is standard in the SVM literature: it has been proven to be computationally advantageous  \citep{Hsieh2008}  when the number of features $m$ is significantly larger than the number of samples $n$ (here $m=1.6 \times 10^4$ and $n\leq 25$).
Details on how to solve and differentiate through \Cref{pb:dual_multiclass_group_svm_main} are in \Cref{app:meta-learning}.

\begin{figure}[t]
    \centering
    \includegraphics[width=1\linewidth]{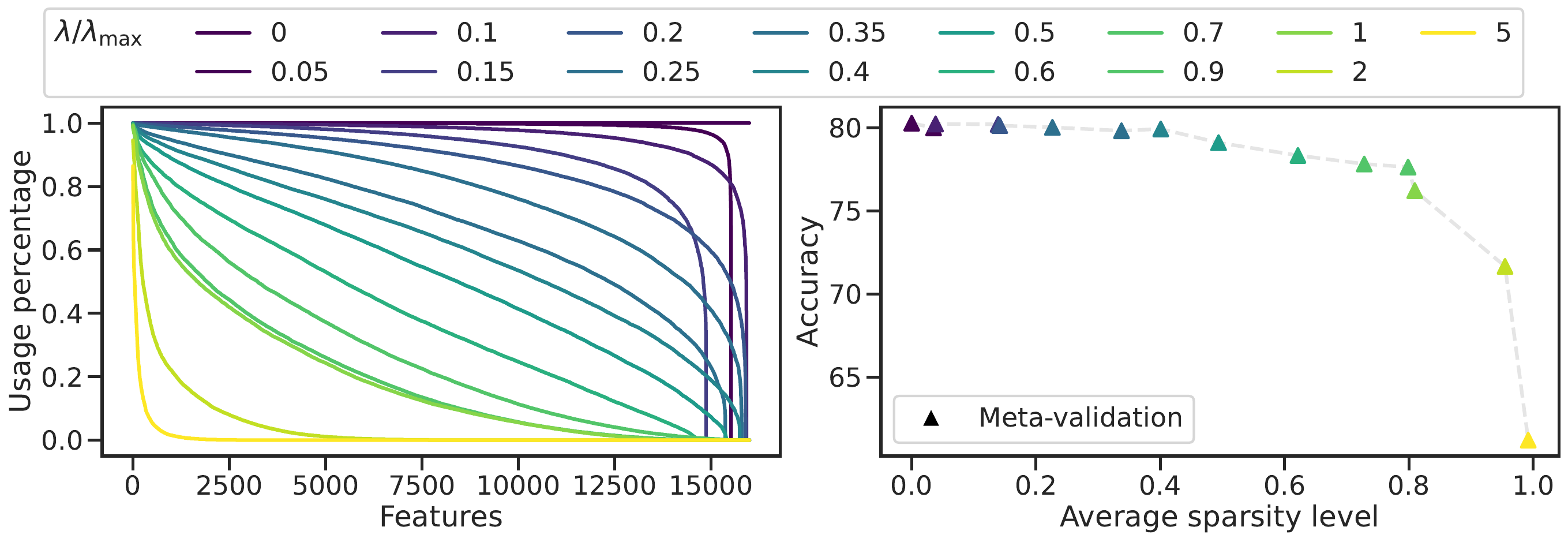}
    \vspace{-0.75cm}
    \caption{ \textit{Left.} Effect of sparsity on the percentage of tasks using specific features, with our meta-learning objective, on \textit{mini}ImageNet. \textit{Right.} The meta-validation accuracy of the meta-learning algorithm against the average level of sparsity in the task-specific predictor, for different values of $\lambda$.}
    \label{fig:meta_learning}
\end{figure}

\emph{Discussion.} In \Cref{fig:meta_learning} (right), we observe that the accuracy of the sparse meta-learning method on novel (meta-validation) tasks is similar to the dense counterpart ($\lambda = 0$), while using only a few of the features available (around $30\%$ of sparsity, with no impact on the performance). Naturally, the performance starts to drop as the sparsity level increases though, albeit being still competitive. We also report in \Cref{fig:meta_learning} (left) how frequently each feature in the learned representation is used by the task-specific predictors on meta-validation tasks (sorted by usage, for each $\lambda$). The gradual decrease in usage suggests that the features are reused in different contexts, across different tasks.
\section{Conclusion}
In this work, we investigated the synergies between sparsity, disentanglement and generalization. We showed that when the downstream task can be solved using only a fraction of the factors of variations, disentangled representations combined with sparse task-specific predictors can improve generalization~(\Cref{sec:dis+sparse=gen}). Our novel identifiability result (\Cref{thm:disentanglement_via_optim}) sheds light on how, in a multi-task setting, sparsity regularization on the task-specific predictors can induce disentanglement. This led to a practical bi-level optimization problem that was shown to yield disentangled representations on regression tasks based on the 3D Shapes dataset. 
Finally, we explored the connection between this bi-level formulation and meta-learning, and we showed how sparse task-specific predictors may achieve similar performance on unseen tasks with only a fraction of the features. Future work could explore identifiability in a more general setting where the task-specific predictors are potentially nonlinear, which should be applicable to more problems. 
\vfill\null


\section*{Acknowledgements}
This research was partially supported by the Canada CIFAR AI Chair Program, by an IVADO excellence PhD scholarship and by a Google Focused Research award. The experiments were in part enabled by computational resources provided by Calcul Quebec and Compute Canada. Simon Lacoste-Julien is a CIFAR Associate Fellow in the Learning in Machines \& Brains program.
Sébastien Lachapelle and Quentin Bertrand would like to thank Samsung Electronics Co., Ldt. for funding this research. The authors would like to thank David Berger for insightful discussions in the early stage of this project.




\bibliography{biblio}
\bibliographystyle{icml2023}

\newpage
\onecolumn
\appendix

\addcontentsline{toc}{section}{Appendix} 
\part{Appendix} 

\parttoc 


\begin{table}[htbp]\caption{{Table of Notations.}}\vspace*{0.5ex}
    {
    \begin{center}
            \begin{tabular}{r c p{10cm} }
                \toprule
                \multicolumn{3}{c}{}\\
                \multicolumn{3}{c}{\underline{Norms \& pseudonorms }}\\
                \multicolumn{3}{c}{}\\
                $\norm{\cdot}$ &  &
                Euclidean norm on vectors and Frobenius norm on matrices \\
                $\norm{\mA}_{2, 1}$ & $:=$ &
                $\sum_{j=1}^m \norm{\mA_{:j}}$ \\
                $\norm{\mA}_{2, 0}$ & $:=$ & $\sum_{j = 1}^m \mathbbm{1}_{\norm{\mA_{:j}} \neq 0} $, where $\mathbbm{1}$ is the indicator function.
                \\
                \multicolumn{3}{c}{}\\
                \multicolumn{3}{c}{\underline{Data}}\\
                \multicolumn{3}{c}{}\\
                $\vx \in \sR^{d}$ &  & Observations \\
                $\gX \subset \sR^{d}$ &  & Support of observations \\
                $y \in \sR $ &  & Target \\
                $\gY \subset \sR $ &  & Support of targets \\
                \multicolumn{3}{c}{}\\
                \multicolumn{3}{c}{\underline{Learned/ground-truth model}}\\
                \multicolumn{3}{c}{}\\
                $\mW \in \sR^{k \times m}$ &  & Ground-truth coefficients \\
                $\hat \mW \in \sR^{k \times m}$ &  & Learned coefficients \\
                $\vtheta$ &  & Ground-truth parameters of the representation \\
                $\hat \vtheta$ &  & Learned parameters of the representation \\
                $f_{\vtheta}: \sR^d \rightarrow \sR^m $ &  & Ground-truth representation \\
                $f_{\hat \vtheta} : \sR^d \rightarrow \sR^m $ &  & Learned representation \\
                $\bm\eta \in \sR^k $ &  & Parameter of the distribution $p(y; \bm\eta)$ \\
                $\sP_\mW$ & & Distribution over ground-truth coefficient matrices $\mW$ \\
                $S $ & $:=$ & $\{j \in [m] \mid \mW_{: j} \not= \bm 0 \}$ (support of $\mW$)\\
                $\sP_{\mW \mid S}$ & & Conditional distribution of $\mW$ given $S$. \\
                $p(S)$ & & Ground-truth distribution over possible supports $S$ \\
                $\gS$ & & Support of the distribution $p(S)$ \\
                \multicolumn{3}{c}{}\\
                \multicolumn{3}{c}{\underline{Optimization}}\\
                \multicolumn{3}{c}{}\\
                $W$ & & Primal variable \\
                $\Lambda$ & & Dual variable \\
                $h^* : \va$ & $\mapsto$ & $\sup_{\vb \in \sR^d} \langle \va, \vb \rangle - h(\vb)$, Fenchel conjugate of the function $h: \sR^d \rightarrow \sR$ \\
                $f \square g : \va$ & $\mapsto $ & $\min_b f(\va - \vb) + g(\vb)$ , inf-convolution of the functions $f$ and $g$ \\
                $\mathrm{BST}: (\va, \tau)$  & $\mapsto$ & $\left( 1 - {\tau}/{\norm{\va}} \right)_+ \va$, block soft-thresholding operator\\
                \bottomrule
            \end{tabular}
        \end{center}
        \label{tab:TableOfNotation}
    }
\end{table}

\newpage

\section{Proofs of~\Cref{sec:dis+sparse=gen}}\label{app:mle_invariance_proof}

\mleInvariance*

\begin{proof}
By definition of $\hat\mW^{(\hat\vtheta)}$, we have that, for all $\hat\mW \in \sR^{k \times m}$,
\begin{align}
    \sum_{(\vx, y) \in \gD} \log p(y; \hat\mW^{(\hat\vtheta)}\vf_{\hat\vtheta}(\vx)) &\geq \sum_{(\vx, y) \in \gD} \log p(y; \hat\mW\vf_{\hat\vtheta}(\vx)) \\
    \sum_{(\vx, y) \in \gD} \log p(y; \hat\mW^{(\hat\vtheta)}\mL\vf_{\vtheta}(\vx)) &\geq \sum_{(\vx, y) \in \gD} \log p(y; \hat\mW\mL\vf_{\vtheta}(\vx)) \,.
\end{align}
Because $\sR^{k \times m}\mL = \sR^{k \times m}$, we have that, for all $\hat\mW \in \sR^{k \times m}$,
\begin{align}
    \sum_{(\vx, y) \in \gD} \log p(y; \hat\mW^{(\hat\vtheta)}\mL\vf_{\vtheta}(\vx)) &\geq \sum_{(\vx, y) \in \gD} \log p(y; \hat\mW\vf_{\vtheta}(\vx)) \,,
\end{align}
which is to say that $\hat\mW^{(\vtheta)} = \hat\mW^{(\hat\vtheta)}\mL$, or put differently, $\hat\mW^{(\hat\vtheta)} = \hat\mW^{(\vtheta)}\mL^{-1}$. It implies
\begin{align}
    \hat\mW^{(\hat\vtheta)}\vf_{\hat\vtheta}(\vx) = \hat\mW^{(\vtheta)}\mL^{-1}\mL\vf_{\vtheta}(\vx) = \hat\mW^{(\vtheta)}\vf_{\vtheta}(\vx) \, , \label{eq:mle_invariance}
\end{align}
which is what we wanted to show.
\end{proof}

\sparsityBayesOptimal*

\begin{proof}
By definition of $\hat\mW_\infty^{(\hat\vtheta)}$, we have that, for all $\tilde\mW \in \sR^{k \times m}$,
\begin{align}
    \sE_{p(\vx,y)} \log p(y; \hat\mW_\infty^{(\hat\vtheta)}\vf_{\hat\vtheta}(\vx)) &\geq \sE_{p(\vx,y)} \log p(y; \tilde\mW\vf_{\hat\vtheta}(\vx)) \\
    \sE_{p(\vx,y)} \log p(y; \hat\mW_\infty^{(\hat\vtheta)}\mL\vf_{\vtheta}(\vx)) &\geq \sE_{p(\vx,y)} \log p(y; \tilde\mW\mL\vf_{\vtheta}(\vx)) \,.
\end{align}
In particular, the inequality holds for $\tilde\mW:= \mW \mL^{-1}$, which yields
\begin{align}
    \sE_{p(\vx,y)} \log p(y; \hat\mW_\infty^{(\hat\vtheta)}\mL\vf_{\vtheta}(\vx)) &\geq \sE_{p(\vx,y)} \log p(y; \mW\vf_{\vtheta}(\vx)) \\
    0 &\geq \sE_{p(\vx,y)}\left[ \log p(y; \mW\vf_{\vtheta}(\vx)) -  \log p(y; \hat\mW_\infty^{(\hat\vtheta)}\mL\vf_{\vtheta}(\vx))\right] \\
    0 &\geq \sE_{p(\vx)}\text{KL}(p(y; \mW\vf_{\vtheta}(\vx)) \mid\mid p(y; \hat\mW_\infty^{(\hat\vtheta)}\mL\vf_{\vtheta}(\vx))) \,.
\end{align}
Since the KL is always non-negative, we have that,
\begin{align}
    \sE_{p(\vx)}\text{KL}(p(y; \mW\vf_{\vtheta}(\vx)) \mid\mid p(y; \hat\mW_\infty^{(\hat\vtheta)}\mL\vf_{\vtheta}(\vx))) = 0\,,
\end{align}
which in turn implies
\begin{align}
    \sE_{p(\vx,y)} \log p(y; \hat\mW_\infty^{(\hat\vtheta)}\mL\vf_{\vtheta}(\vx)) &= \sE_{p(\vx,y)} \log p(y; \mW\vf_{\vtheta}(\vx)) \\
    \sE_{p(\vx,y)} \log p(y; \hat\mW_\infty^{(\hat\vtheta)}\mL\vf_{\vtheta}(\vx)) &= \sE_{p(\vx,y)} \log p(y; \mW\mL^{-1}\mL\vf_{\vtheta}(\vx)) \\
    \sE_{p(\vx,y)} \log p(y; \hat\mW_\infty^{(\hat\vtheta)}\vf_{\hat\vtheta}(\vx)) &= \sE_{p(\vx,y)} \log p(y; \mW\mL^{-1}\vf_{\hat\vtheta}(\vx)) \\
\end{align}
Since the solution to the population MLE from~\Cref{pb:population_mle} is assumed to be unique, this equality holds if and only if $\hat\mW_\infty^{(\hat\vtheta)} = \mW\mL^{-1}$.
\end{proof}

\section{Proofs of \cref{sub:disentanglement_sparse_mtl}}\label{app:ident_theory}
\subsection{Technical Lemmas}
The lemmas of this section can be skipped at first read.

The following lemma will be important for proving~\Cref{thm:disentanglement}. The argument is taken from~\citet{lachapelle2022disentanglement}.
\begin{lemma}[Sparsity pattern of an invertible matrix contains a permutation] \label{lem:L_perm}
Let $\mL \in \sR^{m\times m}$ be an invertible matrix. Then, there exists a permutation $\sigma$ such that $\mL_{i, \sigma(i)} \not=0$ for all $i$.
\end{lemma}
\begin{proof}
Since the matrix $\mL$ is invertible, its determinant is non-zero, \ie
\begin{align}
    \det(\mL) := \sum_{\sigma\in \mathfrak{S}_m} \text{sign}(\sigma) \prod_{i=1}^m \mL_{i, \sigma(i)} \neq 0 \, ,
\end{align}
where $\mathfrak{S}_m$ is the set of $m$-permutations. This equation implies that at least one term of the sum is non-zero, meaning there exists $\sigma\in \mathfrak{S}_m$ such that for all $i \in [m]$, $\mL_{i, \sigma(i)} \neq 0$.
\end{proof}

The following technical lemma will help us dealing with almost-everywhere statements and can be safely skipped at a first read. Before presenting it, we recall the formal definition of a support of a distribution.
\begin{definition}
The support of a Borel measure $\mu$ over a topological space $(X, \tau)$ is the set of point $x \in X$ such that, for all open set $U \in \tau$ containing $x$, $\mu(U) > 0$.
\end{definition}

Throughout this work, we assume implicitly that all measures are Borel measures with respect to the standard topology of the space on which they are defined.

\begin{lemma}\label{lem:suff_var_ae} \Cref{ass:suff_var_task} is equivalent to the following statement:
For all $E_0 \subset \sR^{k \times m}$ such that $\sP_\mW(E_0) = 0$, there exists $\mW^{(1)}, \dots, \mW^{(m)} \in \gW \setminus E_0$ and indices $i_1, ..., i_m \in [k]$ such that the row vectors $\mW^{(1)}_{i_1, :}, \dots, \mW^{(m)}_{i_m, :}$ are linearly independent.
\end{lemma}
\begin{proof}
First of all, the "$\impliedby$" direction is trivial since one can simply pick $E_0 = \emptyset$.

We now show the "$\implies$" direction. First of all, we notice that, since $\mW^{(1)}_{i_1, :}, \dots, \mW^{(m)}_{i_m, :}$ are linearly independent, they form a matrix with nonzero determinant, \ie
\begin{align}
\det\begin{bmatrix}
\mW^{(1)}_{i_1, :}\\
\vdots \\
\mW^{(m)}_{i_m, :}
\end{bmatrix} \not=0 \,.
\end{align}
Define the map $\eta: (\sR^{k \times m})^m \rightarrow \sR^{m \times m}$ as
\begin{align}
    \eta(\bar\mW^{(1)}, \dots, \bar\mW^{(m)}) := \begin{bmatrix}
\bar\mW^{(1)}_{i_1, :}\\
\vdots \\
\bar\mW^{(m)}_{i_m, :}
\end{bmatrix},\ \forall (\bar\mW^{(1)}, \dots, \bar\mW^{(m)}) \in (\sR^{k \times m})^m \,,
\end{align}
which is continuous. Note that $\det(\cdot)$ is also a continuous map, hence $\det \circ \eta$ is continuous as well. Thus, the set $V := (\det \circ \eta)^{-1}(\sR \setminus \{0\})$ is open (since $\sR \setminus \{0\}$ is open). Let $\sP_{\mW}^m$ be the product measure over tuples of matrices $(\bar\mW^{(1)}, \dots, \bar\mW^{(m)})$. Note that its support is $\gW^m$. Because $(\mW^{(1)}, ..., \mW^{(m)})$ is in the open set $V$ and in the support of $\sP_\mW^m$, we have that
\begin{align}
    0 &< \sP_{\mW}^m(V) \\
    &= \sP_{\mW}^m(V \cap \gW^m) + \sP_{\mW}^m(V \cap (\gW^m)^c) \\
    &\leq \sP_{\mW}^m(V \cap \gW^m) + \sP_{\mW}^m((\gW^m)^c) \\
    &= \sP_{\mW}^m(V \cap \gW^m)
\end{align}
Let $E_0 \subset \sR^{k \times m}$ be such that $\sP_\mW(E_0) = 0$. Then, we also have that $\sP_\mW^m(E_0^m) = 0$ and thus
\begin{align}
    \sP_{\mW}^m((V \cap \gW^m) \setminus E_0^m) > 0 \,.
\end{align}
This implies that the set $((\det \circ \eta)^{-1}(\sR \setminus \{0\}) \cap \gW^m)  \setminus E_0^m$ is not empty, \ie there exists $(\bar\mW^{(1)}, \dots, \bar\mW^{(m)}) \in \gW^m \setminus E_0^m$ such that the rows $\bar\mW^{(1)}_{i_1, :}, \dots, \bar\mW^{(m)}_{i_m, :}$ are linearly independent. Since the measure zero set $E_0$ was arbitrary, this concludes the proof.
\end{proof}


\subsection{Proof of~\Cref{thm:disentanglement_via_optim}}\label{app:main_thm}
This section presents the main results building up to~\Cref{thm:disentanglement_via_optim}.

For all $\mW \in \gW$, we are going to denote by $\hat\mW^{(\mW)}$ some estimator of $\mW$. The following result provides conditions under which if $\hat\mW^{(\mW)}$ allows a perfect fit of the ground-truth distribution $p(y \mid \vx, \mW)$, then the representation $\vf_\vtheta$ and the parameter $\mW$ are identified up to an invertible linear transformation. Many works have showed similar results in various context~\citep{TCL2016,iVAEkhemakhem20a,roeder2020linear,ahuja2022towards}. We reuse some of their proof techniques.

\begin{theorem}[Linear identifiability]\label{thm:linear_ident}
Let $\hat\mW^{(\cdot)}: \gW \rightarrow \sR^{k \times m}$. Suppose \Cref{ass:eta_ident,ass:suff_var_rep,ass:suff_var_task} hold and that, for $\sP_\mW$-almost every $\mW \in \gW$ and all $\vx \in \gX$, the following holds
\begin{equation}
    \mathrm{KL}(p(y; \hat\mW^{(\mW)}\vf_{\hat\vtheta}(\vx)) || p(y; \mW\vf_{\vtheta}(\vx)) = 0 \label{eq:perfect_fit}
    \enspace .
\end{equation}
Then, there exists an invertible matrix $\mL\in \sR^{m \times m}$ such that, for all $\vx \in \gX$, $\vf_\vtheta(\vx) = \mL{\vf}_{\hat\vtheta}(\vx)$ and such that, for $\sP_\mW$-almost every $\mW \in \gW$, $\hat\mW^{(\mW)} = \mW\mL$
\end{theorem}
\begin{proof}
By~\cref{ass:eta_ident}, \Cref{eq:perfect_fit} implies that, for $\sP_\mW$-almost every $\mW$ and all $\vx\in\gX$, $\mW\vf_\vtheta(\vx) = \hat\mW^{(\mW)}\vf_{\hat\vtheta}(\vx)$. \Cref{ass:suff_var_task} combined with~\Cref{lem:suff_var_ae} ensures that we can construct an invertible matrix
$\mU := \begin{bmatrix}
    \mW_{i_{1}, :}^{(1)} \\
    \vdots \\
    \mW_{i_{d_z}, :}^{(d_z)}
\end{bmatrix}$ such that $\mU\vf_\vtheta(\vx) = \hat\mU\vf_{\hat\vtheta}(\vx)$ for all $\vx \in \gX$ where $\hat\mU := \begin{bmatrix}
    \hat\mW_{i_{1}, :}^{(\mW^{(1)})} \\
    \vdots \\
    \hat\mW_{i_{d_z}, :}^{(\mW^{(d_z)})}
    \end{bmatrix}$. Left-multiplying by $\mU^{-1}$ on both sides yields $\vf_\vtheta(\vx) = \mL\vf_{\hat\vtheta}(\vx)$, where $\mL := \mU^{-1}\hat\mU$. Using the invertible matrix $\mF$ from~\cref{ass:suff_var_rep}, we can thus write $\mF = \mL\hat\mF$ where we defined $\hat\mF := [\vf_{\hat\vtheta}(\vx^{(1)}), \cdots, \vf_{\hat\vtheta}(\vx^{(d_z)})]$. Since $\mF$ is invertible, so are $\mL$ and $\hat\mF$.

By substituting $\mF = \mL\hat\mF$ in $\mW\mF = \hat\mW^{(\mW)}\hat\mF$, we obtain $\mW\mL\hat\mF = \hat\mW^{(\mW)}\hat\mF$. By right-multiplying both sides by $\hat\mF^{-1}$, we obtain $\mW\mL = \hat\mW^{(\mW)}$.
\end{proof}

The following theorem is where most of the theoretical contribution of this work lies. Note that~\Cref{thm:disentanglement_via_optim}, from the main text, is a straightforward application of this result.

\begin{theorem}(Disentanglement via task sparsity)\label{thm:disentanglement}
Let $\hat\mW^{(\cdot)}: \gW \rightarrow \sR^{k \times m}$. Suppose \Cref{ass:eta_ident,ass:suff_var_rep,ass:suff_var_task,ass:intra_supp_task_var,ass:suff_support} hold and that, for $\sP_\mW$-almost every $\mW \in \gW$ and all $\vx \in \gX$, the following holds
\begin{equation}
    \mathrm{KL}(p(y; \hat\mW^{(\mW)}\vf_{\hat\vtheta}(\vx)) || p(y; \mW\vf_{\vtheta}(\vx)) = 0 
    \enspace .
\end{equation}
Moreover, assume that $\sE\normin{\hat\mW^{(\mW)}}_{2,0} \leq \sE\normin{\mW}_{2,0}$, where both expectations are taken w.r.t. $\sP_{\mW}$ and $\normin{\mW}_{2,0} := \sum_{j=1}^{m} \mathbbm{1}(\mW_{: j} \not= \bm0)$ with $\mathbbm{1}(\cdot)$ the indicator function. Then, $\vf_{\hat\vtheta}$ is disentangled w.r.t. $\vf_\vtheta$~(\Cref{def:disentanglement}).
\end{theorem}
\begin{proof}
First of all, by~\Cref{ass:eta_ident,ass:suff_var_rep,ass:suff_var_task}, we can apply~\Cref{thm:linear_ident} to conclude that $\vf_\vtheta(\vx) = \mL\vf_{\hat\vtheta}(\vx)$ and $\mW\mL = \hat\mW^{(\mW)}$ ($\sP_\mW$-almost everywhere) for some invertible matrix $\mL$.

We can thus write $\sE\normin{\mW\mL}_{2,0} \leq \sE\normin{\mW}_{2,0}$.

We can write
\begin{align}
    \sE\normin{\mW}_{2,0} &= \sE_{p(S)}\sE[\sum_{j=1}^m \mathbbm{1}(\mW_{: j} \not= \bm0) \mid S] \\
    &= \sE_{p(S)}\sum_{j=1}^m \sE[\mathbbm{1}(\mW_{: j} \not= \bm0) \mid S]\\
    &= \sE_{p(S)}\sum_{j=1}^m \sP_{\mW \mid S}[\mW_{: j} \not= \bm0]\\
    &= \sE_{p(S)}\sum_{j=1}^m \mathbbm{1}(j \in S)\,,
\end{align}
where the last step follows from the definition of $S$.

We now perform similar steps for $\sE\normin{\mW \mL}_{2,0}$:
\begin{align}
    \sE\normin{\mW L}_{2,0} &= \sE_{p(S)}\sE[\sum_{j=1}^m \mathbbm{1}(\mW \mL_{: j} \not= \bm0)\mid S]\\
    &= \sE_{p(S)}\sum_{j=1}^m \sE[\mathbbm{1}(\mW \mL_{: j} \not= \bm0)\mid S]\\
    &= \sE_{p(S)}\sum_{j=1}^m \sP_{\mW \mid S}[\mW \mL_{: j} \not= \bm0]\\
    &= \sE_{p(S)}\sum_{j=1}^m \sP_{\mW \mid S}[\mW_{: S} \mL_{S, j} \not= \bm0]\,.
\end{align}

Notice that
\begin{align}
    \sP_{\mW \mid S}[\mW_{: S} \mL_{S, j} \not= \bm0] &= 1 - \sP_{\mW \mid S}[\mW_{: S} \mL_{S, j} = \bm0] 
\end{align}
Let $N_j$ be the support of $\mL_{: j}$, \ie $N_j := \{i \in [m] \mid \mL_{i, j} \not= 0 \}$. When $S \cap N_j = \emptyset$, $\mL_{S,j} = \bm0$ and thus $\sP_{\mW \mid S}[\mW_{: S} \mL_{S, j} = \bm0] = 1$. When $S \cap N_j \not= \emptyset$, $\mL_{S,j} \not= \bm0$, by~\Cref{ass:intra_supp_task_var} we have that $\sP_{\mW \mid S}[\mW_{: S} \mL_{S, j} = \bm0] = 0$. Thus
\begin{align}
    \sP_{\mW \mid S}[\mW_{: S}\mL_{S, j} \not= \bm0] &= 1 - \mathbbm{1}(S \cap N_j = \emptyset) \\
    &= \mathbbm{1}(S \cap N_j \not= \emptyset)\,,
\end{align}
which allows us to write
\begin{align}
    \sE\normin{\mW \mL}_{2,0} &= \sE_{p(S)}\sum_{j=1}^m \mathbbm{1}(S \cap N_j \not= \emptyset)\,.
\end{align}
We thus have that
\begin{align}
    \sE\normin{\mW\mL}_{2,0} &\leq \sE\normin{\mW}_{2,0} \\
    \sE_{p(S)}\sum_{j=1}^m \mathbbm{1}(S \cap N_j \not= \emptyset) &\leq \sE_{p(S)}\sum_{j=1}^m \mathbbm{1}(j \in S) \,. \label{eq:before_perm}
\end{align}
Since $\mL$ is invertible, by~\Cref{lem:L_perm}, there exists a permutation $\sigma: [m] \rightarrow [m]$ such that, for all $j \in [m]$, $\mL_{j, \sigma(j)} \not=0$. In other words, for all $j \in [m]$, $j \in N_{\sigma(j)}$. Of course we can permute the terms of the l.h.s. of~\cref{eq:before_perm}, which yields
\begin{align}
    \sE_{p(S)}\sum_{j=1}^m \mathbbm{1}(S \cap N_{\sigma(j)} \not= \emptyset) &\leq \sE_{p(S)}\sum_{j=1}^m \mathbbm{1}(j \in S)\\
    \sE_{p(S)}\sum_{j=1}^m \left(\mathbbm{1}(S \cap N_{\sigma(j)} \not= \emptyset) - \mathbbm{1}(j \in S)\right) &\leq 0 \,. \label{eq:sum_of_positive}
\end{align}
We notice that each term $\mathbbm{1}(S \cap N_{\sigma(j)} \not= \emptyset) - \mathbbm{1}(j \in S) \geq 0$ since whenever $j \in S$, we also have that $j \in S \cap N_{\sigma(j)}$ (recall $j \in N_{\sigma(j)}$). Thus, the l.h.s. of \cref{eq:sum_of_positive} is a sum of non-negative terms which is itself non-positive. This means that every term in the sum is zero:
\begin{align}
    \forall S \in \gS,\ \forall j \in [m],&\ \mathbbm{1}(S \cap N_{\sigma(j)} \not= \emptyset) = \mathbbm{1}(j \in S)\,.
\end{align}
Importantly,
\begin{align}
    \forall j \in [m],\ \forall S \in \gS,\ j \not\in S \implies S \cap N_{\sigma(j)} = \emptyset\,,
\end{align}
and since $S \cap N_{\sigma(j)} = \emptyset \iff N_{\sigma(j)} \subseteq S^c$ we have that
\begin{align}
    \forall j \in [m],\ \forall S \in \gS,\ j \not\in S \implies N_{\sigma(j)} \subseteq S^c \\
    \forall j \in [m],\ N_{\sigma(j)} \subseteq \bigcap_{S \in \gS \mid j \not\in S} S^c \,. \label{eq:last_intersect}
\end{align}
By~\Cref{ass:suff_support}, we have that $\bigcup_{S \in \gS \mid j \not\in S} S = [m] \setminus \{j\}$. By taking the complement on both sides and using De Morgan's law, we get $\bigcap_{S \in \gS \mid j \not\in S} S^c = \{j\}$, which implies that $N_{\sigma(j)} = \{j\}$ by~\Cref{eq:last_intersect}. Thus, $\mL = \mD\mP$ where $\mD$ is an invertible diagonal matrix and $\mP$ is a permutation matrix.
\end{proof}


Before presenting~\Cref{thm:disentanglement_via_optim} from the main text, we first present a variation of it where we constrain $\sE \normin{\hat\mW^{(\mW)}}_{2,0}$ to be smaller than $\sE \normin{\mW}_{2,0}$. We note that this is weaker than imposing $\normin{\hat\mW^{(\mW)}}_{2,0} \leq \normin{\mW}_{2,0}$ for all $\mW \in \gW$, as is the case in~\Cref{pb:OG_problem_inner} of~\Cref{thm:disentanglement_via_optim}. Note that \Cref{app:relaxing_outer} presents a natural relaxation of \Cref{pb:OG_problem_inner_outer} which we experiment with in \Cref{app:exp_outer_reg}.

\begin{theorem}[Sparse multitask learning for disentanglement]\label{thm:disentanglement_via_optim_outer}
Let $\hat\vtheta$ be a minimizer of
\begin{align}\label[pb_multiline]{pb:OG_problem_inner_outer}
    \begin{aligned}
        \min_{\hat\vtheta}\ &\sE_{\sP_\mW}\sE_{p(\vx, y \mid \mW)}-\log p(y; \hat\mW^{(\mW)}\vf_{\hat\vtheta}(\vx))
        \\
        \mathrm{s. t.}\ &
        \quad
        \forall\ \mW \in \gW,
        \,
        \hat\mW^{(\mW)} \in \argmin_{\tilde\mW} \sE_{p(\vx, y \mid \mW)}-\log p(y; \tilde\mW\vf_{\hat\vtheta}(\vx)) \\
        &\quad \sE_{\sP_\mW}\normin{\hat\mW^{(\mW)}}_{2,0} \leq \sE_{\sP_\mW} \normin{\mW}_{2,0}
        \enspace ,
    \end{aligned}
\end{align}
where $\sP_\mW$ and $p(\vx, y \mid \mW)$ are described in \cref{sub:task_data_gen}. Under~\Cref{ass:eta_ident,ass:suff_var_rep,ass:suff_var_task,ass:intra_supp_task_var,ass:suff_support} and if $\vf_{\tilde\vtheta}$ is continuous for all $\tilde\vtheta$, $\vf_{\hat\vtheta}$ is disentangled w.r.t. $\vf_\vtheta$~(\Cref{def:disentanglement}).
\end{theorem}
\begin{proof}
First, notice that
\begin{align}
    &0 \leq \sE_{\sP_\mW}\sE_{p(\vx \mid \mW)} \text{KL}(p(y; \mW\vf_\vtheta(\vx)) \mid\mid p(y; \hat\mW^{(\mW)}\vf_{\hat\vtheta}(\vx))) \label{eq:expected_kl}\\
    &\sE_{\sP_\mW}\sE_{p(\vx, y \mid \mW)}-\log p(y; \mW\vf_{\vtheta}(\vx)) \leq \sE_{\sP_\mW}\sE_{p(\vx, y \mid \mW)}-\log p(y; \hat\mW^{(\mW)}\vf_{\hat\vtheta}(\vx))\, . \label{eq:lower_bound_objective}
\end{align}
This means the objective is minimized (without constraint) if and only if
\begin{align}
    \sE_{p(\vx \mid \mW)} \text{KL}(p(y; \mW\vf_\vtheta(\vx)) \mid\mid p(y; \hat\mW^{(\mW)}\vf_{\hat\vtheta}(\vx))) = 0
\end{align}
$\sP_\mW$-almost everywhere. For a fixed $\mW$, this equality holds if and only if the $\text{KL}$ equals zero ${p(\vx \mid \mW)}$-almost everywhere, which, by~\Cref{ass:eta_ident}, is true if and only if $\mW\vf_\vtheta(\vx) = \hat\mW^{(\mW)}\vf_{\hat\vtheta}(\vx)$ $p(\vx \mid \mW)$-almost everywhere. Since both $\mW\vf_\vtheta(\vx)$ and $\hat\mW^{(\mW)}\vf_{\hat\vtheta}(\vx)$ are continuous functions of $\vx$, the equality holds over $\gX$ (the support of $p(\vx \mid \mW)$).

This unconstrained global minimum can actually be achieved by respecting the constraints of~\Cref{pb:OG_problem_inner_outer} simply by setting $\hat\vtheta := \vtheta$ and $\hat\mW^{(\mW)} := \mW$. Indeed, the first constraint is satisfied because, for all $\tilde\mW$,
\begin{align}
    &0 \leq \sE_{p(\vx \mid \mW)} \text{KL}(p(y; \mW\vf_\vtheta(\vx)) \mid\mid p(y; \tilde\mW\vf_{\vtheta}(\vx))) \label{eq:expected_kl_2}\\
    &\sE_{p(\vx, y \mid \mW)}-\log p(y; \mW\vf_{\vtheta}(\vx)) \leq \sE_{p(\vx, y \mid \mW)}-\log p(y; \tilde\mW\vf_{\vtheta}(\vx))\, , \label{eq:lower_bound_objective_2}
\end{align}
and clearly the lower bound is attained when $\tilde\mW := \mW$. The second constraint is trivially satisfied, since $\sE_{\sP_\mW}\normin{\hat\mW^{(\mW)}}_{2,0} = \sE_{\sP_\mW}\normin{\mW}_{2,0}$.

The above implies that if $\hat\vtheta$ is some minimizer of~\Cref{pb:OG_problem_inner_outer}, we must have that, (i) for $\sP_\mW$-almost every $\mW$, $\mW\vf_\vtheta(\vx) = \hat\mW^{(\mW)}\vf_{\hat\vtheta}(\vx)$ for all $\vx \in \gX$, (ii) $\sE_{\sP_\mW}||\hat\mW^{(\mW)}||_0 \leq \sE_{\sP_\mW}||\mW||_0$. Thus, \Cref{thm:disentanglement} implies the desired conclusion.
\end{proof}

Based on~\Cref{thm:disentanglement_via_optim_outer}, we can slightly adjust the argument to prove~\Cref{thm:disentanglement_via_optim} from the main text.

\disViaOptInner*

\begin{proof}
The first part of the argument in the proof of~\Cref{thm:disentanglement_via_optim_outer} applies here as well, meaning: unconstrained minimization of the objective holds if and only if, for
$\sP_\mW$-almost every $\mW$ and all $\vx \in \gX$, $\mW\vf_\vtheta(\vx) = \hat\mW^{(\mW)}\vf_{\hat\vtheta}(\vx)$. And again, this unconstrained minimum can be achieved by respecting the constraint of~\Cref{pb:OG_problem_inner} simply by setting $\hat\vtheta := \vtheta$ and $\hat\mW^{(\mW)} := \mW$.

This means that if $\hat\vtheta$ is some minimizer of~\Cref{pb:OG_problem_inner}, we must have (i) for $\sP_\mW$-almost every $\mW$, $\mW\vf_\vtheta(\vx) = \hat\mW^{(\mW)}\vf_{\hat\vtheta}(\vx)$ for all $\vx \in \gX$ and (ii) for all $\mW \in \gW$, $\normin{\hat\mW^{(\mW)}}_{2,0} \leq \normin{\mW}_{2,0}$. Of course the latter point implies $\sE_{\sP_\mW}||\hat\mW^{(\mW)}||_{2,0} \leq \sE_{\sP_\mW}||\mW||_{2,0}$, which allows us to apply~\Cref{thm:disentanglement} to obtain the desired conclusion.
\end{proof}

\subsection{Regularization in the outer problem instead of in the inner problem}\label{app:relaxing_outer}
\Cref{thm:disentanglement_via_optim_outer} presented an alternative bilevel optimization problem to the one of \Cref{thm:disentanglement_via_optim} in the main text. Essentially, the difference is that the constraints $\normin{\hat\mW^{(\mW)}}_{2,0} \leq \normin{\mW}_{2,0}$ for all $\mW \in \gW$ are replaced by the unique constraint $\sE\normin{\hat\mW^{(\mW)}}_{2,0} \leq \sE\normin{\mW}_{2,0}$, which is a weaker constraint.

In \Cref{sec:sparse_bilevel}, we introduced a tractable relaxation of the problem of Theorem~\ref{thm:disentanglement_via_optim}. In this section, we introduce a relaxation of the problem of \Cref{thm:disentanglement_via_optim_outer}. 

A natural idea is to replace the constraint $\sE\normin{\hat\mW^{(\mW)}}_{2,0} \leq \sE\normin{\mW}_{2,0}$ of \Cref{thm:disentanglement_via_optim_outer} by a penalty $\lambda\sE\normin{\hat\mW^{(\mW)}}_{2,1}$ in the outer problem, like so:

\begin{align}\label[pb_multiline]{pb:OG_problem_inner_outer_prime}
    \begin{aligned}
        \min_{\hat\vtheta}\ &\sE_{\sP_\mW}\sE_{p(\vx, y \mid \mW)}-\log p(y; \hat\mW^{(\mW)}\vf_{\hat\vtheta}(\vx)) + \lambda \sE_{\sP_\mW}\normin{\hat\mW^{(\mW)}}_{2,1} 
        \\
        \mathrm{s. t.}\ &
        \quad
        \forall\ \mW \in \gW,
        \,
        \hat\mW^{(\mW)} \in \argmin_{\tilde\mW} \sE_{p(\vx, y \mid \mW)}-\log p(y; \tilde\mW\vf_{\hat\vtheta}(\vx)) \, ,
    \end{aligned}
\end{align}
in which we can replace the expectations by empirical averages to get
\begin{align}
        \min_{\hat\vtheta}\ &\ \frac{1}{T}\sum_{t=1}^T\left[-\frac{1}{n}\sum_{(\vx, y) \in \gD_t} \log p(y; \hat\mW^{(t)}\vf_{\hat\vtheta}(\vx)) + \lambda \normin{\hat\mW^{(t)}}_{2,1}\right] \label[pb_multiline]{pb:OG_problem_outer_relax}
        \\
        \mathrm{s. t.}\ &
        \,
        \hat\mW^{(t)}
        \in
        \argmin_{\tilde\mW } \frac{1}{n}\sum_{(\vx, y) \in \gD_t} \!\!\!\!\!-\log p(y; \tilde\mW\vf_{\hat\vtheta}(\vx)) \, .\nonumber
\end{align}

This can be optimized in the same way as \Cref{pb:OG_problem_inner_relax} via implicit differentiation and standard gradient descent algorithms. The essential difference between \Cref{pb:OG_problem_outer_relax} and \Cref{pb:OG_problem_inner_relax} is that the former has regularization in the outer problem instead of in the inner problem. From a practical point of view, this problem is typically simpler than \Cref{pb:OG_problem_inner_relax} since the inner objective is generally smooth, and standard implicit differentiation techniques apply (the non-smooth term $\normin{\tilde\mW}_{2,1}$ in the inner objective of \cref{pb:OG_problem_inner_relax} requiring some care with implicit differentiation; \citealp{Bertrand2022}). We provide some experimental results in \Cref{app:exp_outer_reg} demonstrating that this alternative works as well.

\subsection{What can go wrong when \Cref{ass:intra_supp_task_var} is violated?}\label{app:counter_example_suff_var}
\Cref{thm:linear_ident} allowed us to conclude that $\hat\mW^{(\mW)} = \mW\mL$ for $\sP_\mW$-almost every $\mW$ and that $\mL\vf_{\hat\vtheta}(\vx) = \vf_\vtheta(\vx)$ for all $\vx \in \gX$. The rest of the argument leading up to \Cref{thm:disentanglement_via_optim} essentially amounts to showing that having $\normin{\hat\mW^{(\mW)}}_{2,0} \leq \normin{\mW}_{2,0}$ for all $\mW \in \gW$ forces $\mL$ to be a permutation-scaling matrix. The intuition is that $\normin{\mW\mL}_{2,0} \leq \normin{\mW}_{2,0}$ everywhere should force $\mL$ to be sparse, and maximal sparsity is precisely when $\mL$ is a permutation-scaling matrix. But just how many $\mW$ do we need and how diverse should they be to make this argument formal? Our answer is given by \Cref{ass:intra_supp_task_var}. But what can go wrong when this assumption is not satisfied? To answer this question, we construct a counterexample in which the distribution $\sP_\mW$ satisfies \Cref{ass:suff_support} but not \Cref{ass:intra_supp_task_var} and a matrix $\mL$ that satisfies the constraint $\normin{\mW\mL}_{2,0} \leq \normin{\mW}_{2,0}$ everywhere but that is not a permutation-scaling matrix. Consider a distribution $\sP_\mW$ with support $\gW := \{[1,1,0], [1,0,1], [0,1,1]\}$ (which is finite) and let
\begin{align}
    \mL := \begin{bmatrix}
        3 & -1 & -1 \\
        -1 & 1 & 3 \\
        1 & 3 & 1
    \end{bmatrix}\, ,
\end{align}
which, of course, is not a permutation-scaling matrix. One can then compute to show that the sparsity constraint holds for all $\mW \in \gW$:
\begin{align}
    \normin{[1\ 1\ 0]\mL}_{2,0} &= \normin{[2\ 0\ 2]}_{2,0} \leq 2 = \normin{[1\ 1\ 0]}_{2,0}\\
    \normin{[1\ 0\ 1]\mL}_{2,0} &= \normin{[4\ 2\ 0]}_{2,0} \leq 2 = \normin{[1\ 0\ 1]}_{2,0}\\
    \normin{[0\ 1\ 1]\mL}_{2,0} &= \normin{[0\ 4\ 4]}_{2,0} \leq 2 = \normin{[0\ 1\ 1]}_{2,0} \,.
\end{align}

This means that, with such a $\sP_\mW$, solving the bilevel problem of \Cref{thm:disentanglement_via_optim} will not necessarily lead to a disentangled representation since one could fall on a ``bad'' $\mL$ such as the one defined above.

\subsection{\Cref{ass:suff_support} holds with high probability when the number of supports is large} \label{app:proba_suff_support}
In this section, we provide a probabilistic argument showing that \Cref{ass:suff_support} holds with high probability when the number of supports is large. Let $\gS^{(T)} := \{S^{(1)}, S^{(2)}, \dots, S^{(T)}\}$ be the set of supports observed, where $T$ is the number of supports. To make this argument, we will assume that the $S^{(t)}$ are sampled independently and identically. Moreover, $\sP[i \in S^{(t)}] = p \in (0,1)$ and these events are assumed independent. 

The next proposition shows that the probability that \Cref{ass:suff_support} fails under the above model is very small when $T$ is large.

\begin{proposition}
Given the probabilistic model described above, we have
\begin{align}
    \sP\left[\exists j \in [m]\ \textnormal{s.t.}\ \bigcup_{S \in \gS^{(T)} \mid j \not\in S} S \not= [m] \setminus \{j\}\right] \leq m(m-1)(1 - p(1-p))^T \xrightarrow{ T \to \infty } 0\,.
\end{align}
\end{proposition}
\begin{proof}
By rewriting slightly the original probablity statement and applying the union bound, we get
    \begin{align}
        & \sP\left[\exists j\in [m]\ \textnormal{s.t.}\ \bigcup_{S \in \gS^{(T)} \mid j \not\in S} S \not= [m] \setminus \{j\}\right] \\
        =& \sP\left[\exists j \in [m], i \in [m] \setminus \{j\}\ \textnormal{s.t.}\ i \not\in \bigcup_{S \in \gS^{(T)} \mid j \not\in S} S \right] \\
        \leq& \sum_{j=1}^m \sum_{i \in [m] \setminus \{j\}} \sP\left[i \not\in \bigcup_{S \in \gS^{(T)} \mid j \not\in S} S \right] \, ,
    \end{align}

We can further write
\begin{align}
    \sP\left[i \not\in \bigcup_{S \in \gS^{(T)} \mid j \not\in S} S \right] & = \sP\left[\forall t \in [T], j \not\in S^{(t)} \implies i \not\in S^{(t)} \right] \\
    &= \sP\left[\forall t \in [T], j \in S^{(t)} \lor i \not\in S^{(t)} \right]\\
    &= \prod_{t=1}^{T} \sP\left[j \in S^{(t)} \lor i \not\in S^{(t)} \right] \, ,
\end{align}
where the last step holds because the supports $S^{(t)}$ are mutually independent. We continue and get
\begin{align}
    \sP\left[i \not\in \bigcup_{S \in \gS^{(T)} \mid j \not\in S} S \right] & = \prod_{t=1}^{T} \sP\left[j \in S^{(t)} \lor i \not\in S^{(t)} \right] \\
    &= \prod_{t=1}^{T} (1 - \sP\left[j \not\in S^{(t)} \land i \in S^{(t)} \right]) \\
    &= \prod_{t=1}^{T} (1 - \sP\left[j \not\in S^{(t)}\right]\sP\left[i \in S^{(t)} \right]) \\
    &= \prod_{t=1}^{T}(1 - (1-p)p)\, ,
\end{align}
where we used the fact that the events $j \not\in S^{(t)}$ and  $i \in S^{(t)}$ are independent (when $i \not= j$). Bringing everything together, one gets
\begin{align}
    \sP\left[\exists j\in [m]\ \textnormal{s.t.}\ \bigcup_{S \in \gS^{(T)} \mid j \not\in S} S \not= [m] \setminus \{j\}\right]  
    &\leq \sum_{j=1}^m \sum_{i \in [m] \setminus \{j\}} \prod_{t=1}^{T} (1 - (1-p)p) \\
    &= m(m-1)(1 - (1-p)p)^T \\
\end{align}

which converges to 0 when $T \rightarrow \infty$ since $0 < 1 - (1-p)p < 1$.
\end{proof}

\subsection{A distribution without density satisfying~\Cref{ass:intra_supp_task_var}}\label{app:discuss_assumption}
Interestingly, there are distributions over $\mW_{1,S} \mid S$ that do not have a density w.r.t. the Lebesgue measure, but still satisfy \Cref{ass:intra_supp_task_var}.
This is the case, e.g., when $\mW_{1,S} \mid S$ puts uniform mass over a $(|S|-1)$-dimensional sphere embedded in $\sR^{|S|}$ and centered at zero. In that case, for all $\va \in \sR^{|S|} \backslash \{0\}$, the intersection of $\text{span}\{\va\}^{\perp}$, which is $(|S|-1)$-dimensional, with the $(|S|-1)$-dimensional sphere is $(|S|-2)$-dimensional and thus has probability zero of occurring.
One can certainly construct more exotic examples of measures satisfying~\Cref{ass:intra_supp_task_var} that concentrate mass on lower dimensional manifolds.

\section{Optimization details}
\subsection{Group Lasso SVM Dual}
\label{proof:dual_mtl_lasso}
\paragraph{Notation.}
The Fenchel conjugate of a function
$h:\sR^{d} \to \sR$ is written $h^*$ and is defined for any $y \in \sR^d,$ by $h^*(y)
= \sup_{x\in \sR^d} \langle x, y \rangle - h(x)$.

\begin{restatable}{definition}{primlalGroupLassoSoftMargin}\emph{(Primal Group Lasso Soft-Margin Multiclass SVM.)}
    \label{def:primal_soft_sparse_multiclass_svm}
    The primal problem
    of the group Lasso soft-margin multiclass SVM is defined as
    \begin{align}\label[pb_multiline]{pb:primal_multiclass_group_svm}
        \begin{aligned}
        \min_{\mW \in \sR^{k \times m}
        }
        \gL_{\mathrm{in}}(\mW; \mF, \mY)
        :=
        \sum_{i=1}^{n}
        \max_ {l \in [k]}
            \left (1
            + ( \mW_{\my_i :} - \mW_{l :} ) \mF_{i:}
            - \mY_{il} \right )
            + \lambda_1 \normin{\mW}_{2, 1}
            + \tfrac{\lambda_2}{2} \normin{\mW}^2
        \end{aligned}
    \end{align}
\end{restatable}
%

\begin{restatable}{proposition}{dualGroupLassoSoftMargin}\emph{(Dual Group Lasso Soft-Margin Multiclass SVM.)}
    \label{prop:dual_group_lasso_svm_main}
        {The dual of the inner problem with $\gL_{\mathrm{in}}$ as defined in (\ref{pb:primal_multiclass_group_svm_main})} writes
        \begin{align}
            &
             \min_{\mLambda \in \sR^{n \times k}}
            \frac{1}{\lambda_2}
            \sum_{j=1}^m
                \normin{\mathrm{BST} \left ( (\mY - \mLambda)^{\top} \mF_{:j}, \lambda_1 \right )}^2
            + \langle \mY, \mLambda \rangle
            \nonumber
            \\
            &
            \mathrm{s. t.} \;
            \forall i, l, \in [n] \times [k],
            \;
            \sum \limits_{l'=1}^k \mLambda_{il'}=1
            \; \text{ and } \;
                \mLambda_{il} \geq 0 \enspace,
                \tag{\ref{pb:dual_multiclass_group_svm_main} }
        \end{align}
        with $\mathrm{BST} : (\va, \tau) \mapsto \left( 1 - {\tau}/{\norm{\va}} \right)_+ \va$ is the block~soft-thresholding operator, $\vF \in \sR^{n \times m}$ the \mbox{concatenation} of $\{\vf_{\hat \vtheta}(x)\}_{(x, y) \in \gD^{\mathrm{train}}}$.~In~\mbox{addition},~the \mbox{primal-dual} link writes, $\forall j \in [m],\ \mW_{:j} = \mathrm{BST} \left ( (\mY - \mLambda)^{\top} \mF_{:j}, \lambda_1 \right ) / \lambda_2$.
    \end{restatable}

The primal objective~\ref{pb:primal_multiclass_group_svm} can be hard to minimize with modern solvers. Moreover in
few-shot learning applications, the number of features $m$ is usually much larger than the number of samples $n$ (in \citealt{Lee_Maji_Ravichandran_Soatto2019meta}, $m=1.6 \cdot 10^4$ and $n\leq 25$), hence we solve the dual of  \Cref{pb:primal_multiclass_group_svm}.
%
%
%
%
\begin{proof}[Proof of \Cref{prop:dual_group_lasso_svm_main}]
Let $g : \vu \mapsto \lambda_1 \normin{\vu}  + \tfrac{\lambda_2}{2} \normin{\vu}^2$.
Proof of \Cref{prop:dual_group_lasso_svm_main} is composed of the following lemmas.
    \begin{lemma}
      \begin{lemmaenum}[topsep=4pt,itemsep=4pt,partopsep=4pt,parsep=4pt]
        \item The dual of \Cref{pb:primal_multiclass_group_svm} is
            \begin{align}\label[pb_multiline]{pb:dual_multiclass_generic_svm}
                    \begin{aligned}
                         \min_{\mLambda \in \sR^{n \times k}}
                        &
                        \sum_{j=1}^m
                        g^*((\mY - \mLambda )^{\top} \mF_{:j})
                        +
                        \langle \mY, \mLambda \rangle
                        \\
                        \mathrm{s. t.}
                        \quad
                        &\forall i \in [n],
                        \,
                        \sum_{l=1}^k \mLambda_{i l}=1
                        \enspace,
                        \quad
                        \forall i \in [n], l \in [k],
                        \, \mLambda_{i l} \geq 0
                        \enspace ,
                \end{aligned}
            \end{align}
            where $g^*$ is the Fenchel conjugate of the function $g$.
        \label{lemma:generic_dual_svm}
        \item The Fenchel conjugate of the function $g$ writes
        \begin{align}
            \forall \vv \in \sR^K, \,
            g^*(\vv) = \frac{1}{\lambda_2} \normin{\mathrm{BST}(\vv, \lambda_1) }^2 \enspace .
        \end{align}
        \label{lemma:fenchel_group_lasso}
      \end{lemmaenum}
    \end{lemma}
\Cref{lemma:generic_dual_svm,lemma:fenchel_group_lasso} yields \Cref{prop:dual_group_lasso_svm_main}.

\begin{proof}[Proof of \Cref{lemma:generic_dual_svm}.]
    The Lagrangian of \Cref{pb:primal_multiclass_group_svm} writes:
    \begin{align}
        \mathcal{L}(\mW, \mxi, \mLambda)
        & =
        \sum_{j=1}^m g(\mW_{:j})
        + \sum_i \mxi_i
        + \sum_{i=1}^n \sum_{l=1}^k (
            1 - \mxi_i - \mW_{\vy_i :} \cdot  \mF_{i:}
            + \mW_{l :} \cdot \mF_{i:}
            -  \mY_{i l}
            ) \mLambda_{i l}
        \enspace .
    \end{align}
    $\partial_{\mxi} \mathcal{L}(\mW, \mxi, \mLambda) =0$ yields
    $\forall i \in [n],\sum_{l=1}^k \mLambda_{i l}=1$.
    Then the Lagrangian rewrites
    \begin{align}
        \min_{\mW}  \min_{\mxi} \mathcal{L}(\mW, \mxi, \mLambda)
        & =
        \min_{\mW, \mxi}
        \sum_{j=1}^m g(\mW_{:j})
        + \sum_{i=1}^n \mxi_i
        + \sum_{i=1}^n \sum_{l=1}^k(
            - \mxi_i - \mW_{\vy_i :} \cdot \mF_{i:}
            + \mW_{l :} \cdot \mF_{i:}
            -  \mY_{i l}
            ) \mLambda_{i l}
        \nonumber
        \\
        & =
        \sum_{j=1}^m
        \underbrace{
        \min_{\mW_{:j}}
         g(\mW_{:j})
        - \underbrace{\sum_{i=1}^n \sum_{l=1}^k (
            \mF_{i:} \mY_{i l} - \mF_{i:} \mLambda_{i l}) \mW_{l :}}_{= \langle (\mY - \mLambda)^{\top} \mF_{:j}, \mW_{:j} \rangle}
            }_{ =
            - g^*((\mY - \mLambda)^{\top} \mF_{:j})}
        - \sum_{i=1}^n \sum_{l=1}^k
        \mY_{i l}
        \mLambda_{i l}
        \enspace .
        \nonumber
    \end{align}
    Then the dual problem writes:
    \begin{align}
         \min_{\mLambda \in \sR^{n \times k}}
        &
        \sum_{j=1}^m
            g^* \left ( (\mY - \mLambda)^{\top} \mF_{:j} \right )
        + \langle \mY, \mLambda \rangle
        \\
        \text{s. t.}
        \quad
        &\forall i \in [n]
        \quad
        \sum_{l=1}^{k} \mLambda_{i l}=1
        \enspace,
        \forall i \in [n], l \in [k],
        \quad\mLambda_{i l} \geq 0
        \enspace .
    \end{align}
\end{proof}
\begin{proof}[Proof of \Cref{lemma:fenchel_group_lasso}]
Let $h : \vu \mapsto \normin{\vu}_2 + \frac{\kappa}{2} \normin{\vu}^2 $.
The proof of \Cref{lemma:generic_dual_svm} is done using the following steps.
\begin{lemma}
  \begin{lemmaenum}
    \item \label{lemma:fenchel_h}
        $h^*(\vv)
        =
        \tfrac{1}{2 \kappa} \normin{\vv}_2^2
        - \left ( \tfrac{\kappa}{2} \normin{\cdot}_2^2 \square \normin{\cdot}_2 \right) (\vv / \kappa)$.
    \item \label{lemma:inf_conv_group}
    $\left ( \tfrac{\kappa}{2} \normin{\cdot}_2^2 \square \normin{\cdot}_2 \right)(\vv)
        =
        \tfrac{\kappa}{2} \normin{\vv}_2^2
        -
        \frac{1}{2 \kappa} \normin{\mathrm{BST}(\kappa \vv, 1)}^2$.
  \end{lemmaenum}
\end{lemma}
\begin{proof}[Proof of \Cref{lemma:fenchel_h}]
With $\kappa = \lambda_2 / \lambda_1$,  the Fenchel transform of $h: \vw \mapsto \normin{\vw}_2 + \kappa \normin{\vw}^2$.

\begin{align*}
    h(\vu) &= \normin{\vu}_2 + \tfrac\kappa2 \normin{\vu}_2^2 \\
    h^*(\vv) &= \sup_{\vw} \left( \vv^\top \vw - \normin{\vw}_2 - \tfrac\kappa2 \normin{\vw}_2^2 \right)\\
        &= \tfrac{1}{2\kappa} \normin{\vv}_2^2 + \sup_{\vw} \left( -\tfrac\kappa2 \normin{\vw - \vv/\kappa}_2^2 - \normin{\vw}_2 \right)\\
        &= \tfrac{1}{2\kappa} \normin{\vv}_2^2 - \inf_{\vw} \left(\tfrac\kappa2 \normin{\vw - \vv/\kappa}_2^2 + \normin{\vw}_2 \right) \\
        &= \tfrac{1}{2\kappa} \normin{\vv}_2^2 -  (\tfrac\kappa2 \normin{\cdot}_2^2 \square \normin{\cdot}_2)(\vv/\kappa)
        \enspace .
\end{align*}
\end{proof}
\begin{proof}[Proof of \Cref{lemma:inf_conv_group}]
\begin{align*}
    (\tfrac\kappa2 \normin{\cdot}_2^2 \square \normin{\cdot}_2)(\vv)
    &= (\tfrac\kappa2 \normin{\cdot}_2^2 \square \normin{\cdot}_2)^{**}(\vv)
    \\
    &= (\tfrac{1}{2\kappa} \normin{\cdot}_2^2 +\iota_{\mathcal{B}_2})^{*}(\vv)
    \\
    &= \sup_{\normin{\vw}_2 \leq 1} \left(\vv^\top \vw - \tfrac{1}{2\kappa}  \normin{\vw}_2^2 \right) \\
    &=\tfrac{\kappa}{2} \normin{\vv}^2 + \sup_{\normin{\vw}_2 \leq 1} - \tfrac{1}{2\kappa} \normin{\kappa \vv - \vw}_2^2 \\
    &=\tfrac{\kappa}{2} \normin{\vv}^2 - \tfrac{1}{2\kappa} \normin{\mathrm{BST}(\kappa \vv, 1)}_2^2
    \enspace .
\end{align*}
\end{proof}

\begin{align*}
    g^*(\vu)
    &= \lambda_1 h^*(\vu / \lambda_1)
    \\
    &= \frac{\lambda_1 }{ 2 \kappa} \normin{\mathrm{BST}(\vu / \lambda_1, 1)}^2
    \\
    &= \frac{\lambda_1^2 }{ 2 \lambda_2}  \normin{\mathrm{BST}(\vu / \lambda_1, 1)}^2
    \\
    &= \frac{1}{\lambda_2} \normin{\mathrm{BST}(\vu, \lambda_1)}^2 \enspace  .
\end{align*}

\end{proof}

\end{proof}


\section{Experimental details}
\subsection{Disentangled representation coupled with sparsity regularization improves generalization}\label{app:dis_lass_gen}
We consider the following data generating process: We sample the ground-truth features $\vf_\vtheta(\vx)$ from a Gaussian distribution $\gN(\bm0, \bm\Sigma)$ where $\bm\Sigma \in \sR^{m\times m}$ and $\bm\Sigma_{i, j}= 0.9^{|i-j|}$. Moreover, the labels are given by $y = \vw  \cdot \vf_\vtheta(\vx) + \epsilon$ where $\vw \in \sR^m$, $\epsilon \sim \mathcal{N}(0, 0.04)$ and $m=100$. The ground-truth weight vector $\vw$ is sampled once from $\mathcal{N}(0, I_{m\times m})$ and mask some of its components to zero: we vary the fraction of meaningful features ($\ell/m$) from very sparse ($\ell/m=5\%$) to less sparse ($\ell/m=80\%$) settings.
For each case, we study the sample complexity by varying the number of training samples from 25 to 150, but evaluating the generalization performance on a larger test dataset (1000 samples). To generate the entangled representations, we multiply the true latent variables $\vf_{\vtheta}(\vx)$ by a randomly sampled orthogonal matrix $\mL$, \ie $\vf_{\hat\vtheta}(\vx) := \mL\vf_\vtheta(\vx)$. For the disentangled representation, we simply consider the true latents, \ie $\vf_{\hat\vtheta}(\vx) := \vf_\vtheta(\vx)$. Note that in principle we could have considered an invertible matrix $\mL$ that is not orthogonal for the linearly entangled representation and a component-wise rescaling for the disentangled representation. The advantage of not doing so and opting for our approach is that the conditioning number of the covariance matrix of $\vf_{\hat\vtheta}(\vx)$ is the same for both the entangled and the disentangled, hence offering a fairer comparison. 

For both the case of entangled and disentangled representation, we solve the regression problem with Lasso and Ridge regression, where the associated hyperparameters (regularization strength) were inferred using 5-fold cross-validation on the input training dataset. Using both lasso and ridge regression would help us to show the effect of encouraging sparsity.

In Figure~\ref{fig:sparsity-disentanglement-gains} for the sparsest case ($\ell/m=5\%$), we observe that that Disentangled-Lasso approach has the best performance when we have fewer training samples, while the Entangled-Lasso approach performs the worst. As we increase the number of training samples, the performance of Entangled-Lasso approaches that of Disentangled-Lasso, however, learning under the Disentangled-Lasso approach is sample efficient. Disentangled-Lasso obtains $R^2$ greater than 0.5 with only 25 training samples, while other approaches obtain $R^2$ close to zero. Also, Disentagled-Lasso converges to the optimal $R^2$ using only 50 training samples, while Entangled-Lasso does the same with 150 samples.

Note that the improvement due to disentanglement does not happen for the case of ridge regression as expected and there is no difference between the methods Disentangled-Ridge and Entangled-Ridge because the L2 norm is invariant to orthogonal transformation. Also, having sparsity in the underlying task is important. Disentangled-Lasso shows the max improvement for the case of $\ell/m=5\%$, with the gains reducing as we decrease the sparsity in the underlying task ($l/m= 80\%$).

\subsection{Disentanglement in 3D Shapes}
\label{app:dis_exp}

\begin{figure}[th]
    \centering
    \includegraphics[width=0.5\columnwidth]{figures/3dshape_influ_corr_fulllegend.png}
    \includegraphics[width=0.82\columnwidth]{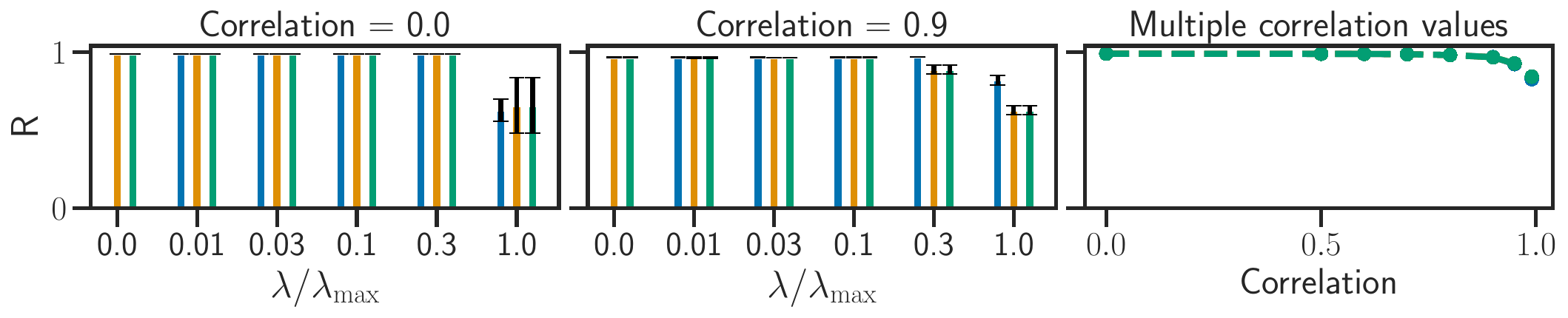}
    \includegraphics[width=0.82\columnwidth]{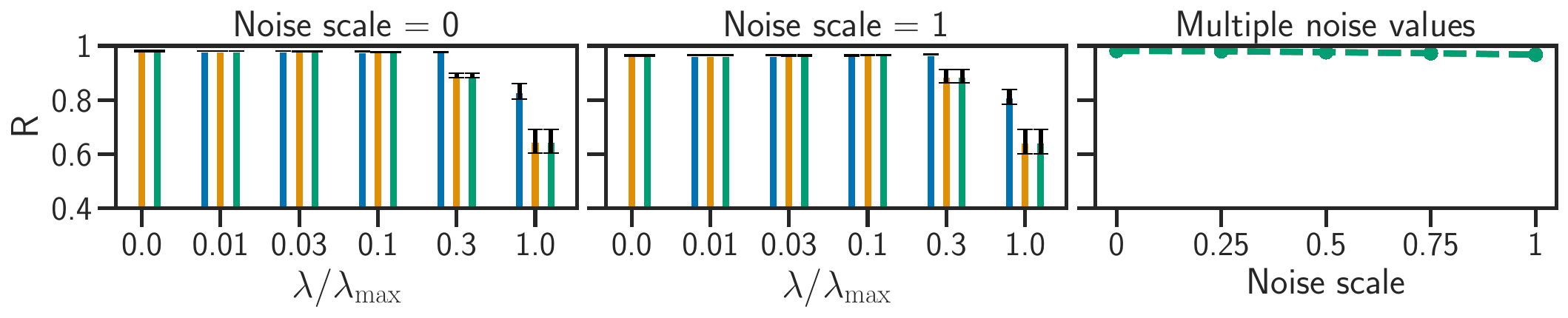}
    \caption{Prediction performance (R Score) for inner-Lasso, inner-Ridge and inner-Ridge combined with ICA as a function of the regularization parameter (left and middle). Varying  level of correlation between latents (top) and noise on the latents (bottom). The right columns shows performance of the best hyperparameter for different values of correlation and noise levels.}
    \label{fig:3dshape_influ_corr_r}
\end{figure}

\subsubsection{Dataset generation}\label{app:data_gen_3dshapes}

\textbf{Details on 3D Shapes.} The 3D Shapes dataset~\citep{3dshapes18} contains synthetic images of colored shapes resting in a simple 3D scene. These images vary across 6 factors: Floor hue (10 values linearly spaced in [0, 1]); Wall hue (10 values linearly spaced in [0, 1]); Object hue (10 values linearly spaced in [0, 1]); Scale (8 values linearly spaced in [0, 1]); Shape (4 values in [0, 1, 2, 3]); and Orientation (15 values linearly spaced in [-30, 30]). These are the factors we aim to disentangle. We standardize them to have mean 0 and variance 1. We denote by $\gZ \subset \sR^6$, the set of all possible latent factor combinations. In our framework, this corresponds to the support of the ground-truth features $\vf_\vtheta(\vx)$. We note that the points in $\gZ$ are arranged in a grid-like fashion in $\sR^6$.

\textbf{Task generation.} For all tasks $t$, the labelled dataset $\gD_t = \{(\vx^{(t,i)}), y^{(t,i)})\}_{i=1}^n$ is generated by first sampling the ground-truth latent variables $\vz^{(t,i)} := \vf_{\vtheta}(\vx^{(t,i)})$ i.i.d. according to some distribution $p(\vz)$ over $\gZ$, while the corresponding input is obtained doing $\vx^{(t,i)} := \vf_\vtheta^{-1}(\vz^{(t,i)})$ ($\vf_\vtheta$ is invertible in 3D Shapes). Then, a sparse weight vector $\vw^{(t)}$ is sampled randomly by doing $\vw^{(t)} := \bar\vw^{(t)} \odot \vs^{(t)}$, were $\odot$ is the Hadamard (component-wise) product, $\bar\vw^{(t)} \sim \gN(\bm0, I)$ and $\vs \in \{0,1\}^{6}$ is a binary vector with independent components sampled from a Bernoulli distribution with ($p=0.5$). Then, the labels are computedfor each example as $y^{(t,i)} := \vw^{(t)} \cdot \vx^{(t,i)} + \epsilon^{(t,i)}$, where $\epsilon^{(t,i)}$ is independent Gaussian noise. In every task, the dataset has size $n=50$. New tasks are generated continuously as we train. \Cref{fig:3dshape_mcc,fig:3dshape_influ_corr_r} explores various choices of $p(\vz)$, \ie by varying the level of correlation between the latent variables and by varying the level of noise on the ground-truth latents. \Cref{fig:distributions_over_latents} shows a visualization of some of these distributions over latents.

\textbf{Noise on latents.} To make the dataset slightly more realistic, we get rid of the artificial grid-like structure of the latents by adding noise to it. This procedure transforms $\gZ$ into a new support $\gZ_\alpha$, where $\alpha$ is the noise level. Formally, $\gZ_\alpha := \bigcup_{\vz \in \gZ} \{\vz + \vu_\vz\}$ where the $\vu_z$ are i.i.d samples from the uniform over the hypercube
$$\left[-\alpha\frac{\Delta \vz_1}{2}, \alpha\frac{\Delta \vz_1}{2}\right] \times\left[-\alpha\frac{\Delta \vz_2}{2}, \alpha\frac{\Delta \vz_2}{2}\right] \times \hdots \times \left[-\alpha\frac{\Delta \vz_6}{2}, \alpha\frac{\Delta \vz_6}{2}\right]\, ,$$
where $\Delta\vz_i$ denotes the gap between contiguous values of the factor $\vz_i$. When $\alpha=0$, no noise is added and the support $\gZ$ is unchanged, \ie $\gZ_1 = \gZ$. As long as $\alpha \in [0,1]$, contiguous points in $\gZ$ cannot be interchanged in $\gZ_\alpha$. We also clarify that the ground-truth mapping $\vf_\vtheta$ is modified to $\vf_{\vtheta,\alpha}$ consequently: for all $\vx \in \gX$, $\vf_{\vtheta,\alpha}(\vx) := \vf_\vtheta(\vx) + \vu_\vz$. We emphasize that the $\vu_\vz$ are sampled only once such that $\vf_{\vtheta,\alpha}(\vx)$ is actually a deterministic mapping.

\textbf{Varying correlations.} To verify that our approach is robust to correlations in the latents, we construct $p(\vz)$ as follows: We consider a Gaussian density centred at $\bm0$ with covariance $\bm\Sigma_{i,j} := \rho + \mathbbm{1}(i = j)(1-\rho)$. Then, we evaluate this density on the points of $\gZ_\alpha$ and renormalize to have a well-defined probability distribution over $\gZ_\alpha$. We denote by $p_{\alpha, \rho}(\vz)$ the distribution obtain by this construction.

In the top rows of~\Cref{fig:3dshape_mcc,fig:3dshape_influ_corr_r}, the latents are sampled from $p_{\alpha = 1, \rho}(\vz)$ and $\rho$ varies between 0 and 0.99. In the bottom rows of~\Cref{fig:3dshape_mcc,fig:3dshape_influ_corr_r}, the latents are sampled from $p_{\alpha, \rho = 0.9}(\vz)$ and $\alpha$ varies from 0 to 1.

\begin{figure}
    \centering
    \includegraphics[width=\linewidth]{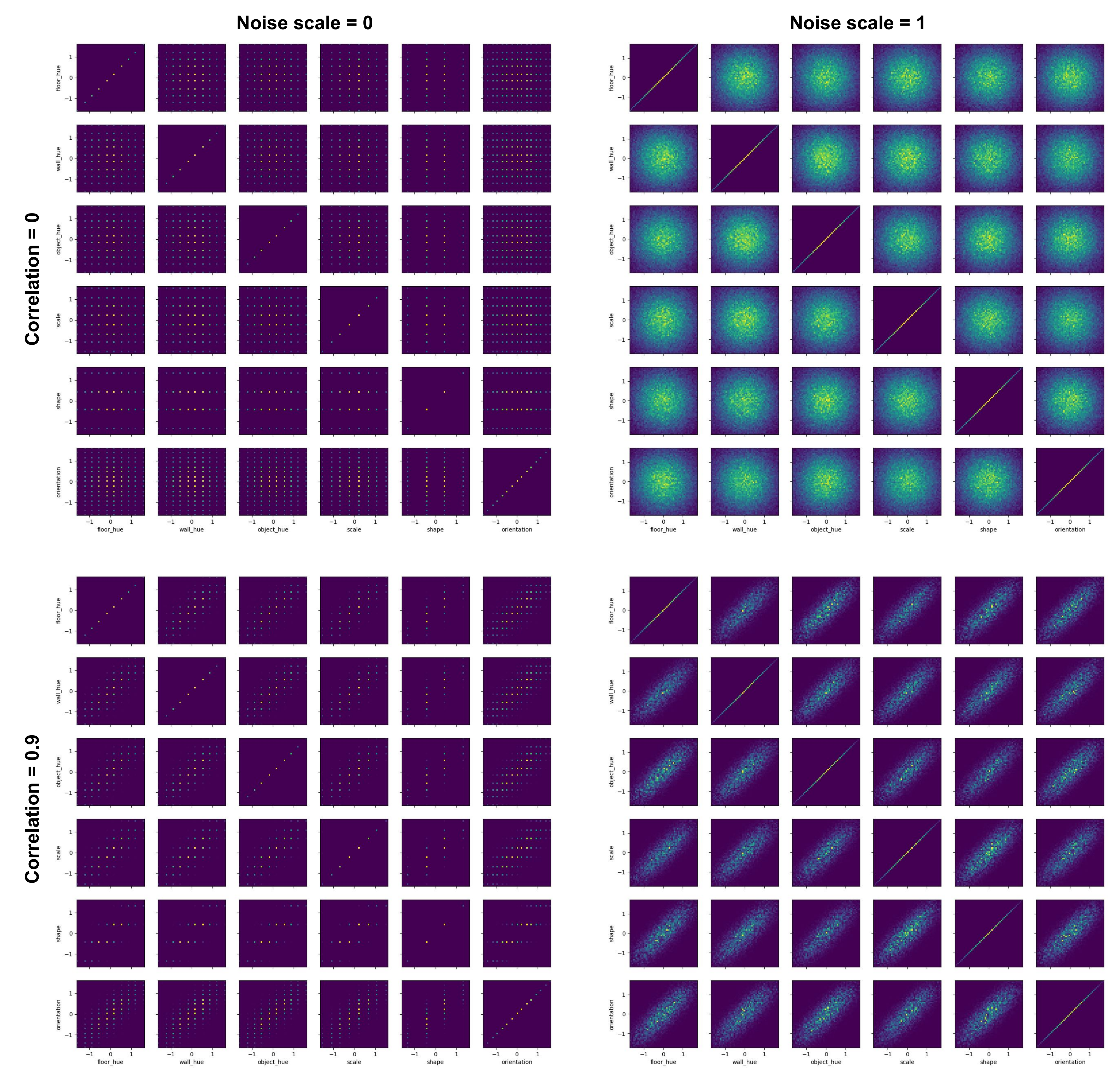}
    \caption{\textbf{Visualization of the various distributions over latents.} For 4 combinations of correlation levels and noise levels, we show the 2-dimensional histograms of samples from the corresponding distribution over latents described in~\Cref{app:data_gen_3dshapes}. Each histogram shows the joint distribution over two latent factors.}
    \label{fig:distributions_over_latents}
\end{figure}

\subsubsection{Metrics}
We evaluate disentanglement via the \textit{mean correlation coefficient}~\citep{TCL2016, iVAEkhemakhem20a} which is computed as follows: The Pearson correlation matrix $C$ between the ground-truth features and learned ones is computed. Then, $\text{MCC} = \max_{\pi \in \text{permutations}} \frac{1}{m} \sum_{j=1}^m |C_{j,\pi(j)}|$. We also evaluate linear equivalence by performing linear regression to predict the ground-truth factors from the learned ones, and report the mean of the Pearson correlations between the ground-truth latents and the learned ones. This metric is known as the \textit{coefficient of multiple correlations}, $R$, and turns out to be the square-root of the more widely known \textit{coefficient of determination}, $R^2$. The advantage of using $R$ over $R^2$ is that we always have $\text{MCC} \leq R$.

\subsubsection{Architecture, inner solver \& hyperparameters} \label{app:arch_solver_hyperparams}
We use the four-layer convolutional neural network typically used in the disentanglement literature~\citep{pmlr-v97-locatello19a}. As mentioned in~\Cref{sec:sparse_bilevel}, the norm of the representation $\vf_{\hat\vtheta}(\vx)$ must be controlled to make sure the regularization remains effective. To do so, we apply batch normalization~\citep{batchnorm2015} at the very last layer of the neural network and do not learn its scale and shift parameters. Empirically, we do see the expected behavior that, without any normalization, the norm of $\vf_{\hat\vtheta}(\vx)$ explodes as we train, leading to instabilities and low sparsity.

In these experiments, the distribution $p(y; \bm\eta)$ used for learning is a Gaussian with fixed variance. In that case, the inner problem of \Cref{sec:sparse_bilevel} reduces to Lasso regression. Computing the hypergradient w.r.t. $\vtheta$ requires solving this inner problem. To do so, we use Proximal Coordinate Descent~\citep{Tseng2001,Richtarik2014}.

\textbf{Details on $\lambda / \lambda_{\max}$.} In~\Cref{fig:3dshape_mcc,fig:3dshape_influ_corr_r}, we explore various levels of regularization $\lambda$. In our implementation, we set $\lambda = \epsilon \lambda_{\max}$ where $\epsilon \geq 0$. In inner-Lasso, we set $\lambda_{\max} := \frac{1}{n}\normin{\mF^\top \vy}_\infty$ ($\mF \in \sR^{n \times m}$ is the design matrix of the features of the samples of a task), while in inner-Ridge we have $\lambda_{\max} := \frac{1}{n}\normin{\mF}^2$. Note that this means $\lambda$ is dynamically changing as we train because $\mF$ changes. However we never backpropagate through $\lambda_{\max}$ (we block the gradient from flowing). Thus, in all figures, we report $\epsilon = \lambda / \lambda_{\max}$.

\subsubsection{{Experiments violating assumptions}}
\label{app:experiment_violation}
{In this section, we explore variations of the experiments of~\Cref{sec:dis_exp}, but this time the assumptions of \Cref{thm:disentanglement_via_optim} are violated.}

{\Cref{fig:3dshape_influ_support_violation_r_mcc_binomial} shows different degrees of violation of~\Cref{ass:suff_support}. We consider the cases where $\gS := \{\{1,2\}, \{3,4\}, \{5,6\}\}$ (block size = 2), $\gS := \{\{1,2,3\}, \{4,5,6\}\}$ (block size = 3) and $\gS := \{\{1,2,3,4,5,6\}\}$ (block size = 6). Note that the latter case corresponds to having no sparsity at all in the ground-truth model, \ie all tasks require all features. The reader can verify that these three cases indeed violate \Cref{ass:suff_support}. In all cases, the distribution $p(S)$ puts uniform mass over its support $\gS$. Similarly to the experiments from the main text, $\vw := \bar{\vw} \odot \vs$, where $\bar\vw \sim \mathcal{N}({\bm 0}, \mI)$ and $\vs \sim p(S)$ ($\vs$ is the binary representation of the set $S$). Overall, we can see that inner-Lasso does not perform as well when \Cref{ass:suff_support} is violated. For example, when there is no sparsity at all (block size = 6), inner-Lasso performs poorly and is even surpassed by inner-Ridge. Nevertheless, for mild violations (block size = 2), disentanglement (as measured by MCC) remains reasonably high. We further notice that all methods obtain very good R score in all settings. This is expected in light of \Cref{thm:linear_ident}, which guarantees identifiability up to linear transformation without requiring \Cref{ass:suff_support}.}

\begin{figure}[h]
    \centering
    \includegraphics[width=0.5\columnwidth]{figures/3dshape_influ_corr_fulllegend.png}
    \includegraphics[width=0.82\columnwidth]{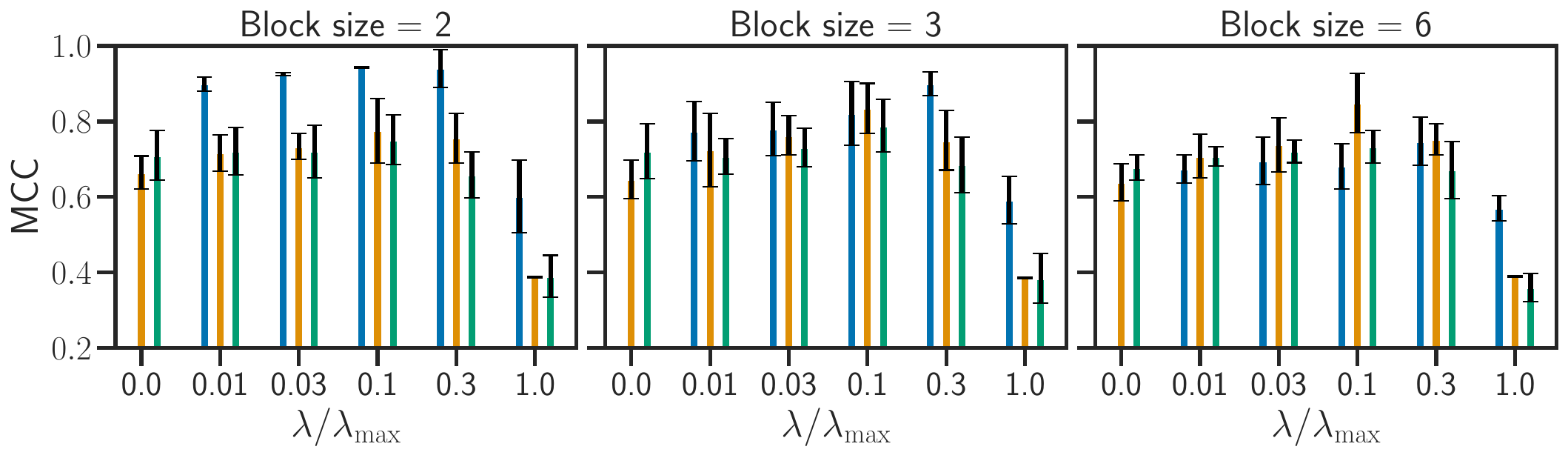}
    \includegraphics[width=0.82\columnwidth]{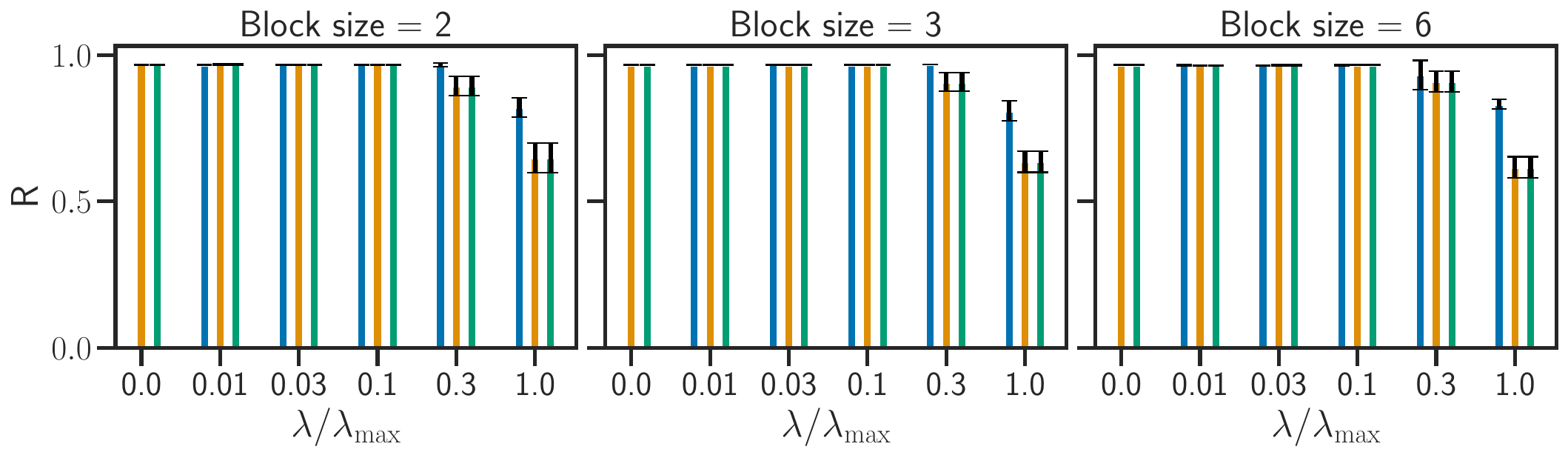}
    \caption{{
        Disentanglement (MCC, top) and prediction (R Score, bottom) performances for inner-Lasso, inner-Ridge and inner-Ridge combined with ICA as a function of the regularization parameter.
    The metrics are plotted for multiple value of block size for the support. Block size $=6$ corresponds to no sparsity in the ground truth coefficients.
    }
    }
    \label{fig:3dshape_influ_support_violation_r_mcc_binomial}
\end{figure}

{\Cref{fig:3dshape_influ_corr_mcc_laplace} presents experiments that are identitical to those of~\Cref{fig:3dshape_mcc} in the main text, except for how $\vw$ is generated. Here, the components of $\vw$ are sampled independently according to $\vw_i \sim \mathrm{Laplace}(\mu = 0, b = 1)$. We note that, under this process, the probability that $\vw_i = 0$ is zero. This means all features are useful and \Cref{ass:suff_support} is violated. That being said, due to the fat tail behavior of the Laplacian distribution, many components of $\vw$ will be close to zero (relatively to its variance). Thus, this can be thought of as a weaker form of sparsity where many features are relatively unimportant. \Cref{fig:3dshape_influ_corr_mcc_laplace} shows that inner-Lasso can still disentangle very well. In fact, the performance is very similar to the experiments that presented actual sparsity~(\Cref{fig:3dshape_mcc}).}

\begin{figure}[h]
    \centering
    \includegraphics[width=0.5\columnwidth]{figures/3dshape_influ_corr_fulllegend.png}
    \includegraphics[width=0.82\columnwidth]{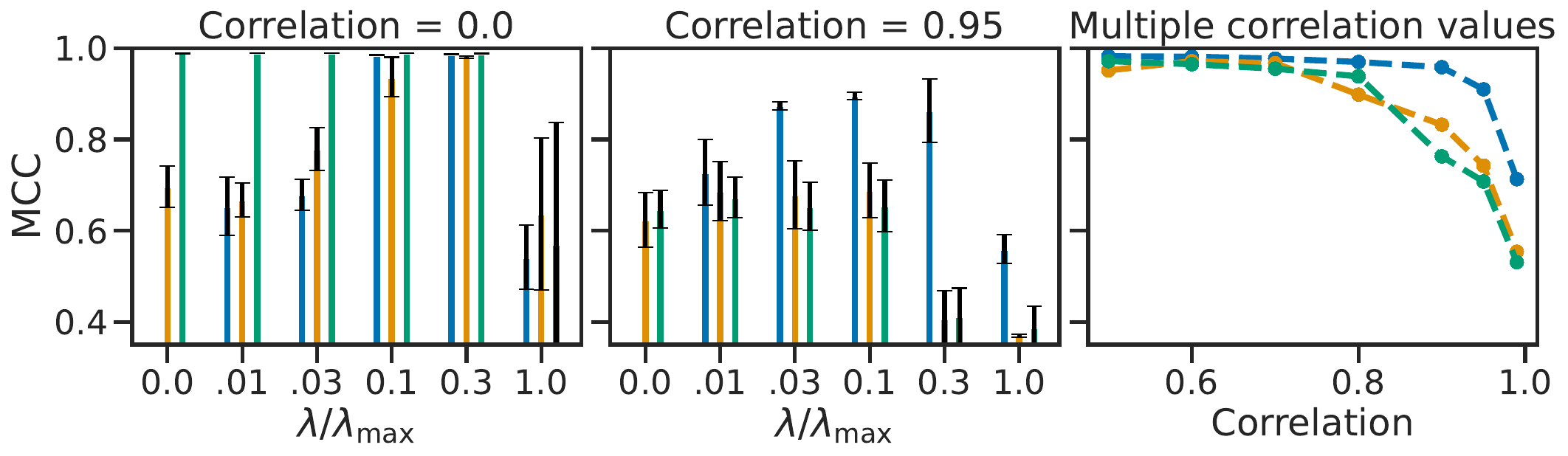}
    \includegraphics[width=0.82\columnwidth]{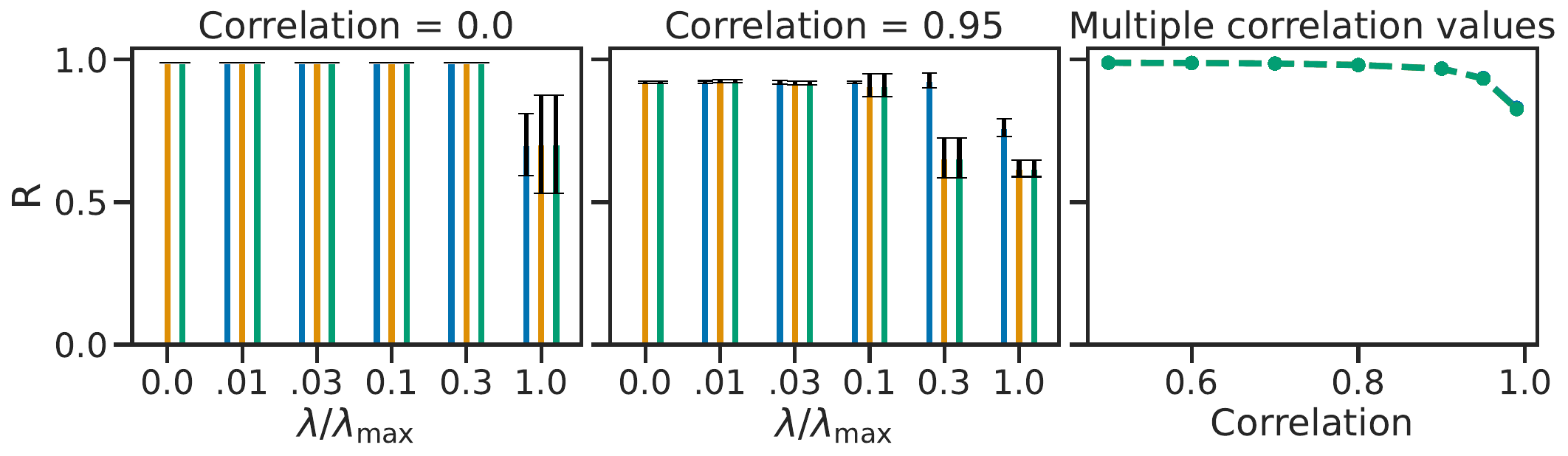}
    \includegraphics[width=0.82\columnwidth]{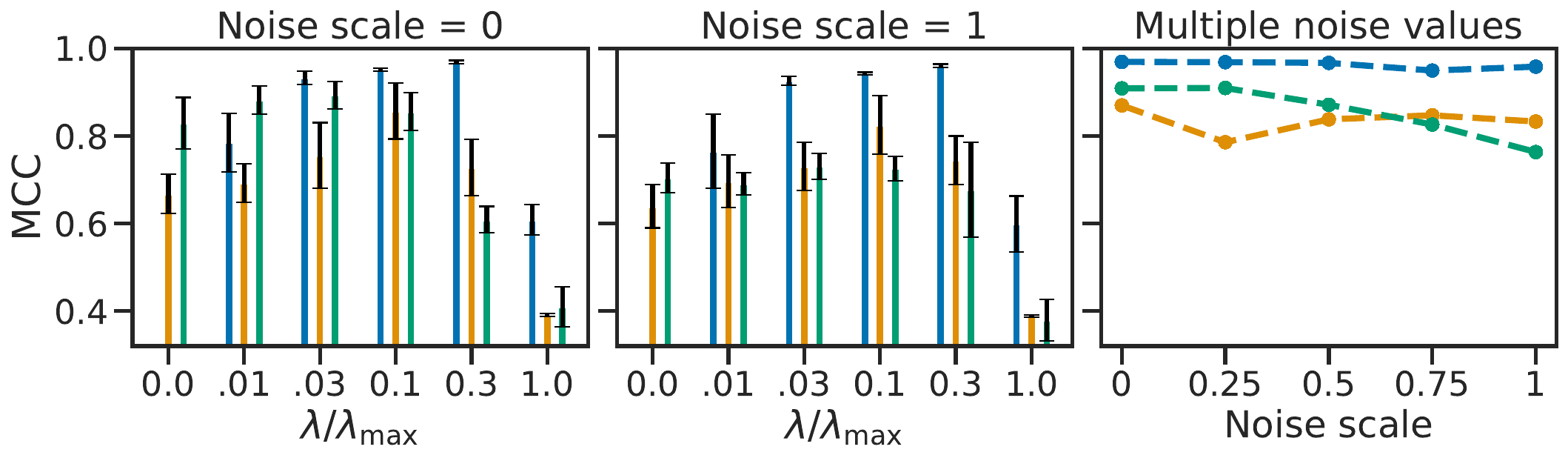}
    \includegraphics[width=0.82\columnwidth]{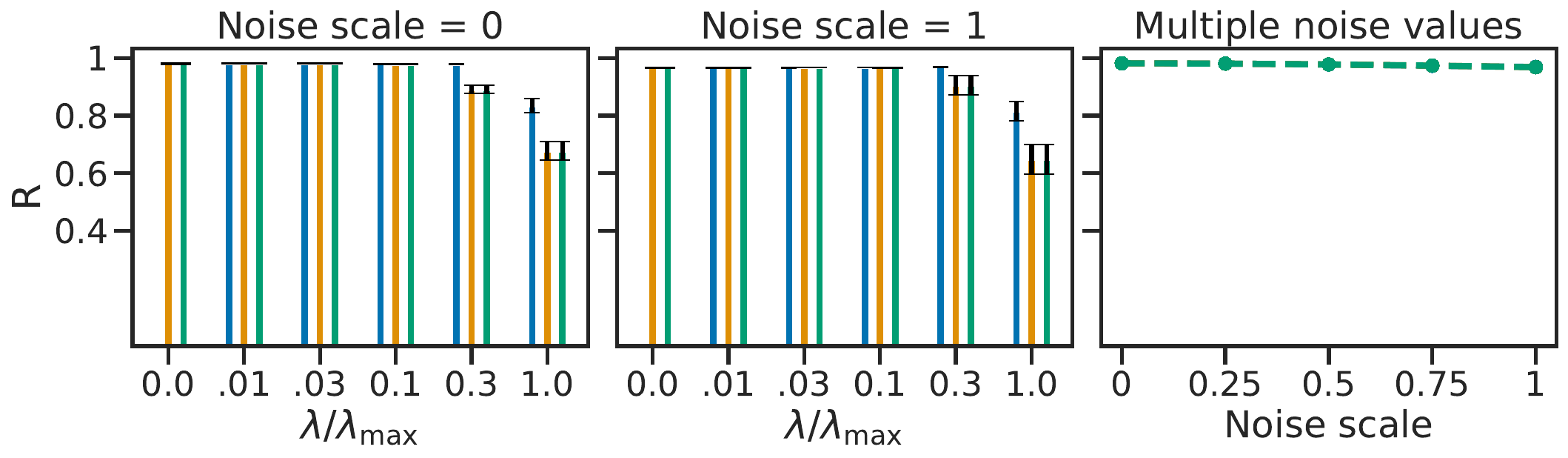}
    \caption{{
        Same experiment as~\Cref{fig:3dshape_mcc}, but the task coefficient vectors $\vw$ are sampled from a Laplacian distribution (instead of what was described in \Cref{app:data_gen_3dshapes}). Performance is barely affected, showing some amount of robustness to violations of \Cref{ass:suff_support}.
    }}
    \label{fig:3dshape_influ_corr_mcc_laplace}
\end{figure}

\subsubsection{Experiments with regularization in the outer problem}\label{app:exp_outer_reg}
\Cref{thm:disentanglement_via_optim_outer} presented an alternative optimization problem to that of \Cref{thm:disentanglement_via_optim} to learn a disentangled representation. \Cref{app:relaxing_outer} presented a tractable relaxation of this alternative. The essential operational difference is that the sparsity regulatization appears in the outer problem instead of the inner problem. Figure~\ref{fig:outer_regularization_exp} shows this alternative works as well empirically. Details in the caption.

\begin{figure}[ht]
    \centering
    \includegraphics[width=0.8\linewidth]{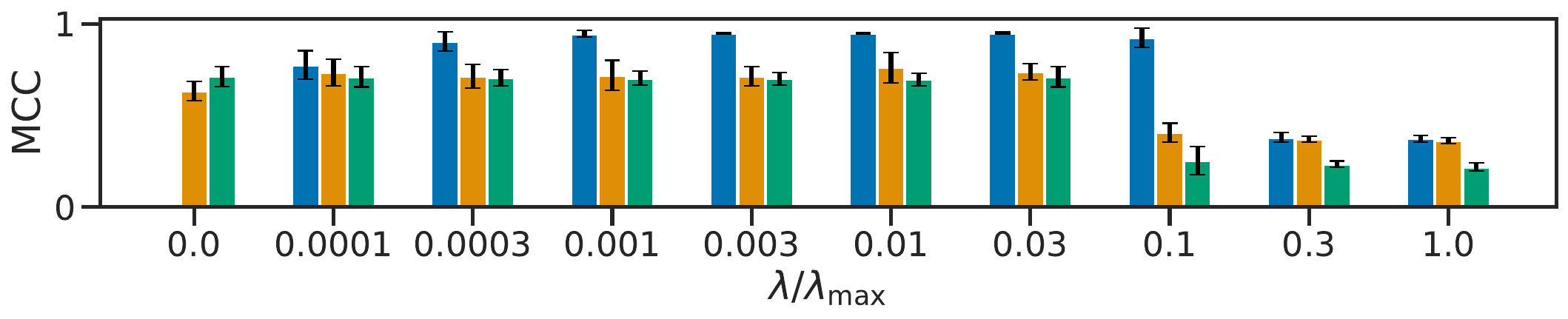} \includegraphics[width=0.8\linewidth]{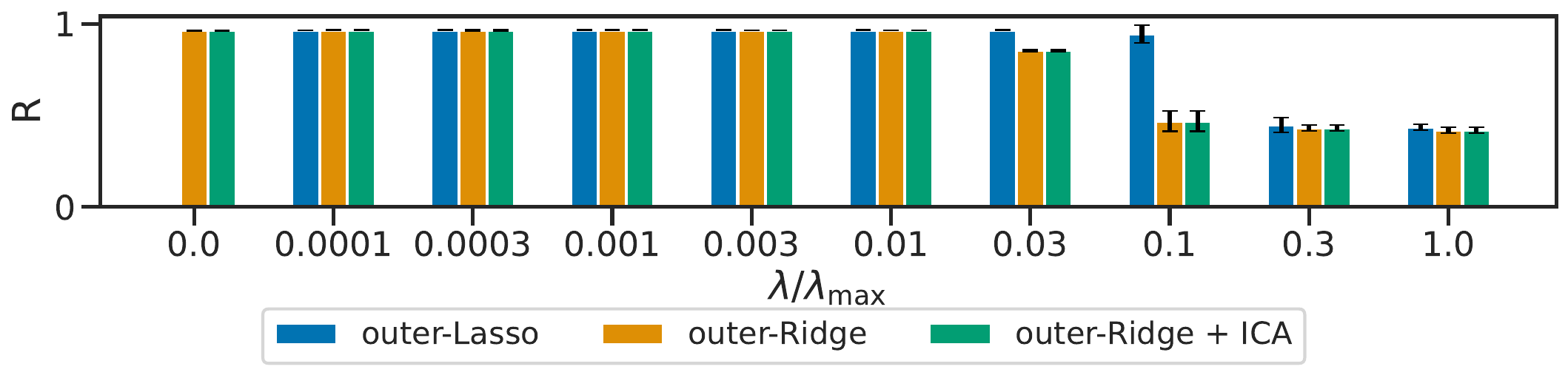}
    \caption{\textbf{outer-Lasso} solves \Cref{pb:OG_problem_outer_relax} (with regularization in the outer problem) while \textbf{outer-Ridge} solves the same problem but with an $L_{2}$-norm instead of $L_{2,1}$. The method \textbf{outer-Ridge + ICA} is outer-Ridge with an additional step of linear ICA on top of the learned representation. The results obtained are very similar to the main results of \Cref{fig:3dshape_mcc,fig:3dshape_influ_corr_r}. In this dataset, the latents are sampled from $p_{\alpha = 1, \rho=0.9}(\vz)$ (See~\Cref{app:data_gen_3dshapes}) and the weight coefficients are sampled from the binomial-Gaussian process described in \Cref{app:data_gen_3dshapes}.}
    \label{fig:outer_regularization_exp}
\end{figure}

\subsubsection{{Visual evaluation}}\label{app:visual_eval}

{\Cref{fig:latent-responses_lasso_no_corr,fig:latent-responses_no_reg_no_corr,fig:latent-responses_lasso_corr,fig:latent-responses_ridge_corr} show how various learned representations respond to changing a single factor of variation in the image~\citep[Figure 7.A.B]{Higgins2017betaVAELB}. We see what was expected: the higher the MCC, the more disentangled the learned features appear, thus validating MCC as a good metric for disentanglement. See captions for details.}

\begin{figure}[h]
    \centering
    \includegraphics[width=0.65\columnwidth]{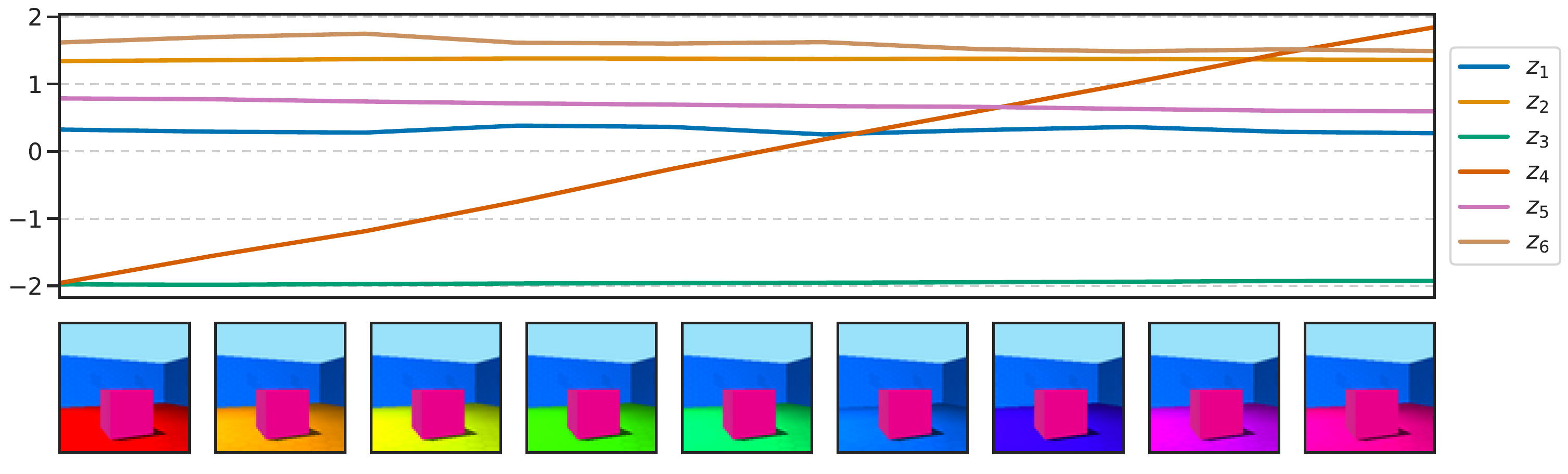}
    \includegraphics[width=0.65\columnwidth]{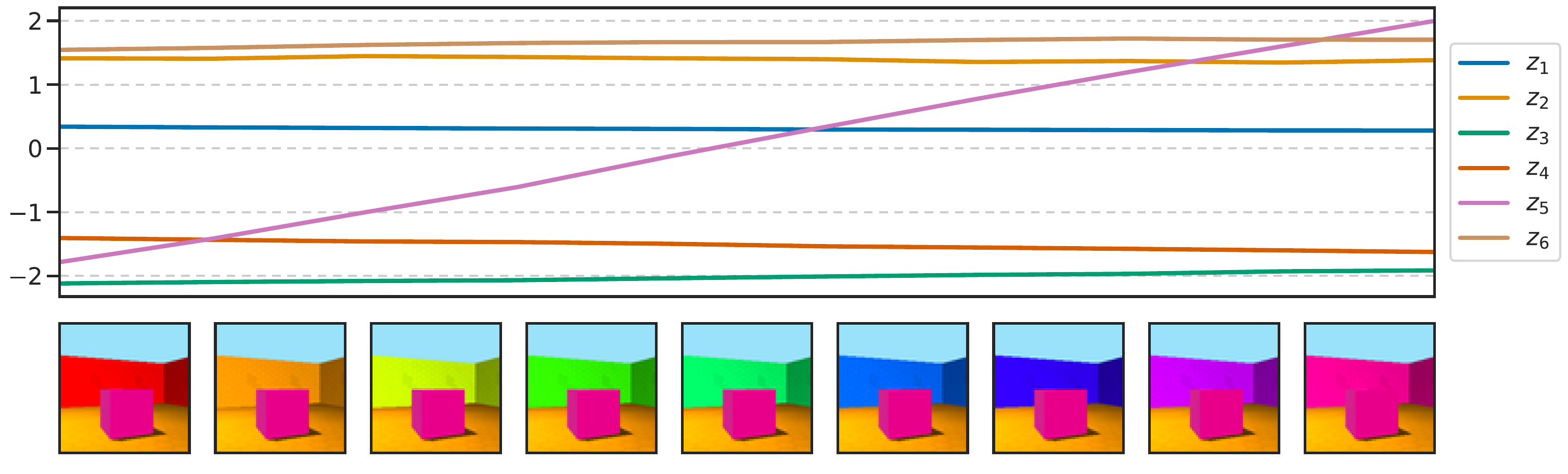}
    \includegraphics[width=0.65\columnwidth]{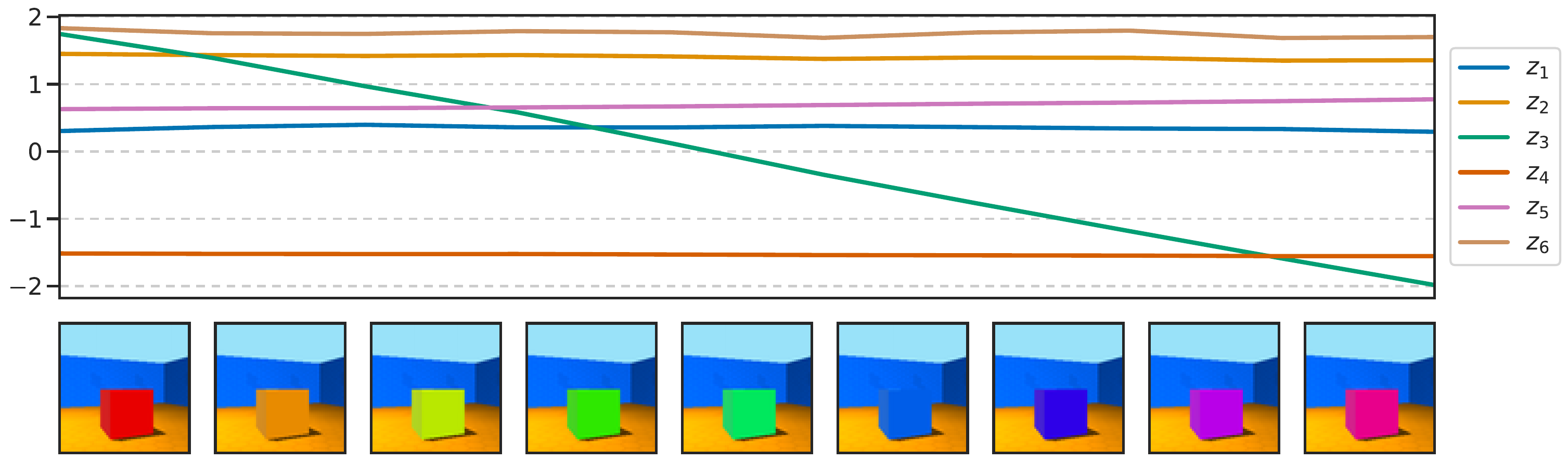}
    \includegraphics[width=0.65\columnwidth]{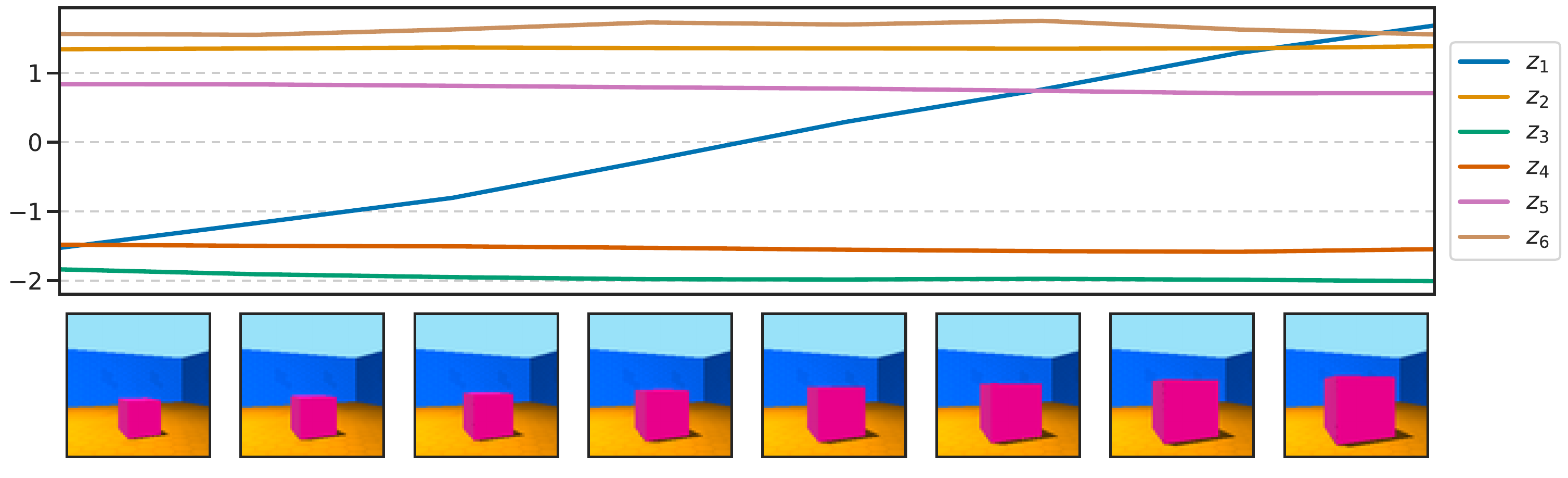}
    \includegraphics[width=0.65\columnwidth]{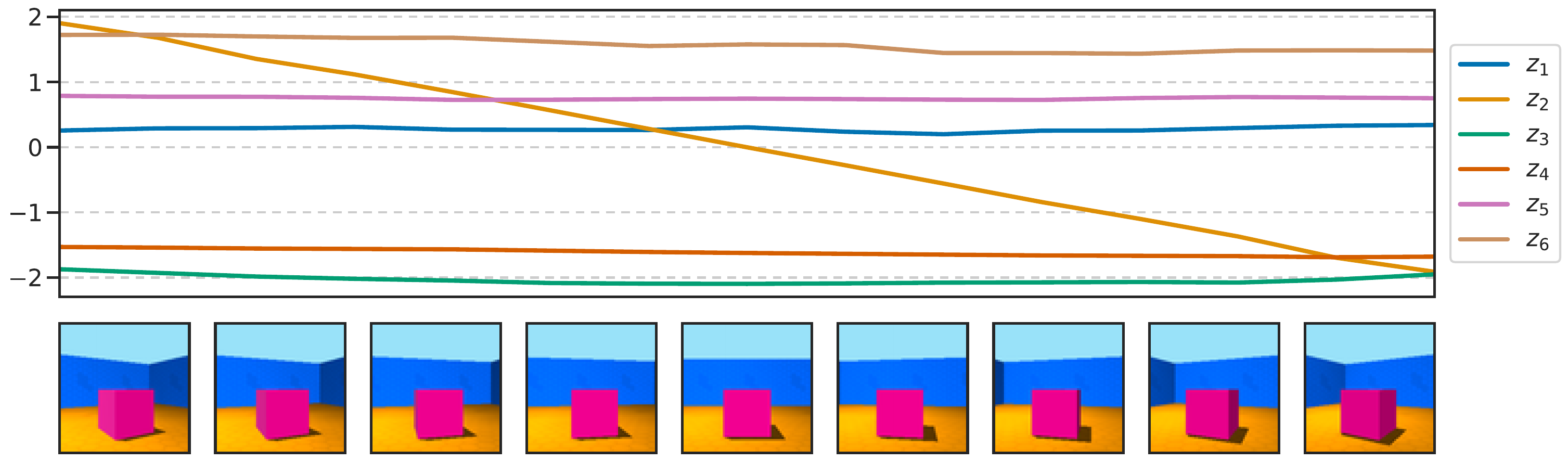}
    \includegraphics[width=0.65\columnwidth]{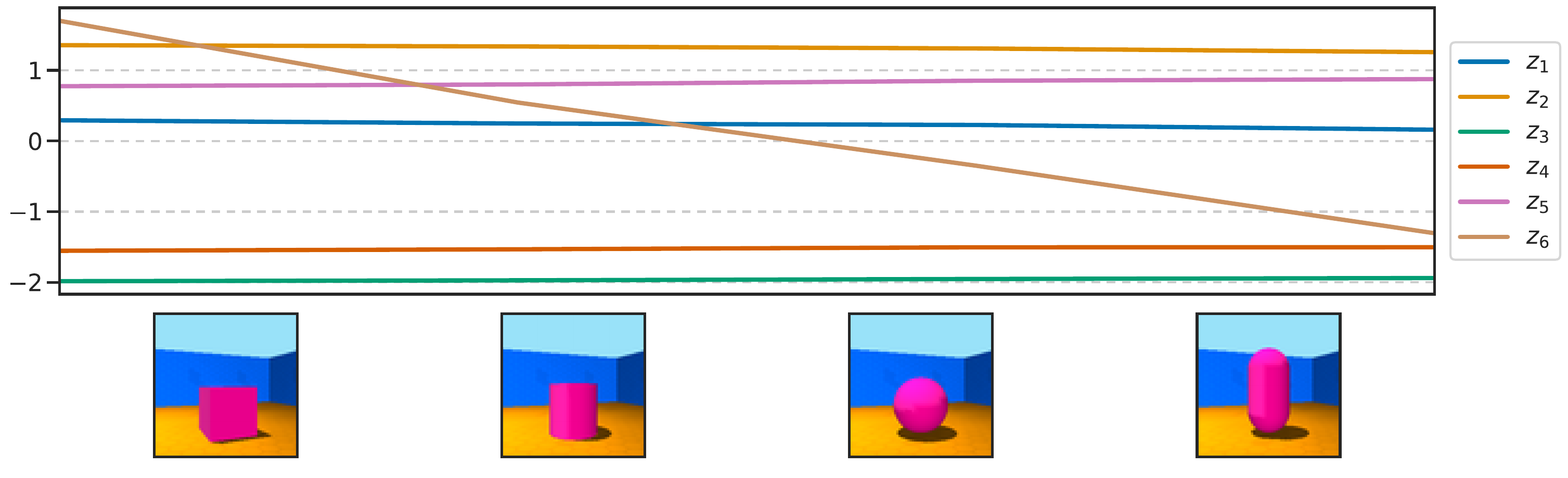}
    \caption{{Varying one factor at a time in the image and showing how the learned representation varies in response. This representation was learned by \textbf{inner-Lasso} (best hyperparameter) on a dataset with \textbf{0 correlation between latents} and a noise scale of 1. The corresponding \textbf{MCC is 0.99}. We can see that varying a single factor in the image always result in changing a single factor in the learned representation.}}
    \label{fig:latent-responses_lasso_no_corr}
\end{figure}

\begin{figure}[h]
    \centering
    \includegraphics[width=0.65\columnwidth]{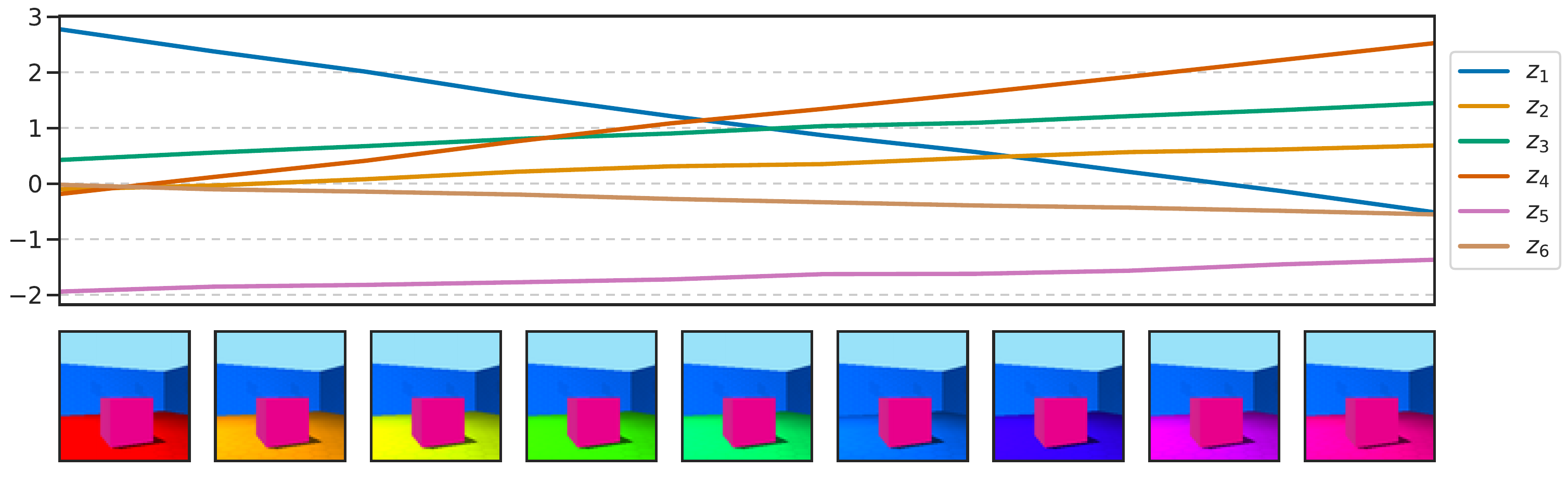}
    \includegraphics[width=0.65\columnwidth]{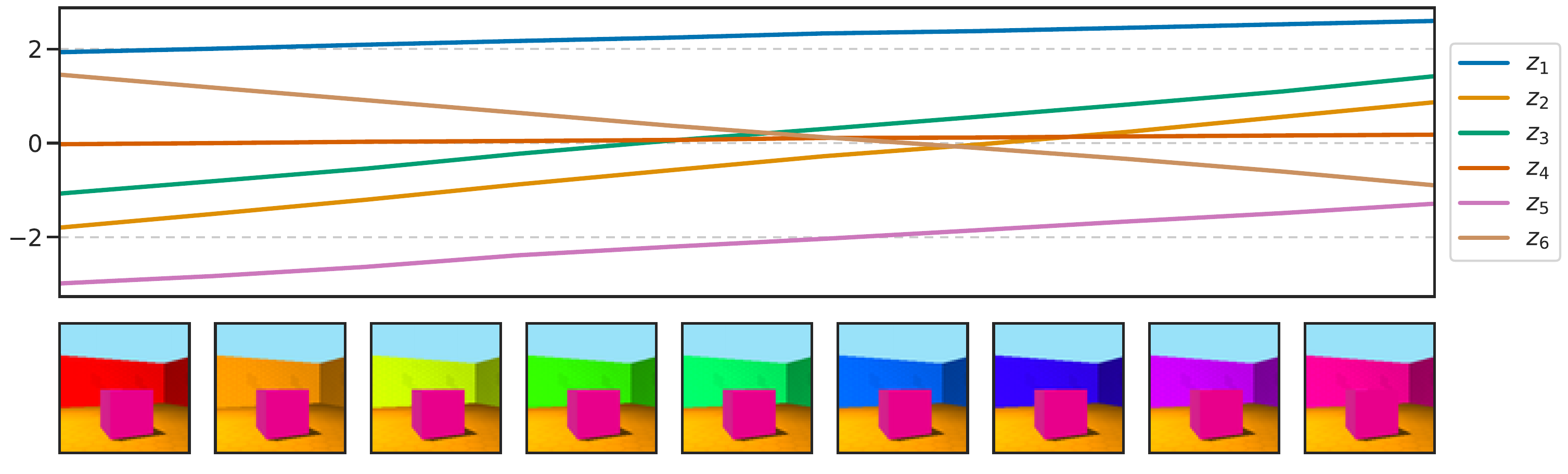}
    \includegraphics[width=0.65\columnwidth]{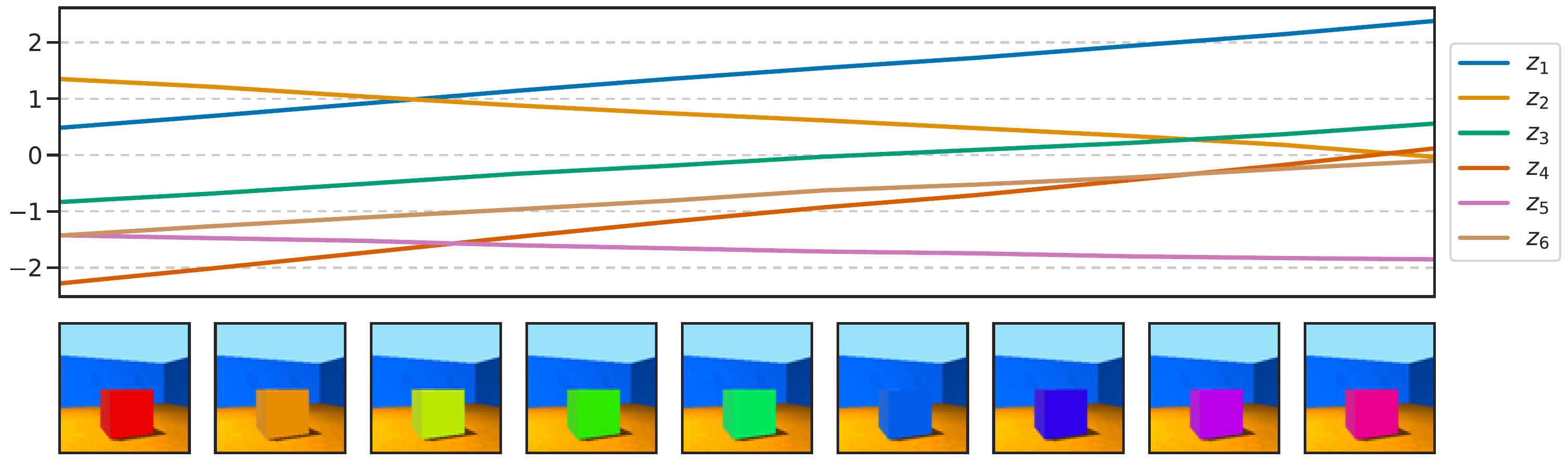}
    \includegraphics[width=0.65\columnwidth]{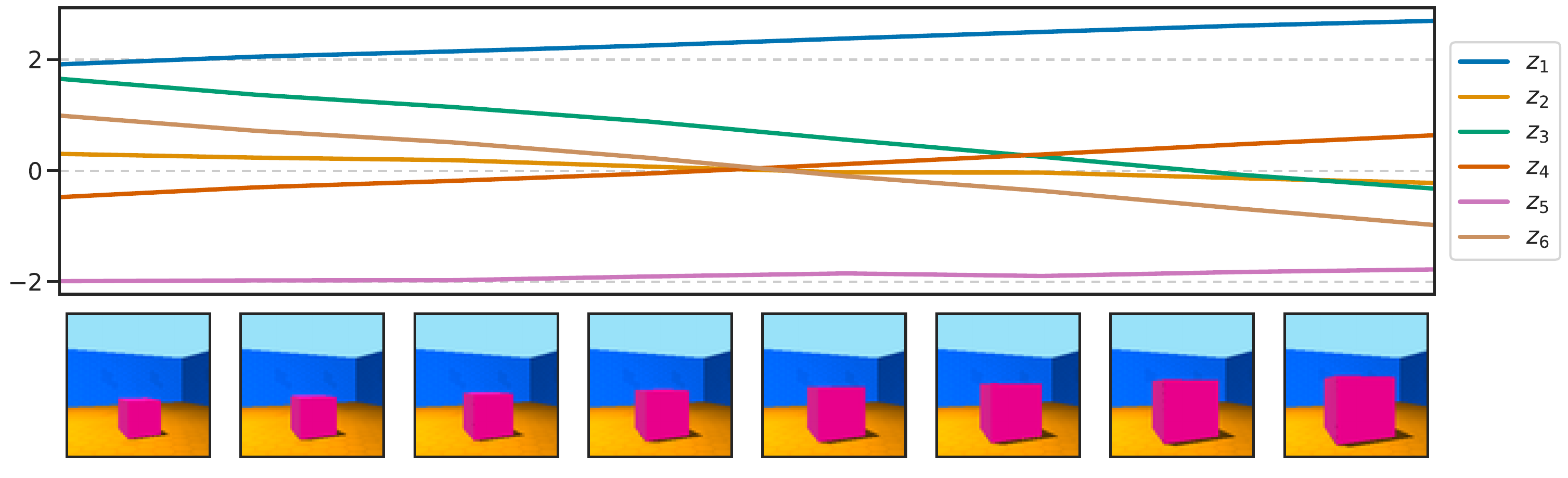}
    \includegraphics[width=0.65\columnwidth]{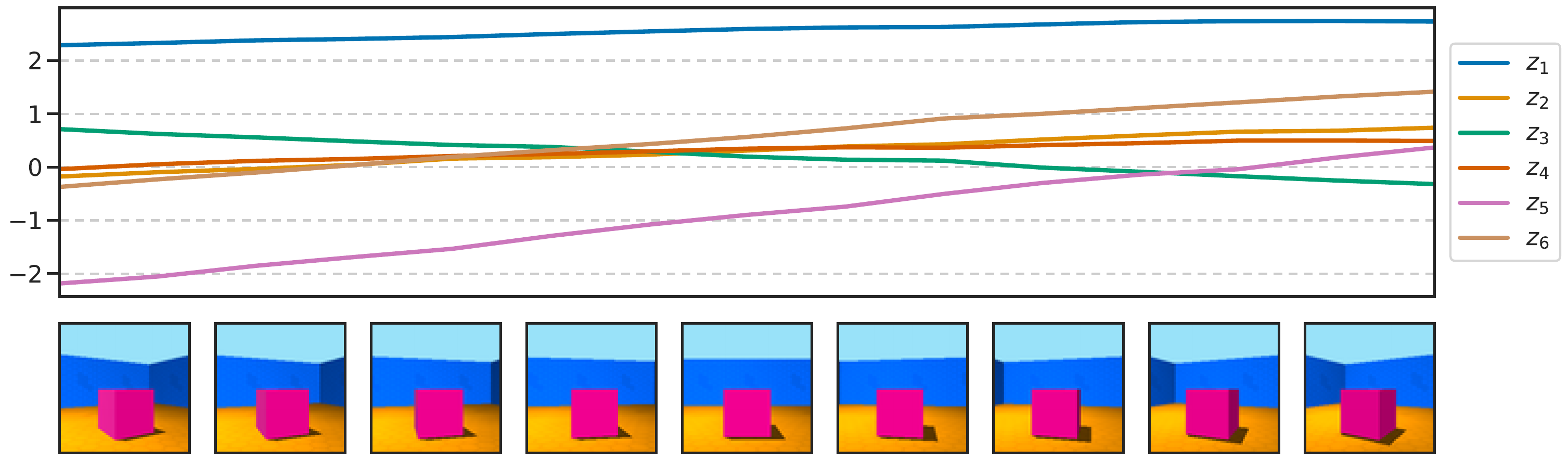}
    \includegraphics[width=0.65\columnwidth]{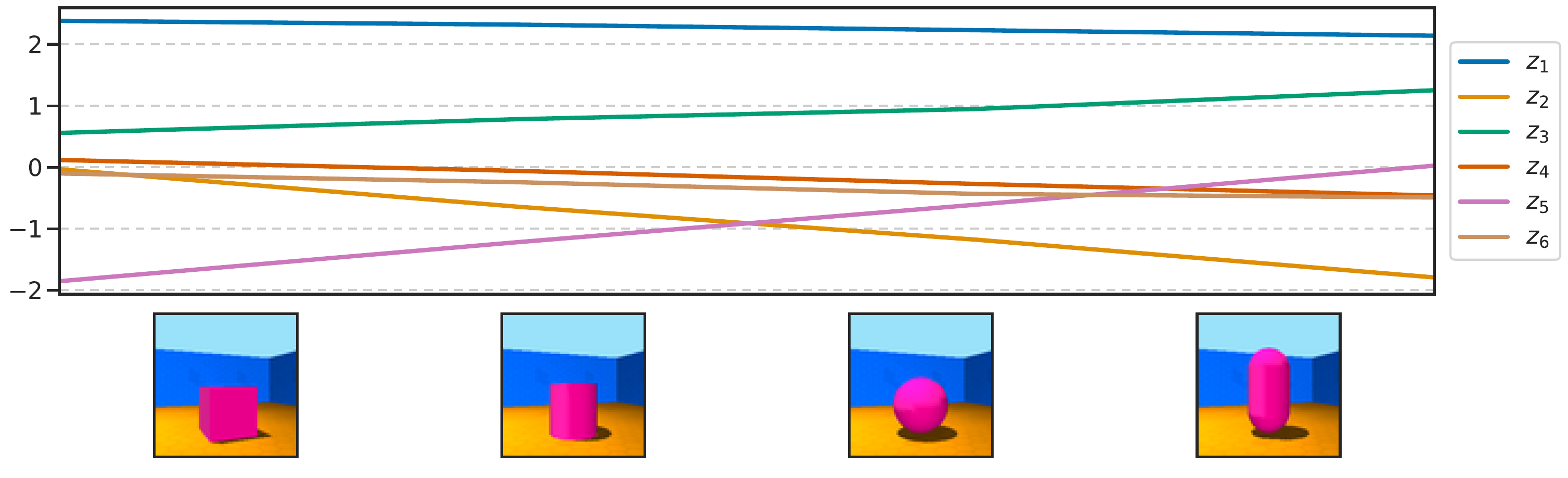}
    \caption{{Varying one factor at a time in the image and showing how the learned representation varies in response. This representation was learned \textbf{without regularization} of any kind (\ie with inner-Ridge with regularization coefficient equal to zero) on a dataset with \textbf{0 correlation between} and a noise scale of 1. The corresponding \textbf{MCC is 0.63}. We can see that varying a single factor in the image result in changing multiple factors in the learned representation, \ie the representation is not disentangled.}}
    \label{fig:latent-responses_no_reg_no_corr}
\end{figure}

\begin{figure}[h]
    \centering
    \includegraphics[width=0.65\columnwidth]{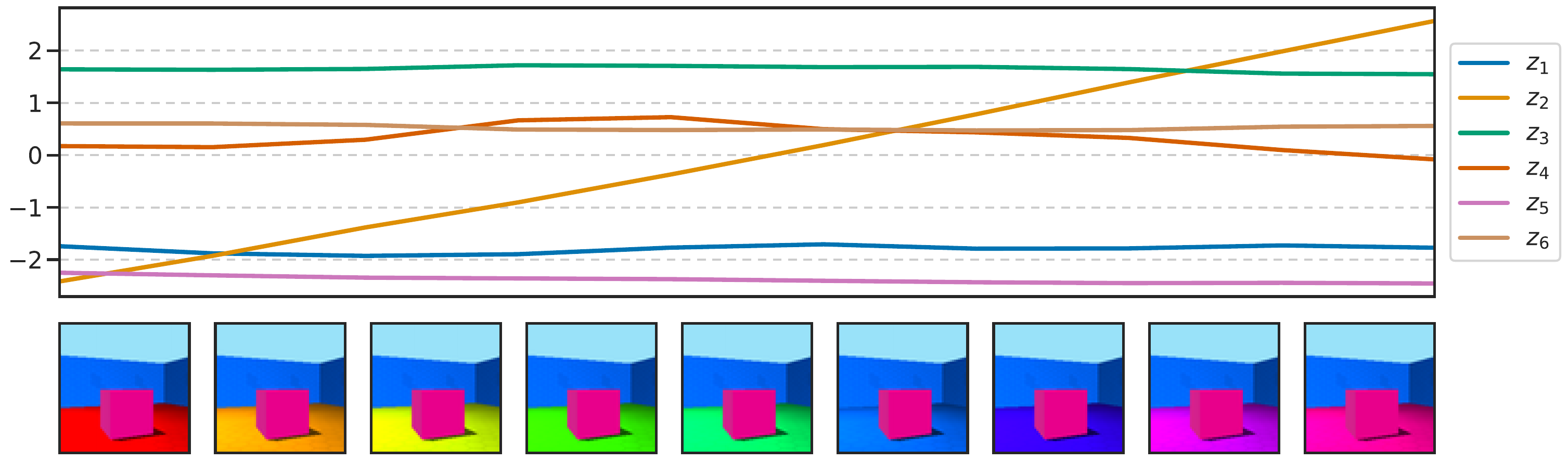}
    \includegraphics[width=0.65\columnwidth]{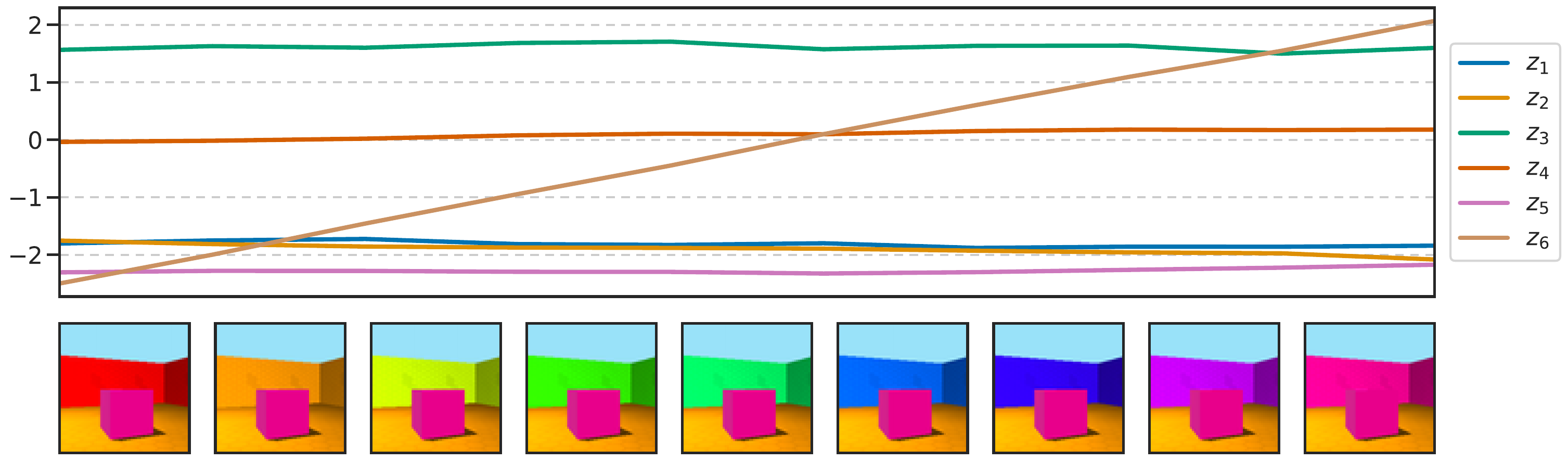}
    \includegraphics[width=0.65\columnwidth]{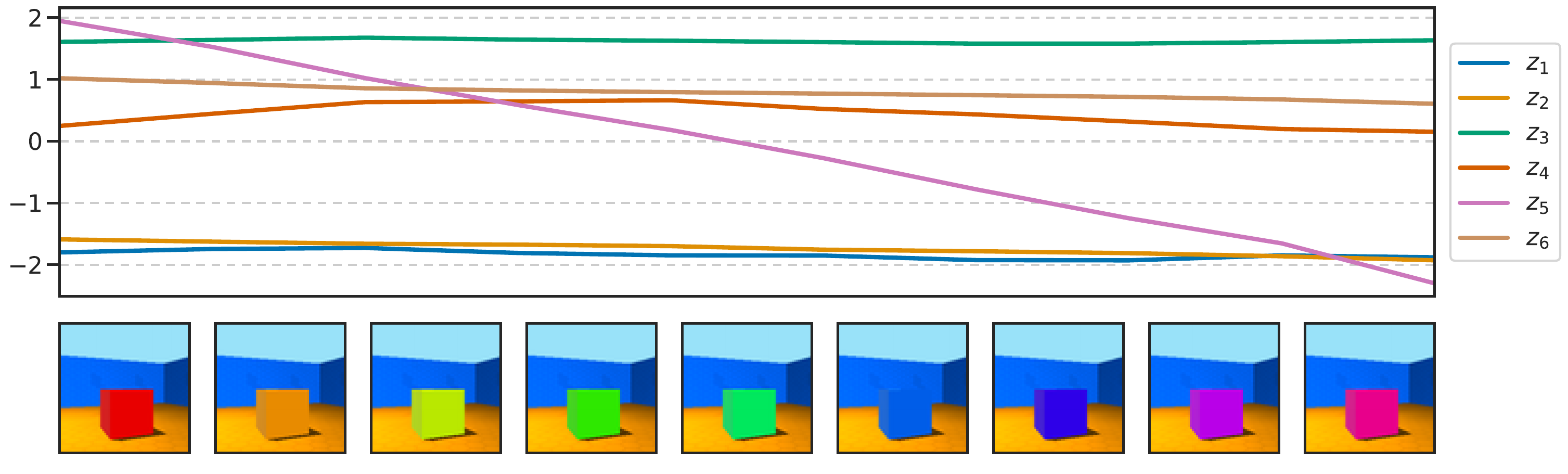}
    \includegraphics[width=0.65\columnwidth]{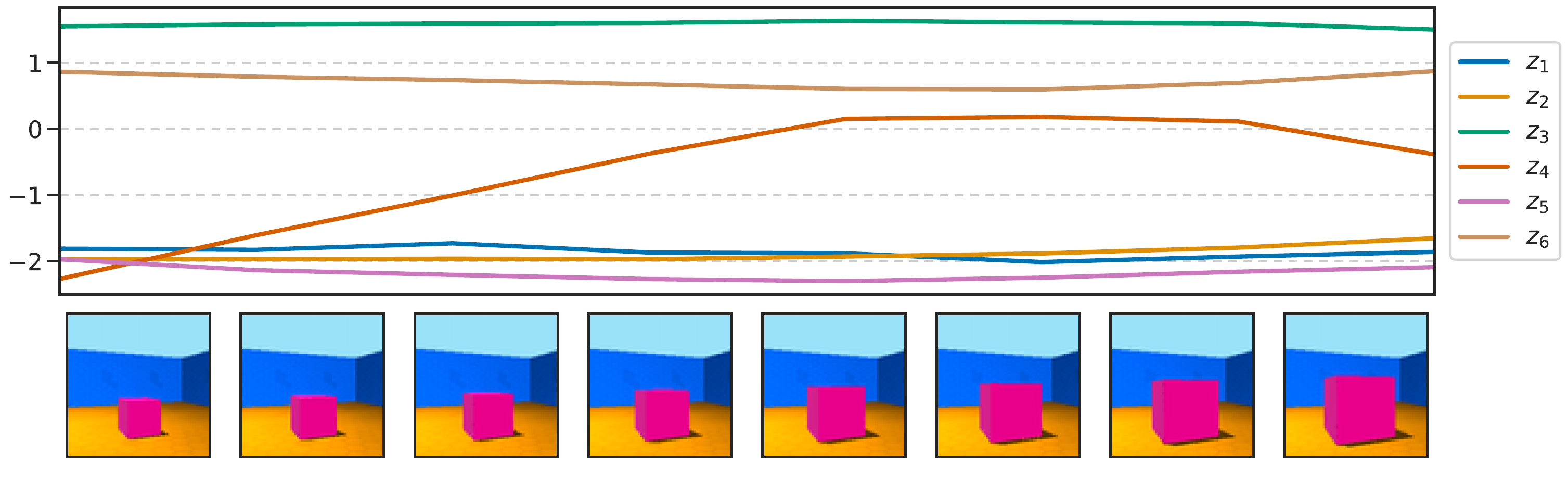}
    \includegraphics[width=0.65\columnwidth]{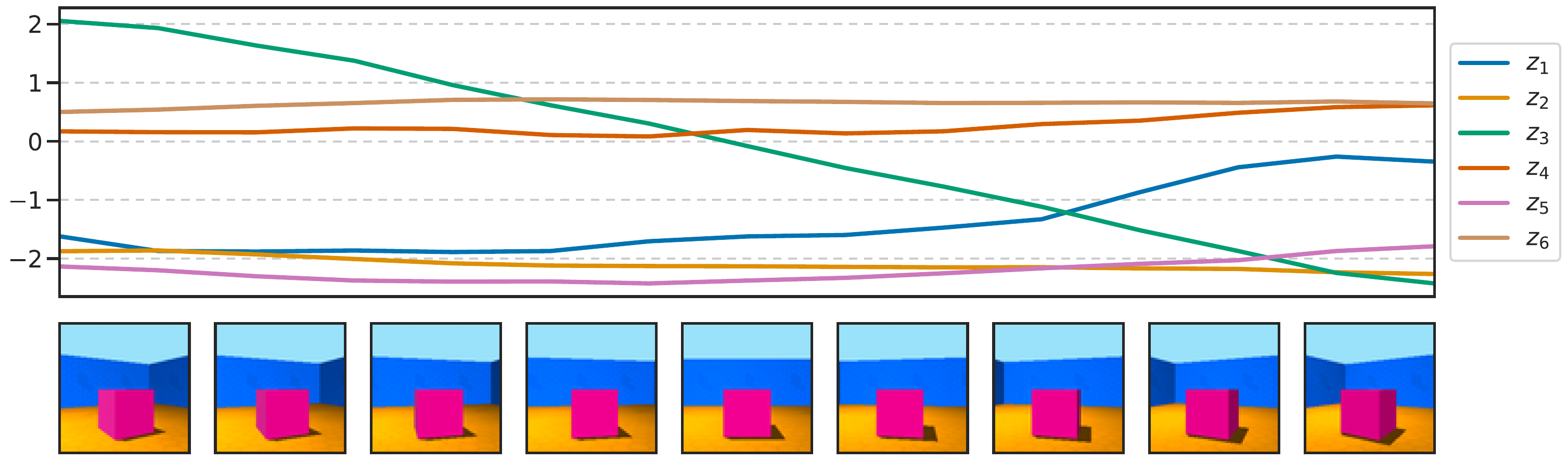}
    \includegraphics[width=0.65\columnwidth]{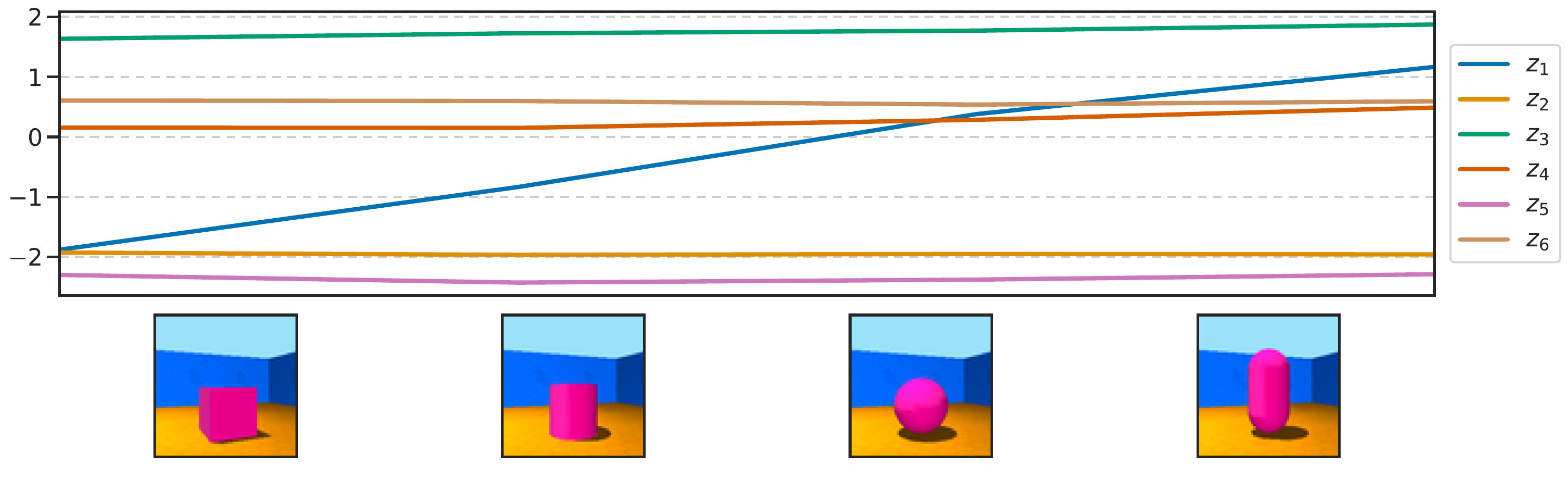}
    \caption{{Varying one factor at a time in the image and showing how the learned representation varies in response. This representation was learned with \textbf{inner-Lasso} (best hyperparameter) on a dataset with \textbf{correlation 0.9 between latents} and a noise scale of 1. The corresponding \textbf{MCC is 0.96}. Qualitatively, the representation appears to be well disentangled, but not as well as in~\Cref{fig:latent-responses_lasso_no_corr} (reflected by a drop in MCC of 0.03).}}
    \label{fig:latent-responses_lasso_corr}
\end{figure}

\begin{figure}[h]
    \centering
    \includegraphics[width=0.65\columnwidth]{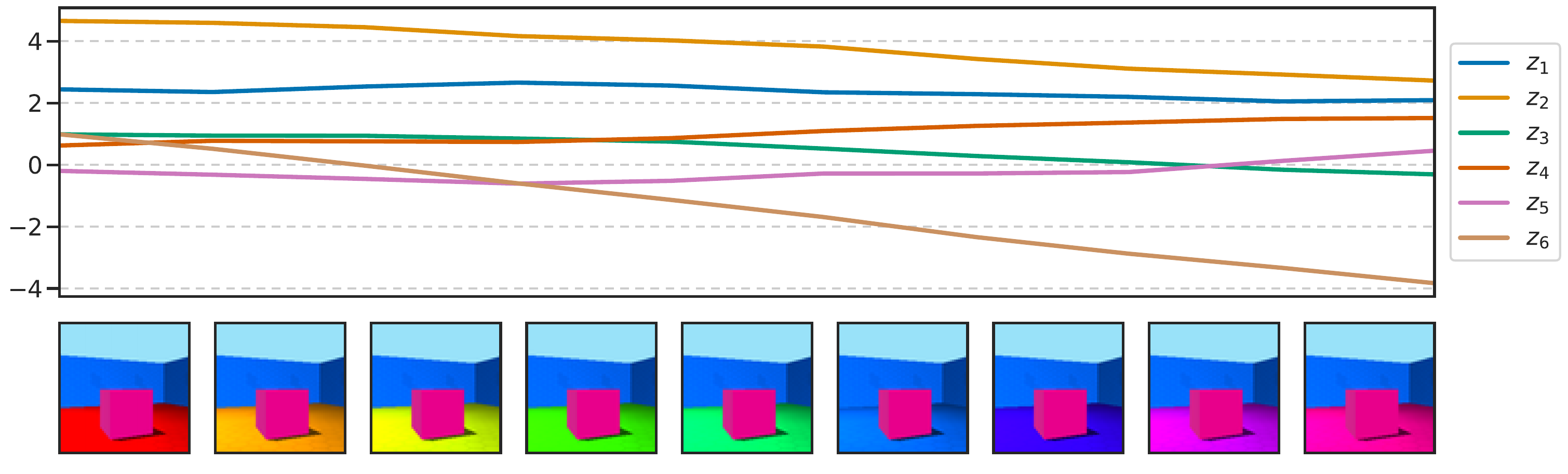}
    \includegraphics[width=0.65\columnwidth]{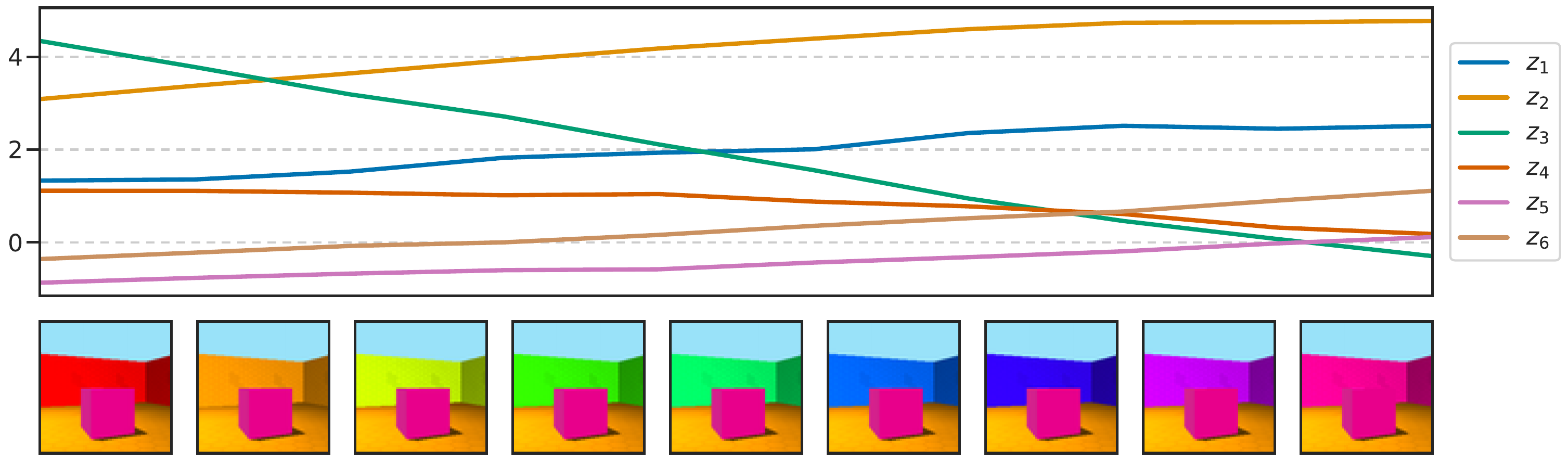}
    \includegraphics[width=0.65\columnwidth]{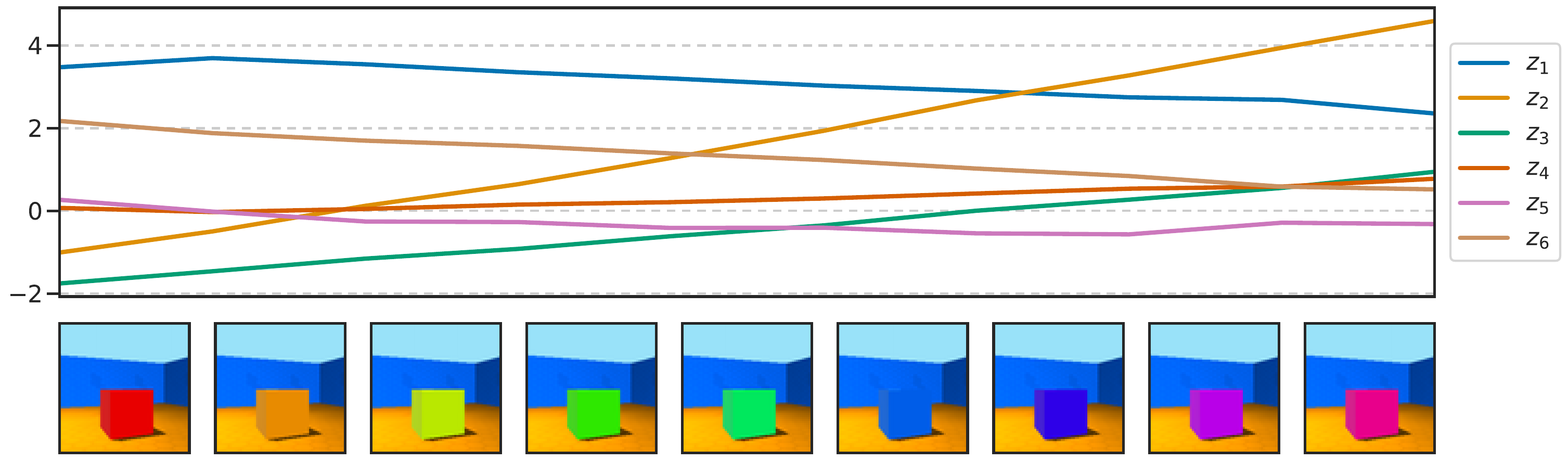}
    \includegraphics[width=0.65\columnwidth]{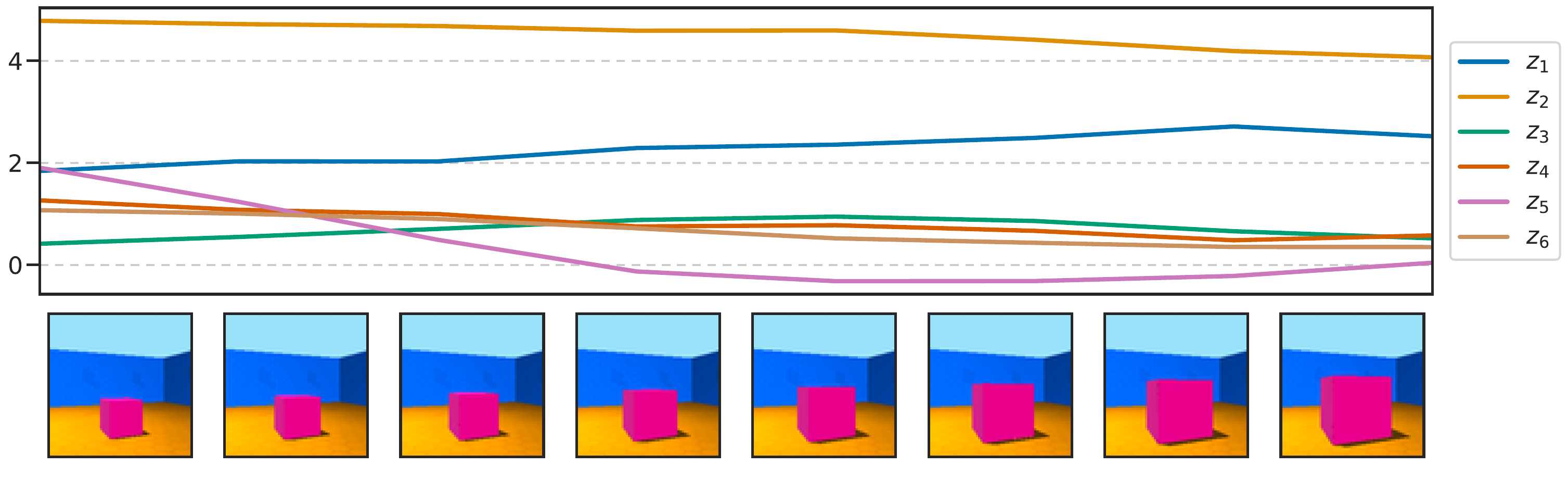}
    \includegraphics[width=0.65\columnwidth]{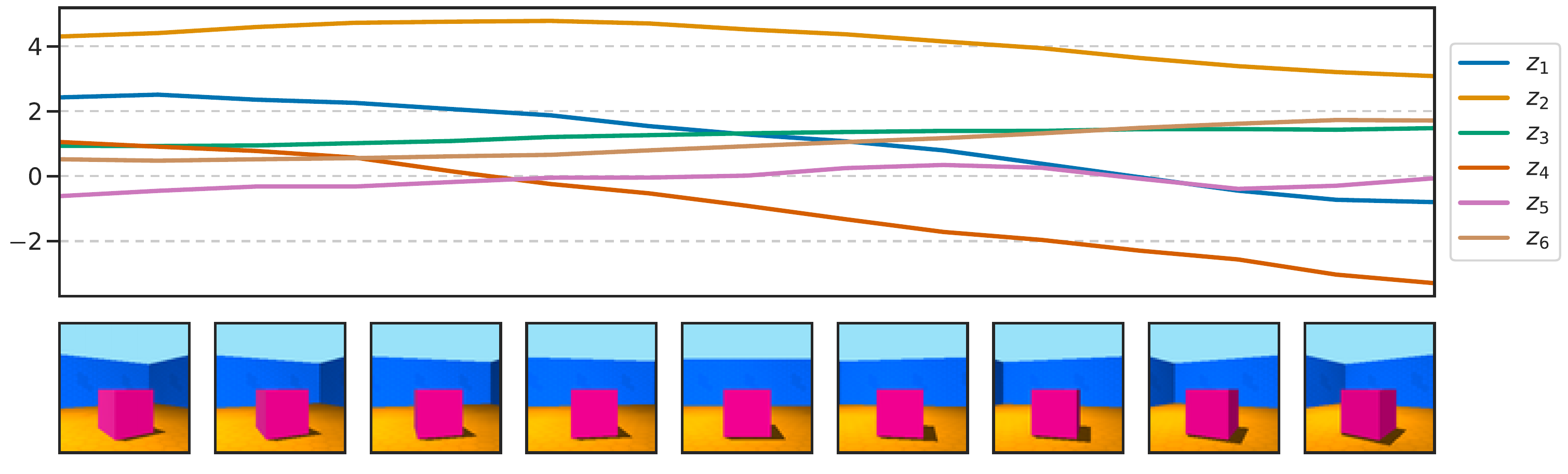}
    \includegraphics[width=0.65\columnwidth]{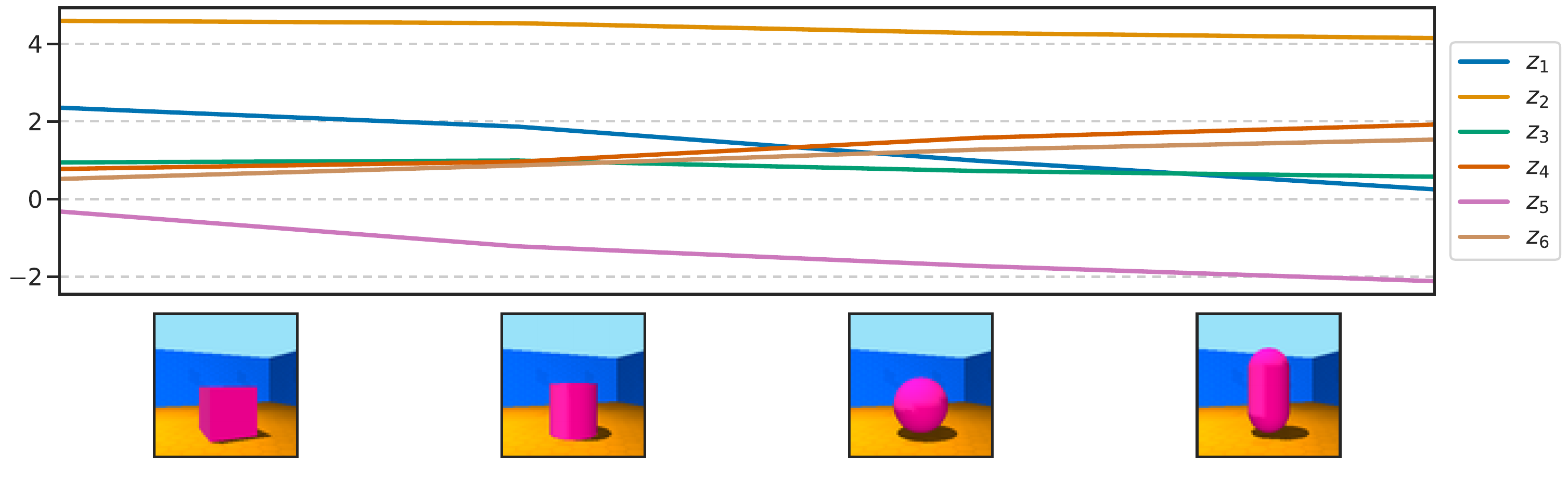}
    \caption{{Varying one factor at a time in the image and showing how the learned representation varies in response. This representation was learned with \textbf{inner-Ridge} (best hyperparameter) on a dataset with \textbf{correlation 0.9 between latents} and a noise scale of 1. The corresponding \textbf{MCC is 0.79}. For most latent factors, we cannot identify a dominant feature, except maybe for background and object colors. The representation appears more disentangled than~\Cref{fig:latent-responses_no_reg_no_corr}, but less disentangled than~\Cref{fig:latent-responses_lasso_corr}, as reflected by their corresponding MCC values.}}
    \label{fig:latent-responses_ridge_corr}
\end{figure}

\subsubsection{{Additional metrics for disentanglement}}
\label{app:dci_metrics}

{We implemented metrics from the DCI framework~\citep{eastwood2018framework} to evaluate disentanglement. 1) DCI-Disentanglement: How many ground truth latent components are related to a particular component of the learned latent representation; 2) DCI-Completeness: How many learned latent components are related to a particular component of the ground truth latent representation. Note that for the definition of disentanglement used in the present work~\Cref{def:disentanglement}, we want both DCI-disentanglement and DCI-completeness to be high.}

{The DCI framework requires a matrix of relative importance. In our implementation, this matrix is the coefficient matrix resulting from performing linear regression with inputs as the learned latent representation $\vf_{\hat\vtheta}(\vx)$ and targets as the ground truth latent representation $\vf_{\vtheta}(\vx)$, and denote the solution as the matrix $W$. Further, denote by $I= |W|$ as the importance matrix, as $I_{i,j}$ denotes the relevance of inferred latent $\vf_{\hat\vtheta}(\vx)_j$ for predicting the true latent $\vf_{\vtheta}(\vx)_i$.}

{ Now, for computing DCI-disentanglement, we normalize each row of the importance matrix $I[i,:]$ by its sum so that it represents a probability distribution. Then disentanglement is given by $ \frac{1}{m} \times \sum_{i}^{m} 1 - H(I[i,:])$, where $H$ denotes the entropy of a distribution. Note that for the desired case of each ground truth latent component being explained by a single inferred latent component, we would have $H(I[i, :])= 0$ as we have a one-hot vector for the probability distribution. Similarly, for the case of each ground truth latent component being explained uniformly by all the inferred latents, $H(I[i,:])$ would be maximized and hence the DCI score would be minimized. 
To compute the DCI-completeness, we first normalize each column of the importance matrix $I[:, j]$ by its sum so that it represents a probability distribution and then compute $\frac{1}{m} \times \sum_{i}^{m} 1 - H(I[:,j])$. }

{Figure~\ref{fig:3dshape_dci} shows the results for the 3D Shapes experiments (Section ~\ref{sec:dis_exp}) with the DCI metric to evaluate disentanglement. Notice that we find the same trend as we had with the MCC metric~\ref{fig:3dshape_mcc}, that inner-Lasso is more robust to correlation between the latent variables, and inner-Ridge + ICA performance drops down significantly with increasing correlation.}

\begin{figure}[tb]
    \centering
    \includegraphics[width=0.7\columnwidth]{figures/3dshape_influ_corr_fulllegend.png}
    \includegraphics[width=1\columnwidth]{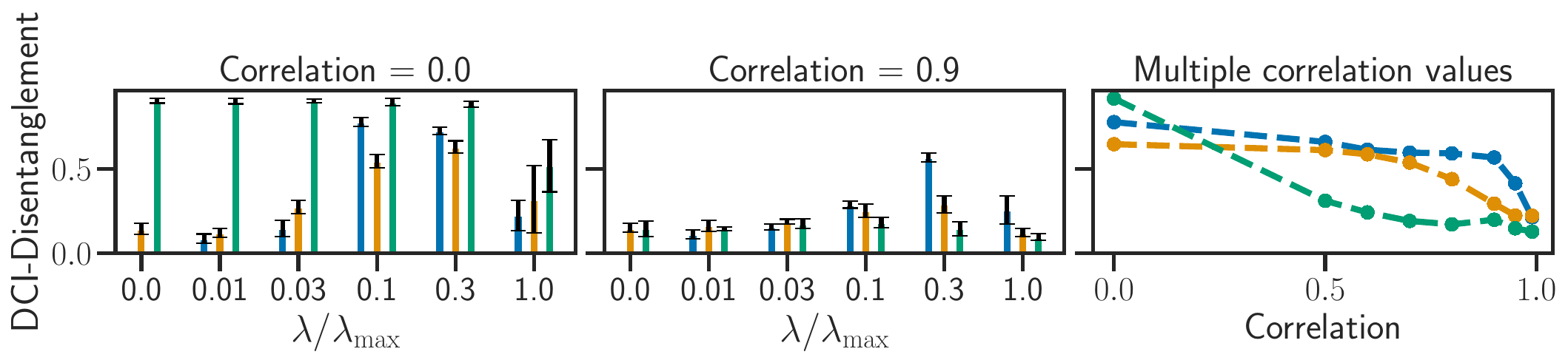}
    \includegraphics[width=1\columnwidth]{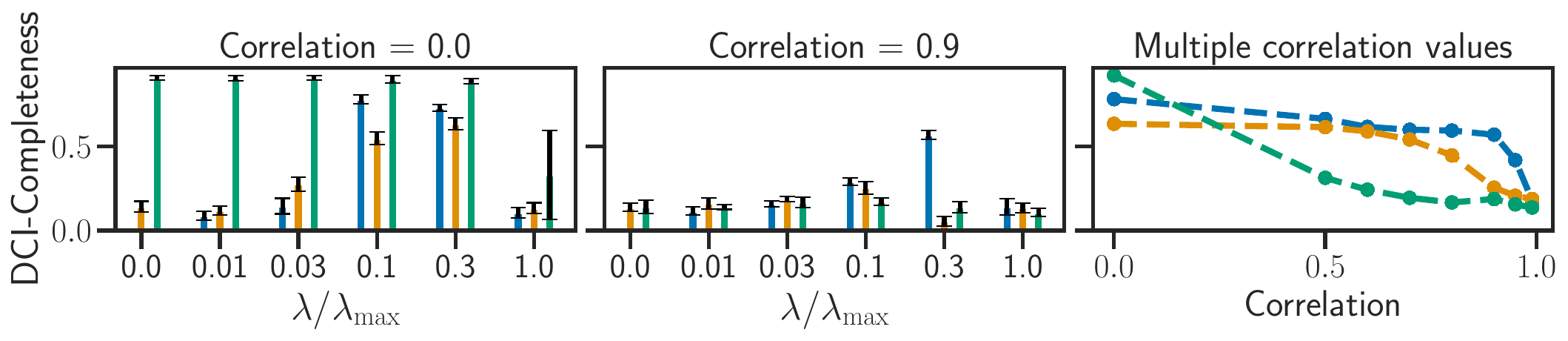}
    \caption{{Disentanglement performance (DCI) for inner-Lasso, inner-Ridge and inner-Ridge combined with ICA as a function of the regularization parameter (left and middle). The right columns shows performance of the best hyperparameter for different values of correlation and noise. The top row shows the results for the disentanglement metric of DCI and the bottom row shows the results for the completeness metric of DCI.}}
    \label{fig:3dshape_dci}
\end{figure}

\subsection{Meta-learning experiments}
\label{app:meta-learning}

\textbf{Experimental settings.} We evaluate the performance of our meta-learning algorithm based on a group-sparse SVM learners on the \textit{mini}ImageNet \citep{vinyals2016matchingnet} dataset. Following the standard nomenclature in few-shot classification \citep{hospedales2021metalearningsurvey} with $k$-shot $N$-way, where $N$ is the number of classes in each classification task, and $k$ is the number of samples per class in the training dataset $\gD_{t}^{\mathrm{train}}$, we consider the experimental setting $5$-shot $5$-way. We use the same residual network architecture as in \citep{Lee_Maji_Ravichandran_Soatto2019meta}, with $12$ layers and a representation of size $p = 1.6\times 10^{4}$.

\textbf{Technical details.}
    \looseness=-1
    The objective of \Cref{pb:dual_multiclass_group_svm_main} is composed of a smooth term and block separable non-smooth term, hence it can be solved efficiently using proximal block coordinate descent \citep{Tseng2001}.
    {Although \Cref{thm:disentanglement_via_optim} is not directly applicable to the meta-learning formulation proposed in this section, we conjecture that similar techniques could be reused to prove an identifiability result in this setting.}
    As in \Cref{sub:tractable_bilevel}, the argmin differentiation of the solution of \Cref{pb:dual_multiclass_group_svm_main} can be done using implicit differentiation \citep{Bertrand2022}.


\end{document}